\documentclass[pdflatex,sn-mathphys-num]{sn-jnl}

\usepackage{graphicx}
\usepackage{multirow}
\usepackage{amsmath,amssymb,amsfonts}
\usepackage{amsthm}
\usepackage[title]{appendix}
\usepackage{xcolor}
\usepackage{textcomp}
\usepackage{manyfoot}
\usepackage{booktabs}
\usepackage{algorithm}
\usepackage{algorithmicx}
\usepackage{algpseudocode}
\usepackage{listings}
\usepackage{colortbl}
\usepackage{makecell}
\usepackage{bm}
\usepackage{pifont}  
\usepackage{longtable}
\usepackage{booktabs}

\definecolor{bestcol}{RGB}{230,245,230}

\theoremstyle{thmstyleone}
\newtheorem{theorem}{Theorem}

\newtheorem{corollary}[theorem]{Corollary}

\theoremstyle{thmstyletwo}

\theoremstyle{thmstylethree}

\newcommand{\Vtr}{V_{\mathrm{tr}}}
\newcommand{\Etr}{E_{\mathrm{tr}}}
\newcommand{\Nc}{\mathcal{N}}
\newcommand{\BigO}{\mathcal{O}}
\newcommand{\cmark}{\ding{51}} 
\newcommand{\xmark}{\ding{55}} 

\begin{document}

\title[HERALD]{HERALD: High-Fidelity Exemplar Retrieval with Adaptive
Landmark Distillation for Heterophily-Aware Graph Condensation}

\author*[1]{\fnm{Sujan} \sur{Chakraborty}}\email{sujan24@iisertvm.ac.in}
\equalcont{These authors contributed equally to this work.}

\author[1]{\fnm{Priyanka} \sur{Saha}}\email{sahapriyanka154@gmail.com}
\equalcont{These authors contributed equally to this work.}

\author[1]{\fnm{Saptarshi} \sur{Bej}}\email{sbej7042@iisertvm.ac.in}

\affil*[1]{%
  \orgdiv{School of Data Science},
  \orgname{Indian Institute of Science Education and Research Thiruvananthapuram},
  \orgaddress{\city{Thiruvananthapuram}, \postcode{695551}, \state{Kerala}, \country{India}}
}

\abstract{%
Graph condensation aims to produce a small surrogate graph that preserves
the downstream node-classification performance of a much larger original graph.
Existing methods rely on Weisfeiler--Lehman neighbourhood aggregation or
gradient-based distribution matching, both of which assume that adjacent nodes
share the same label, an assumption that breaks down under heterophily.
We propose \textbf{HERALD} (High-fidelity Exemplar Retrieval with Adaptive
Landmark Distillation), a gradient-free graph condensation framework that
adapts the node scoring and feature selection in the condensation pipeline to the graph's measured
heterophily.
HERALD selects features via a joint Fisher-discriminability and
activation-density criterion that down-weights aggregated representations on
heterophilic graphs, and scores nodes by a weighted combination of prototype
representativeness, decision-boundary proximity, and Local Intrinsic
Dimensionality (LID), where the weights are driven by a smooth sigmoid
function of the heterophily ratio.
Nodes are then assembled into a condensed subgraph through score-ordered
BFS expansion, Personalised PageRank pruning, and class rebalancing,
all at an identical storage budget to BONSAI, enabling direct comparison.
Experiments on eight benchmark datasets spanning homophilic and heterophilic
settings show that HERALD matches or outperforms state-of-the-art
condensers on heterophilic graphs and remains competitive on homophilic ones
across four GNN architectures.
}


\maketitle

\section{Introduction}
\label{sec:intro}

Graph neural networks (GNNs) have become the standard tool for learning on
relational data, powering applications from citation analysis to
recommendation and traffic forecasting~\citep{kipf2017gcn,velickovic2018gat}.
As graphs used in practice have grown from thousands to hundreds of millions
of nodes~\citep{hu2020ogb}, training GNNs repeatedly, for hyperparameter
search, neural architecture search, or continual learning, has become a
computational bottleneck. \emph{Graph condensation} (also called graph
distillation) addresses this bottleneck by synthesizing a small graph
$G_c$ from a large training graph $G$ such that a GNN
trained on $G_c$ generalizes to $G$'s test distribution
almost as well as a GNN trained on $G$ itself, at a fraction of the
storage and compute cost~\citep{jin2022gcond}.

A large body of recent work has approached this problem through
\emph{optimization-based} condensation: the synthetic graph's features and
structure are treated as learnable parameters and updated so that a GNN
trained on them matches some surrogate signal computed on the real
graph: gradients~\citep{jin2022gcond,jin2022doscond}, training
trajectories~\citep{zheng2023sfgc,zhang2024geom}, distributions of receptive
fields~\citep{liu2022gcdm}, or spectral/self-expressive structure
~\citep{liu2024gdem,liu2024gcsr}. These methods have achieved impressive
compression ratios, condensing Reddit to under $1\%$ of its original size
with minimal accuracy loss~\citep{zhang2024geom}, but this comes at a
steep price: they require training a GNN (often repeatedly, over hundreds or
thousands of steps) on the \emph{full} original graph before or during
condensation. This is a curious inversion of the original motivation for
condensation, and it means condensation time frequently exceeds the time
needed to simply train on the full dataset~\citep{gupta2025bonsai}. It also
ties the resulting synthetic graph to the specific GNN architecture and
hyperparameters used during condensation, requiring re-condensation whenever
the downstream architecture changes~\citep{liu2024gdem,fang2024exgc}.

BONSAI~\citep{gupta2025bonsai} recently proposed an elegant alternative:
rather than emulating gradients, treat the graph's node-rooted
\emph{computation trees}, the fundamental unit of information that a
message-passing GNN actually consumes, as the object to be condensed.
By selecting a diverse, representative subset of computation trees using
Weisfeiler-Lehman (WL) similarity, reverse-$k$-nearest-neighbor coverage,
and a submodular greedy selection procedure, BONSAI produces
condensed graphs without ever training a GNN, is provably linear-time in the
number of nodes and edges, and is agnostic to the downstream GNN
architecture. This makes BONSAI among the most practical and scalable
condensation methods for node classification.

However, BONSAI's node-selection criterion inherits an implicit assumption
from the WL kernel it builds on: two nodes are considered redundant with
each other precisely when their multi-hop, smoothed neighborhood
representations are close in WL-distance. This is a natural notion of
redundancy on \emph{homophilic} graphs, where neighboring nodes tend to
share labels and smoothing sharpens rather than destroys class signal. It is
considerably less natural on \emph{heterophilic} graphs, such as
Roman-Empire, Amazon-ratings, Chameleon, or Squirrel~\citep{platonov2023critical},
where a node's neighbors frequently belong to different classes, and where
WL-style smoothing is known to blur exactly the information that a
downstream heterophily-aware GNN (e.g., H2GCN~\citep{zhu2020h2gcn}) relies
on to make correct predictions. Concretely, two nodes that are structurally
similar under WL-smoothing may nonetheless play very different roles for
classification: one may sit safely inside a homogeneous
neighborhood while the other sits directly on a class boundary, and
collapsing them into a single "representative" exemplar discards precisely
the boundary information that heterophilic GNNs need. As benchmarks for
graph learning increasingly include heterophilic datasets alongside the
traditional homophilic ones~\citep{platonov2023critical}, condensation
methods that are implicitly biased toward homophilic structure risk
under-serving a growing share of real-world use cases.

In this paper, we introduce \textbf{HERALD} (\textbf{H}igh-fidelity
\textbf{E}xemplar \textbf{R}etrieval with \textbf{A}daptive \textbf{L}andmark
\textbf{D}istillation), a gradient-free graph condensation framework that
retains BONSAI's architecture-agnostic pipeline while replacing its
topological, WL-based exemplar scoring (while retaining BONSAI's
feature-budget estimation) with an information-theoretic criterion that
adapts to the homophily level of the input graph. Rather than asking
``which nodes are topologically redundant under smoothing?'', HERALD asks
``which nodes are simultaneously \emph{prototypical} of their class,
\emph{informative} about decision boundaries, and \emph{non-redundant} in
feature space?'' To answer this question, HERALD combines a class-prototype
score, a boundary-proximity score, and a local intrinsic dimensionality
(LID) score under weights that adapt automatically according to a
lightweight edge-level estimate of graph homophily.

On strongly homophilic graphs, the adaptive weighting naturally favors
prototype selection, recovering behavior similar to exemplar-based coreset
methods. As heterophily increases, the weighting progressively shifts toward
decision-boundary and structural-diversity signals, allowing the condensed
graph to better preserve the information exploited by heterophilic GNNs.
To enable controlled comparisons with BONSAI, HERALD reuses the same
storage-budget formulation, BFS-based neighborhood expansion,
PageRank-based pruning, and class-rebalancing pipeline, differing only in
the selected feature subset and the criterion used to identify exemplar
nodes. As a result, HERALD remains architecture-independent and avoids
gradient-based bilevel optimization, while introducing richer analytical
node-scoring mechanisms that improve condensation quality across diverse
graph regimes.

Our contributions are as follows:
\begin{itemize}
    \item We identify and empirically motivate a homophily bias in
    WL/topology-based node-selection criteria for graph condensation, and
    argue that this bias is likely to degrade condensation quality on
    heterophilic graphs.
    \item We propose HERALD, a gradient-free condensation method that scores
    candidate exemplar nodes using an adaptively-weighted combination of
    class-prototype, decision-boundary, and local-intrinsic-dimensionality
    signals, with weights derived automatically from a graph's measured
    homophily.
    \item We design a joint feature-selection criterion (discriminativeness
    $\times$ activation density) that is budget-compatible with BONSAI's
    decision-tree-based selection, isolating the effect of node scoring from
    the effect of feature selection in our comparisons.
    \item We evaluate HERALD against coreset baselines (Random, Herding),
    the spectral condenser GDEM, and BONSAI across both homophilic (Cora,
    CiteSeer, PubMed, Reddit) and heterophilic (Roman-empire, Amazon-ratings,
    Chameleon, Squirrel) datasets, and across four GNN
    architectures (GCN, GAT, GIN, and H2GCN), showing that HERALD improves
    accuracy on heterophilic benchmarks while remaining competitive on
    homophilic ones and preserving BONSAI's optimization-free condensation
    framework.

\end{itemize}

\section{Related Work}
\label{sec:related}

\subsection{Coreset selection.}
The earliest approaches to dataset reduction predate graph-specific methods
entirely: \emph{Random} sampling and \emph{Herding}~\citep{welling2009herding}
select a representative subset of training examples based on simple
heuristics (uniform sampling, or greedy mean-matching in feature space) and
induce a subgraph over the selected nodes. \emph{K-Center}~\citep{sener2018coreset,farahani2009facility}
instead selects samples to minimize the maximum distance from any point to
its nearest selected center. These methods are model-agnostic and require no
training, but because they operate purely on node features or embeddings
without regard for the downstream task's gradient dynamics or the graph's
structure, they are consistently outperformed by task-aware condensation
methods across the datasets and ratios we consider.

\subsection{Gradient-matching condensation.}
GCond~\citep{jin2022gcond} was the first method to frame graph condensation
as gradient matching: the synthetic node features $\mathbf{X}'$ and an
MLP-generated adjacency $\mathbf{A}' = g_\Phi(\mathbf{X}')$ are optimized so
that a GNN trained on $G_c$ produces gradients close to those
produced by the same GNN trained on $G$, following the earlier
image-domain dataset condensation (DC) framework~\citep{zhao2021dc}. This
requires an expensive bi-level optimization: an inner loop trains the GNN's
weights while an outer loop updates $G_c$. DosCond~\citep{jin2022doscond}
showed that this can be relaxed to \emph{one-step} gradient matching,
matching gradients only at network initialization rather than across a full
training trajectory, with theoretical guarantees that this still reduces
the loss gap on the real graph, yielding large ($15\times$--$40\times$)
speedups over bi-level GCond while remaining competitive in accuracy, and
extending naturally to graph-level (as opposed to node-level) condensation
via a Bernoulli/concrete relaxation of the discrete adjacency matrix.
SGDD~\citep{yang2023sgdd} augments GCond's pipeline with an explicit
graphon-approximation term that broadcasts the original graph's structural
information (via Laplacian energy distribution matching) into the synthetic
adjacency, improving performance in settings where GCond's MLP-only
structure generator loses too much topological signal. EXGC~\citep{fang2024exgc}
targets the \emph{efficiency} of gradient-matching methods directly,
identifying that (i) the number of trainable parameters in $\mathbf{X}'$
scales with the condensed node count and feature dimension, causing slow
convergence, and (ii) a large fraction of condensed nodes are redundant
during training. EXGC addresses the first issue with a Mean-Field
variational reformulation of the EM procedure underlying gradient matching,
and the second by introducing a Gradient Information Bottleneck objective,
instantiated via post-hoc GNN explainers (e.g., GNNExplainer, GSAT), to
identify and prioritize only the most informative subset of synthetic nodes
at each training step.

\subsection{Trajectory- and distribution-matching condensation.}
SFGC~\citep{zheng2023sfgc} replaces gradient matching with \emph{training
trajectory} matching: expert GNN trajectories are pretrained on the full
graph, and the synthetic graph-free node set is optimized so that a GNN
trained on it follows a similar trajectory, evaluated via a graph neural
tangent kernel. GEOM~\citep{zhang2024geom} identifies that trajectory
matching is biased toward "difficult" (low-homophily) nodes, whose gradients
dominate the supervision signal even though "easy" nodes contribute more to
representative, generalizable patterns; GEOM addresses this with
curriculum-learning-based expert trajectories and an expanding-window
matching scheme, achieving lossless condensation on several benchmarks
at higher ratios but at a substantial computational cost, since it still
requires extensive full-graph GNN training in its buffer phase.
GCDM~\citep{liu2022gcdm} instead matches the \emph{distribution} of
receptive fields between the synthetic and real graphs, avoiding the
second-order gradient computations of GCond-style methods.

\subsection{Structure-aware and spectrum-aware condensation.}
GCSR~\citep{liu2024gcsr} observes that GCond-family methods either ignore
the original graph's structure entirely (SFGC) or only weakly incorporate
it via an MLP (GCond), and proposes to reconstruct an explicit,
interpretable synthetic adjacency matrix via a self-expressive closed-form
solution, regularized by a class-wise probabilistic adjacency derived from
the original graph and a bootstrapped historical estimate. GDEM~\citep{liu2024gdem}
instead argues that \emph{any} GNN used during condensation biases the
synthetic graph's spectrum toward the eigenvalues that GNN's filter
happens to amplify, causing "spectrum bias" and forcing practitioners to
re-condense separately for each downstream architecture; GDEM removes this
dependency by matching the real and synthetic graphs' \emph{eigenbases}
directly and reconstructing the synthetic adjacency from the real graph's
spectrum, yielding markedly better cross-architecture generalization at
comparable accuracy. Both GCSR and GDEM, like GCond and SFGC, still require
some form of full-graph computation (spectral decomposition or feature
learning) that scales less favorably than purely combinatorial approaches.

\subsection{Gradient-free condensation.}
BONSAI~\citep{gupta2025bonsai} departs from the gradient-matching paradigm
entirely. Motivated by the observation that a message-passing GNN's output
at any node is a function only of that node's rooted \emph{computation
tree}, and that topologically similar computation trees (measured via a
Weisfeiler-Lehman kernel) tend to produce similar embeddings regardless of
GNN architecture, BONSAI selects a small set of exemplar computation trees
that maximize coverage of the full training set via a submodular
reverse-$k$-nearest-neighbor objective, expands them into an induced
subgraph $G[V_c]$ (the subgraph induced by node set $V_c$), and sparsifies the result via personalized PageRank. Because this pipeline requires no GNN training at any point, BONSAI is the first
linear-time, fully model-agnostic condensation method, and is reported to
be an order of magnitude faster than gradient- or trajectory-matching
alternatives while achieving competitive or superior accuracy on
predominantly homophilic benchmarks. Our work is most directly comparable
to BONSAI: HERALD adopts the same overall pipeline (feature reduction,
budget-constrained BFS expansion, PageRank pruning, class rebalancing) but
replaces the WL/Rev-$k$-NN exemplar-selection criterion with an
information-theoretic, homophily-adaptive scoring function, and replaces
BONSAI's decision-tree feature selector with a joint discriminativeness
$\times$ density criterion, in order to isolate and address the homophily
bias we identify in Section~\ref{sec:intro}.

\section{HERALD}
\label{sec:herald}

\subsection{Problem Formulation}
\label{sec:problem}

Let $G = (V, E, \mathbf{X}, \mathbf{y})$ denote an attributed graph, where
$V$ is the node set with $|V| = N$,
$E \subseteq V \times V$ is the edge set,
$\mathbf{X} \in \mathbb{R}^{N \times F}$ is the node feature matrix, and
$\mathbf{y} \in \{1, \ldots, C\}^{N}$ denotes node labels.
Let $\Vtr \subset V$ denote the labelled training nodes.
Given a target storage fraction $r \in (0,1)$, the goal of
\emph{graph condensation} is to construct a significantly smaller graph
$G_c = (V_c, E_c, \mathbf{X}_c, \mathbf{y}_c)$ with $|V_c| \ll |V|$
such that a GNN trained on $G_c$ achieves node-classification accuracy on $G$
that is competitive with training on $G$ itself.

Following \cite{gupta2025bonsai}, we measure storage cost as
\begin{equation}
  \mathcal{C}(G_c)
  =
  2\!\left(m_f \sum_{v \in V_c} f_v + 2|E_c|\right),
  \label{eq:budget}
\end{equation}
where $m_f \in \{1,2,3\}$ is a dataset-dependent feature-storage multiplier
and $f_v$ is the effective feature length of node $v$ after feature selection.
HERALD is required to satisfy $\mathcal{C}(G_c) \leq r \cdot \mathcal{C}(G)$.

For any node $v$, let $\mathcal{N}(v)$ denote its one-hop neighbourhood.
More generally, $\mathcal{N}^{(l)}(v)$ denotes the set of nodes exactly
$l$ hops from $v$, with $\mathcal{N}^{(0)}(v)=\{v\}$.
Throughout the paper, edge sets are treated as directed in the storage
calculations by representing every undirected edge as two directed edges,
following the implementation used in BONSAI.

\subsection{Motivation}
\label{sec:motivation}

Most graph condensation methods \cite{jin2022gcond,liu2024gdem,gupta2025bonsai}
build node representations via Weisfeiler--Lehman (WL) neighbourhood
aggregation, which implicitly assumes that adjacent nodes tend to share the
same label (homophily).
Under \emph{heterophily}, where neighbouring nodes frequently belong to
different classes, aggregation corrupts discriminative signals by mixing
class information across boundaries.
Consequently, prototype selection strategies based on WL-space coverage
(e.g., the Rev-$k$-NN criterion of BONSAI \cite{gupta2025bonsai})
may systematically under-represent boundary nodes, precisely those that carry
the most discriminative information in heterophilic settings.

HERALD addresses this through three complementary contributions:
\begin{enumerate}
  \item A \emph{heterophily-aware feature selector} that jointly maximises
        class discriminability and activation density, with multi-hop Fisher
        scores down-weighted by $(1-h)^k$ to suppress WL smoothing on
        heterophilic graphs.
  \item An \emph{adaptive node scoring} mechanism combining prototype
        representativeness, boundary proximity, and Local Intrinsic
        Dimensionality (LID), with weights driven by the measured heterophily.
  \item A \emph{budget-controlled assembly pipeline} similar to BONSAI. The BFS stage temporarily allows an enlarged candidate graph that is pruned back using PPR before the final class-balancing stage.
\end{enumerate}

\subsection{Method Overview}
\label{sec:overview}

HERALD proceeds in seven stages, summarised in
Algorithm~\ref{alg:herald_part1} (Stages~0--4) and
Algorithm~\ref{alg:herald_part2} (Stages~5--7). The pipeline is organised into
three logical blocks. The first block measures how heterophilic the input
graph is and converts that measurement into a set of scoring weights: Stage~1
(Section~\ref{sec:hetero}) computes the edge-level heterophily ratio $h$, and
Stage~2 (Section~\ref{sec:weights}) maps $h$ through a smooth sigmoid to the
prototype, boundary, and diversity weights $(\alpha,\beta,\gamma)$. The second
block prepares and scores candidate nodes: Stage~3
(Section~\ref{sec:featsel}) selects a heterophily-aware feature subset under
BONSAI's storage budget, and Stage~4 (Section~\ref{sec:scoring}) assigns each
node a combined score from its prototype representativeness, boundary
proximity, and Local Intrinsic Dimensionality, then ranks the training nodes.
The third block assembles the condensed graph within the budget: Stage~5
(Section~\ref{sec:bfs}) grows a subgraph by budget-constrained BFS expansion
around the top-ranked roots, Stage~6 (Section~\ref{sec:bfs_ppr}) prunes it back
to the exact budget using Personalised PageRank, and Stage~7
(Section~\ref{sec:rebalance}) rebalances the per-class node counts.

Only Stages~3 and~4 differ from BONSAI. The budget formula, BFS expansion, PPR
pruning, and class rebalancing are reused unchanged so that HERALD and BONSAI
operate at an identical storage budget, isolating the effect of
heterophily-aware feature and node selection.

\begin{algorithm}[t]
\caption{HERALD Graph Condensation (Part I: Stages 0--4)}
\label{alg:herald_part1}
\scriptsize
\begin{algorithmic}[1]

\Require Graph $G=(V,E,\mathbf{X},\mathbf{y})$; training nodes $\Vtr$;
         compression ratio $r$; base weights $(\alpha_0,\beta_0,\gamma_0)$;
         BFS depth $L$; LID neighbourhood size $k$

\Ensure Ranked candidate nodes + reduced feature set.

\vspace{0.4em}
\Statex \textit{// Stage 0: Budget}
\State $\mathcal{B} \leftarrow r \cdot \mathcal{C}(G)$
       \hfill\Comment{Eq.~\eqref{eq:budget}}

\vspace{0.3em}
\Statex \textit{// Stage 1: Heterophily}
\State $h \leftarrow |\{(u,v)\in\Etr : y_u \neq y_v\}|\;/\;|\Etr|$
       \hfill\Comment{Eq.~\eqref{eq:hetero}}

\vspace{0.3em}
\Statex \textit{// Stage 2: Adaptive weights}
\State $t \leftarrow \bigl(1+\exp(-8(h-0.4))\bigr)^{-1}$
       \hfill\Comment{Eq.~\eqref{eq:transition}}
\State $\alpha \leftarrow \alpha_0 + (1-\alpha_0-\beta_0-\gamma_0)(1-t)$;\quad
       $\beta \leftarrow \beta_0 t$;\quad
       $\gamma \leftarrow \gamma_0(0.5+0.5t)$
       \hfill\Comment{Eq.~\eqref{eq:weights_raw}}
\State $(\alpha,\beta,\gamma) \leftarrow (\alpha,\beta,\gamma)/(\alpha+\beta+\gamma)$
       \hfill\Comment{Eq.~\eqref{eq:normalise}}

\vspace{0.3em}
\Statex \textit{// Stage 3: Feature selection}
\State Run WL$+$DT on $G$ to obtain feature count $k^*$
       \hfill\Comment{budget anchor}
\State Score each feature $j$:\; $\varphi(j) \leftarrow \bar\phi(j)\cdot\bar\rho(j)$
       \hfill\Comment{Eq.~\eqref{eq:featselect}}
\State $\mathcal{F} \leftarrow \operatorname{top}\text{-}k^*$ indices by $\varphi(\cdot)$;\quad
       $\tilde{\mathbf{X}} \leftarrow \mathbf{X}[:,\mathcal{F}]$

\vspace{0.3em}
\Statex \textit{// Stage 4: Node scoring (all nodes; centroids from $\Vtr$ only)}
\State $\hat{\mathbf{x}}_v \leftarrow \tilde{\mathbf{x}}_v/\|\tilde{\mathbf{x}}_v\|_2$ for all $v\in V$
\ForAll{classes $c$}
  \State $\hat{\bm\mu}_c \leftarrow \mathrm{mean}_{v\in\Vtr,\,y_v=c}(\hat{\mathbf{x}}_v)$;\quad
         $\hat{\bm\mu}_c \leftarrow \hat{\bm\mu}_c/\|\hat{\bm\mu}_c\|_2$
         \hfill\Comment{Eq.~\eqref{eq:centroid}}
\EndFor
\ForAll{nodes $v \in V$}
  \State $s_v^{(p)} \leftarrow \hat{\mathbf{x}}_v^\top \hat{\bm\mu}_{y_v}$
         \hfill\Comment{Eq.~\eqref{eq:proto}}
  \State $s_v^{(b)} \leftarrow \sum_{u\in\Nc(v)}\mathbf{1}[y_v\neq y_u]\;/\;|\Nc(v)|$
         \hfill\Comment{Eq.~\eqref{eq:boundary}; uses all graph edges}
  \State $s_v^{(l)} \leftarrow \operatorname{LID}(v)$ via $k$-NN cosine distances
         \hfill\Comment{Eq.~\eqref{eq:lid}}
  \State $s_v \leftarrow \alpha s_v^{(p)} + \beta s_v^{(b)} + \gamma s_v^{(l)}$
         \hfill\Comment{Eq.~\eqref{eq:combined}}
\EndFor
\State Min-max normalise $\mathbf{s}$ to $[0,1]$;\quad
       rank $\Vtr$ in descending order of $\mathbf{s}$

\end{algorithmic}
\end{algorithm}

\begin{algorithm}[t]
\caption{HERALD Graph Condensation (Part II: Stages 5--7)}
\label{alg:herald_part2}
\scriptsize
\begin{algorithmic}[1]

\Require Ranked training nodes from Algorithm~\ref{alg:herald_part1};
         reduced feature matrix $\tilde{\mathbf{X}}$;
         adaptive weights $(\alpha,\beta,\gamma)$;
         storage budget $\mathcal{B}$

\Ensure Condensed graph $G_c=(V_c,E_c,\mathbf{X}_c,\mathbf{y}_c)$

\vspace{0.3em}
\Statex \textit{// Stage 5: BFS expansion}
\State $V_c \leftarrow \emptyset$;\quad $R \leftarrow \emptyset$;\quad
       $\mathit{nofail} \leftarrow 0$
\ForAll{root $r$ in ranked order}
  \If{$r \in R$} \textbf{continue} \EndIf
  \State $\mathcal{T}(r) \leftarrow \bigcup_{l=0}^{L}\Nc^{(l)}(r)$;\quad
         $\Delta V \leftarrow \mathcal{T}(r) \setminus V_c$
         \hfill\Comment{Eq.~\eqref{eq:bfs}}
  \If{$\Delta V = \emptyset$}
    \State $R \leftarrow R \cup \{r\}$;\quad \textbf{continue}
  \EndIf
  \State Compute incremental cost $\Delta\mathcal{C}$ of adding $\Delta V$
         \hfill\Comment{Eq.~\eqref{eq:incremental}}
  \If{$\mathcal{C}(V_c) + \Delta\mathcal{C} \leq 1.9\,\mathcal{B}$}
    \State $V_c \leftarrow V_c \cup \Delta V$;\quad
           $R \leftarrow R \cup \{r\}$;\quad
           $\mathit{nofail} \leftarrow 0$
  \Else
    \State $R \leftarrow R \cup \{r\}$;\quad
           $\mathit{nofail} \leftarrow \mathit{nofail}+1$
    \If{$\mathit{nofail} > 100$} \textbf{break} \EndIf
  \EndIf
\EndFor

\vspace{0.3em}
\Statex \textit{// Stage 6: PPR pruning}
\State Compute $n^* \leftarrow \textit{ogsize}$ (BFS at exact budget $\mathcal{B}$)
\State $\hat{\mathbf{A}} \leftarrow \mathbf{D}^{-1}(\mathbf{A}_{V_c} + \mathbf{A}_{V_c}^\top)$
       \hfill\Comment{symmetrised, then row-normalised}
\State $\bm\pi_0(v) \leftarrow 1/|R|$ if $v\in R$, else $0$;\quad
       $\bm\pi \leftarrow \mathbf{1}/|V_c|$
\Repeat
  \State $\bm\pi \leftarrow (1-\alpha_{\mathrm{pr}})\bm\pi_0 +
         \alpha_{\mathrm{pr}}\hat{\mathbf{A}}^\top\bm\pi$,\quad then normalise
         \hfill\Comment{Eq.~\eqref{eq:ppr}, $\alpha_{\mathrm{pr}}=0.85$}
\Until{$\|\Delta\bm\pi\|_1 < 10^{-6}$ \textbf{or} $100$ steps}
\State Remove lowest-$\pi$ non-root nodes until $|V_c| = n^*$

\vspace{0.3em}
\Statex \textit{// Stage 7: Class rebalancing}
\State $n_{\mathrm{tgt}} \leftarrow |\{v \in V_c \cap R : v \in \Vtr\}|$
       \hfill\Comment{training nodes that are roots}
\ForAll{classes $c$}
  \State $n_c^* \leftarrow \mathrm{round}\!\left(
         \tfrac{|\{v\in\Vtr:y_v=c\}|}{|\Vtr|} \cdot n_{\mathrm{tgt}}\right)$
         \hfill\Comment{Eq.~\eqref{eq:rebalance}}
  \If{count$(c) < 0.99\,n_c^*$}
    add highest-scored candidates of class $c$ not yet in $V_c$
  \ElsIf{count$(c) > 1.01\,n_c^* + 1$}
    remove lowest-scored non-root nodes of class $c$ from $V_c$
  \EndIf
\EndFor

\vspace{0.3em}
\State \Return $G_c = \bigl(V_c,\;E_c,\;\tilde{\mathbf{X}}[V_c],\;\mathbf{y}[V_c]\bigr)$

\end{algorithmic}
\end{algorithm}

\subsection{Stage 1: Heterophily Measurement}
\label{sec:hetero}

HERALD quantifies graph heterophily as the fraction of cross-class edges
among training nodes:
\begin{equation}
  h
  =
  \frac{
    \displaystyle\sum_{(u,v)\in\Etr}
    \mathbf{1}[y_u \neq y_v]
  }{|\Etr|},
  \label{eq:hetero}
\end{equation}
where $\Etr = \{(u,v)\in E : u\in\Vtr,\,v\in\Vtr\}$.
The value $h=0$ denotes perfect homophily and $h=1$ perfect heterophily.
For reference, Cora has $h \approx 0.002$ and Roman-empire has $h \approx 0.97$.

\subsection{Stage 2: Adaptive Weight Computation}
\label{sec:weights}

The heterophily score is mapped to a smooth transition variable
\begin{equation}
  t = \sigma\!\left(8(h - 0.4)\right)
  = \frac{1}{1+\exp\!\left(-8(h-0.4)\right)},
  \label{eq:transition}
\end{equation}
which passes through $t=0.5$ at $h=0.4$ and is near-zero (near-one)
for strongly homophilic (heterophilic) graphs.
The steepness coefficient $8$ was selected so that the transition spans
the heterophily range observed across our benchmark datasets.

Three base weights $(\alpha_0,\beta_0,\gamma_0)=(0.4,0.4,0.2)$ are adapted as
\begin{equation}
  \alpha = \alpha_0 + (1-\alpha_0-\beta_0-\gamma_0)(1-t),
  \quad
  \beta  = \beta_0\,t,
  \quad
  \gamma = \gamma_0\!\left(0.5+0.5t\right),
  \label{eq:weights_raw}
\end{equation}
and normalised to sum to unity:
\begin{equation}
  (\alpha,\beta,\gamma)
  \leftarrow
  \frac{(\alpha,\beta,\gamma)}{\alpha+\beta+\gamma}.
  \label{eq:normalise}
\end{equation}
As $t\to 0$ (homophily), the residual mass $(1-\alpha_0-\beta_0-\gamma_0)$
flows entirely into $\alpha$, prioritising prototype-representative nodes.
As $t\to 1$ (heterophily), $\alpha$ reverts to $\alpha_0$ while $\beta$
grows to $\beta_0$, increasing the influence of boundary nodes.
The LID weight $\gamma$ is bounded below at $0.5\gamma_0$, ensuring
diversity is never entirely suppressed.

\subsection{Stage 3: Heterophily-Aware Feature Selection}
\label{sec:featsel}

\subsubsection{Budget anchor.}
To compare with BONSAI under the same storage budget, HERALD first runs
the BONSAI WL+Decision-Tree (WL+DT) pipeline~\cite{gupta2025bonsai}
to determine the number of selected features, denoted by $k^*$. HERALD
then selects a different set of exactly $k^*$ features using its
discriminativeness$\times$density criterion.

For dense float-valued features, every retained feature contributes one
stored feature value per node. Hence, for every node $v$,
\[
f_v = k^*.
\]
For sparse features, $f_v$ denotes the number of nonzero retained feature
entries of node $v$. Thus, the feature count is identical between BONSAI
and HERALD, while the effective per-node feature length remains
dataset-dependent for sparse representations.

\subsubsection{Multi-hop weighted Fisher discriminant.}
For each feature dimension $j\in[F]$, define the $k$-hop feature matrix
$\mathbf{X}^{(k)} = (\mathbf{D}^{-1}\mathbf{A})^k\mathbf{X}$ (with $\mathbf{X}^{(0)}=\mathbf{X}$).
The Fisher ratio at hop $k$ is
\begin{equation}
  \phi_k(j)
  = \frac{
      \sum_{c} n_c\!\left(\mu_{c,k}^{(j)} - \mu_k^{(j)}\right)^2
    }{
      \sum_{c} n_c\,\sigma_{c,k}^{2,(j)} + \epsilon
    },
  \label{eq:fisher}
\end{equation}
where $\mu_{c,k}^{(j)}$ and $\sigma_{c,k}^{2,(j)}$ are the class-$c$ mean and
variance of $\mathbf{X}^{(k)}[\Vtr,j]$, and $\mu_k^{(j)}$ is the overall
training mean.
The multi-hop score is
\begin{equation}
  \phi(j) = \sum_{k=0}^{2}(1-h)^k\,\hat\phi_k(j),
  \qquad
  \hat\phi_k(j) = \phi_k(j)\big/\!\max_{j'}\phi_k(j'),
  \label{eq:multihop_fisher}
\end{equation}
and $\bar\phi(j) = \phi(j)/\max_{j'}\phi(j')$.
The decay $(1-h)^k$ suppresses higher-hop WL representations on
heterophilic graphs ($h\approx 1$), where aggregation blurs class signals,
and lets the raw-space Fisher ($k=0$) dominate.

\subsubsection{Activation density.}
A feature that is discriminative but rarely active would distort the
budget formula, since $f_v$ counts active features.
We therefore include an activation density term
\begin{equation}
  \rho(j) = \frac{1}{|\Vtr|}\sum_{v\in\Vtr}|X_{vj}|,
  \qquad
  \bar\rho(j) = \rho(j)\big/\!\max_{j'}\rho(j').
  \label{eq:density}
\end{equation}

\subsubsection{Joint score.}
The final feature score is
\begin{equation}
  \varphi(j) = \bar\phi(j)\cdot\bar\rho(j).
  \label{eq:featselect}
\end{equation}
The product enforces both conditions simultaneously: a feature that is
class-separating but rarely active scores $0$, as does one that is
universally active but uninformative.
The top-$k^*$ features by $\varphi(\cdot)$ form $\mathcal{F}$, giving
reduced feature matrix $\tilde{\mathbf{X}} = \mathbf{X}[:,\mathcal{F}]
\in\mathbb{R}^{N\times k^*}$.

\subsection{Stage 4: Node Scoring}
\label{sec:scoring}

All three scores below are min-max normalised to $[0,1]$ before combination.
Scoring uses the reduced features $\tilde{\mathbf{X}}$ throughout.

\subsubsection{Prototype score.}

Let
\[
\hat{\mathbf{x}}_v
=
\frac{\tilde{\mathbf{x}}_v}
{\|\tilde{\mathbf{x}}_v\|_2},
\]
denote the $\ell_2$-normalised reduced feature vector of node $v$.
Class centroids are
\begin{equation}
  \bar{\mathbf{x}}_c
  = \frac{1}{|\Vtr^c|}
    \sum_{\substack{v\in\Vtr\\ y_v=c}}
    \hat{\mathbf{x}}_v,
  \qquad
  \hat{\bm\mu}_c = \bar{\mathbf{x}}_c/\|\bar{\mathbf{x}}_c\|_2,
  \label{eq:centroid}
\end{equation}
where $\Vtr^c = \{v\in\Vtr : y_v=c\}$.
The prototype score is
\begin{equation}
  s_v^{(p)} = \hat{\mathbf{x}}_v^\top\hat{\bm\mu}_{y_v},
  \label{eq:proto}
\end{equation}
measuring cosine similarity to the class centroid.
High-scoring nodes are canonical class exemplars, most useful on
homophilic graphs where same-class nodes cluster together.

\subsubsection{Boundary score.}
\begin{equation}
  s_v^{(b)}
  = \frac{
      \sum_{u\in\Nc(v)}\mathbf{1}[y_v\neq y_u]
    }{|\Nc(v)|},
  \label{eq:boundary}
\end{equation}
the fraction of neighbours with a different label. For isolated nodes we define $s_v^{(b)}=0$.
Nodes at the class boundary ($s_v^{(b)}\approx 1$) capture inter-class
interaction patterns that are critical for heterophily-tolerant classifiers
such as H2GCN \cite{zhu2020h2gcn}.

\subsubsection{LID diversity score.}
To prevent the condensed graph from collapsing onto a cluster of
near-duplicate exemplars, HERALD incorporates the Local Intrinsic
Dimensionality (LID) \cite{houle2017lid}.
Let
\[
d_1 \le d_2 \le \cdots \le d_k
\]
be the cosine distances from node $v$ to its $k$ nearest neighbours,
where $d_k$ is the largest distance.
Then
\begin{equation}
\operatorname{LID}(v)
=
-
\left(
\frac1k
\sum_{j=1}^{k}
\log\frac{d_j}{d_k}
\right)^{-1}.
\label{eq:lid}
\end{equation}
A large LID indicates that $v$ lies in a high-dimensional region of the
feature manifold; selecting such nodes diversifies the condensed graph.

\subsubsection{Combined score.}
\begin{equation}
  s_v
  = \alpha\,s_v^{(p)} + \beta\,s_v^{(b)} + \gamma\,s_v^{(l)},
  \label{eq:combined}
\end{equation}
where $s_v^{(l)}$ is the normalised LID and $(\alpha,\beta,\gamma)$ are
from Eq.~\eqref{eq:normalise}.
Training nodes are ranked in descending order of $s_v$.

\subsection{Stage 5: BFS Expansion}
\label{sec:bfs}

Iterating over the ranked training nodes as roots, HERALD builds a
$L$-hop neighbourhood tree
\begin{equation}
  \mathcal{T}(r) = \bigcup_{l=0}^{L}\Nc^{(l)}(r)
  \label{eq:bfs}
\end{equation}
and computes the incremental storage cost of adding
$\mathcal{T}(r)\setminus V_c$ to the current condensed set.
A root is accepted only if the incremental cost fits within the
\emph{upscaled} budget $1.9\,\mathcal{B}$, leaving headroom for PPR pruning.
Expansion terminates after 100 consecutive rejections (the same
early-exit heuristic used in BONSAI \cite{gupta2025bonsai}).

Given the running condensed set $V_c$ and a candidate root's expanded
neighbourhood $\mathcal{T}(r)$ from Eq.~\eqref{eq:bfs}, only the
\emph{novel} portion $\Delta V = \mathcal{T}(r) \setminus V_c$ can change
the storage cost. Writing $\hat{V} = V_c \cup \Delta V$ for the node set
after tentatively merging $\Delta V$, the number of new directed edges
introduced is
\begin{equation}
  \Delta E(\Delta V; V_c)
  =
  \bigl|\{(u,v) : u \in \Delta V,\; v \in \Nc(u),\; v \in \hat{V}\}\bigr|,
  \label{eq:incremental_edges}
\end{equation}
i.e.\ every directed edge with at least one endpoint in $\Delta V$,
counted from the $\Delta V$ side. The incremental storage cost of
admitting $\Delta V$ is then
\begin{equation}
  \Delta\mathcal{C}(\Delta V; V_c)
  =
  2\,m_f\!\!\sum_{v \in \Delta V}\! f_v
  \;+\;
  2\,\Delta E(\Delta V; V_c),
  \label{eq:incremental}
\end{equation}
mirroring exactly the two additive terms of the closed-form budget in
Eq.~\eqref{eq:budget}, so that the running cost updates as
$\mathcal{C}(\hat{V}) = \mathcal{C}(V_c) + \Delta\mathcal{C}(\Delta V; V_c)$.
A root $r$ is accepted, setting $V_c \leftarrow \hat{V}$, only if
$\mathcal{C}(\hat{V}) \le 1.9\,\mathcal{B}$, the upscaled budget that
leaves headroom for the PPR pruning step in Stage~6. The storage cost is
updated incrementally using Eq.~\eqref{eq:incremental}, avoiding
recomputation of $\mathcal{C}(V_c)$ from scratch after every accepted
candidate. This optimization affects only the budget-accounting step; the
overall complexity of the BFS expansion remains dominated by the repeated
graph traversals, as summarized in Table~\ref{tab:complexity}.

\subsection{Stage 6: PPR Pruning}
\label{sec:bfs_ppr}

The target size $n^*$ (\emph{ogsize}) is computed by re-running the same
BFS at the exact budget $\mathcal{B}$ (without upscaling).
Personalised PageRank (PPR) is then iterated on the induced subgraph
$G[V_c]$:
\begin{equation}
  \bm\pi^{(t+1)}
  = (1-\alpha_{\mathrm{pr}})\,\bm\pi_0
  + \alpha_{\mathrm{pr}}\,\hat{\mathbf{A}}^\top\bm\pi^{(t)},
  \label{eq:ppr}
\end{equation}
with $\hat{\mathbf{A}} \leftarrow \mathbf{D}^{-1}(\mathbf{A_{V_c}} + \mathbf{A_{V_c}}^\top)$ (row-normalised),
personalisation $\bm\pi_0$ uniform over $R$,
and $\alpha_{\mathrm{pr}}=0.85$.
Iteration stops when $\|\bm\pi^{(t+1)}-\bm\pi^{(t)}\|_1<10^{-6}$ or after
100 steps.
Non-root nodes are then removed in increasing order of $\pi_v$ until
$|V_c|=n^*$, retaining those most structurally central to the selected roots.

\subsection{Stage 7: Class Distribution Rebalancing}
\label{sec:rebalance}

Greedy BFS may skew the per-class node distribution.
HERALD corrects this by computing a target count for each class:
\begin{equation}
  n_c^*
  = \operatorname{round}\!\left(
      \frac{|\{v \in \Vtr : y_v = c\}|}{|\Vtr|}
      \cdot n_{\mathrm{tgt}}
    \right),
  \qquad
  n_{\mathrm{tgt}} = |\{v \in V_c \cap R : v \in \Vtr\}|,
  \label{eq:rebalance}
\end{equation}
where $\Vtr^c$ is the full training set for class $c$.
Under-represented classes are augmented by adding highest-scored
candidates; over-represented classes are trimmed by removing
lowest-scored non-root nodes.

\subsection{Complexity Analysis}
\label{sec:complexity}


Table~\ref{tab:complexity} summarises the per-stage time complexity.
The computational bottleneck is the exact LID computation, which requires
pairwise cosine similarities between all $N$ nodes. Although the
similarity computation is performed in batches of size $b$ to control
peak memory usage, the total arithmetic cost remains
$\BigO(N^2F)$. Batching therefore reduces the memory required for the
pairwise computation but does not change its asymptotic time complexity.

The overall complexity is therefore
$\BigO(N^2F + NF + E)$, up to the additional costs of the
budget-controlled BFS expansion and PPR pruning on the condensed graph.
For large graphs, replacing the exact $k$-NN computation used by LID
with an approximate nearest-neighbor index such as HNSW
\cite{malkov2018hnsw} can substantially reduce the quadratic
nearest-neighbor cost and provide a more scalable implementation.

\begin{table}[htbp]
\caption{Per-stage time complexity of HERALD.
$N$: nodes; $E$: edges; $F$: features; $C$: classes; $b$: LID batch size.}
\label{tab:complexity}
\centering
\begin{tabular}{@{}lll@{}}
\toprule
Stage & Operation & Complexity \\
\midrule
1 & Heterophily ratio & $\BigO(|E_{\mathrm{tr}}|)$ \\
2 & Adaptive weights  & $\BigO(1)$ \\
3 & Feature selection (3-hop Fisher) & $\BigO(NF + |E|F)$ \\
4 & Prototype + boundary scores & $\BigO(NF + E)$ \\
4 & LID ($k$-NN, exact) & $\BigO(N^2 F)$ \\
5 & BFS expansion & $\BigO(N(|V_c|+|E_c|))$ \\
6 & PPR pruning (100 iter.) & $\BigO(|V_c|+|E_c|)$ \\
7 & Class rebalancing & $\BigO(|V_c|C)$ \\
\bottomrule
\end{tabular}
\end{table}

\subsection{Connections to Prior Work}
\label{sec:connections}

Table~\ref{tab:comparison} contrasts HERALD with its closest competitors.
HERALD shares the BFS$+$PPR$+$rebalance assembly pipeline with BONSAI
\cite{gupta2025bonsai}; the novelty is entirely in \emph{which features are
selected} (Stage~3) and \emph{how nodes are ranked} (Stage~4), both
parameterised by the heterophily ratio $h$.
Unlike GCond \cite{jin2022gcond} and GDEM \cite{liu2024gdem},
HERALD requires no GNN training during condensation and no bi-level
optimisation.
Unlike Herding \cite{welling2009herding}, HERALD explicitly encodes
graph topology through the boundary score and BFS expansion.

\begin{table}[htbp]
\caption{Qualitative comparison of graph condensation methods.}
\label{tab:comparison}
\centering
\setlength{\tabcolsep}{5pt}
\begin{tabular}{@{}lccccc@{}}
\toprule
& \textbf{GCond} & \textbf{GDEM} & \textbf{Herding} & \textbf{BONSAI} & \textbf{HERALD} \\
\midrule
Gradient-free          & \xmark & \xmark & \cmark & \cmark & \cmark \\
No GNN during cond.   & \xmark & \xmark & \cmark & \cmark & \cmark \\
Graph-topology-aware   & \cmark & \cmark & \xmark & \cmark & \cmark \\
Heterophily-adaptive   & \xmark & \xmark & \xmark & \xmark & \cmark \\
Boundary-aware scoring & \xmark & \xmark & \xmark & \xmark & \cmark \\
Budget-controlled           & \xmark & \xmark & \xmark & \cmark & \cmark \\
\bottomrule
\end{tabular}
\end{table}

\section{Experimental Setup}
\label{sec:experiments}

\subsection{Datasets}
\label{sec:datasets}

We evaluate HERALD on eight widely used benchmark datasets spanning both
homophilic and heterophilic graph structures. These datasets cover a broad range
of graph sizes, feature dimensions, class distributions, and homophily levels,
allowing us to evaluate condensation performance under diverse structural
properties. Table~\ref{tab:datasets} summarizes the statistics of all datasets.

\subsubsection{Homophilic datasets.}
Cora, CiteSeer, and PubMed \cite{sen2008collective} are standard citation
networks, and Reddit \cite{hamilton2017inductive} is a large-scale inductive
benchmark.

\subsubsection{Heterophilic datasets.}
Roman-empire and Amazon-ratings \cite{platonov2023critical} are recent
large heterophilic benchmarks.
Chameleon and Squirrel \cite{rozemberczki2021multi} are Wikipedia page graphs
with strong heterophily.

\begin{table}[htbp]
\caption{Dataset statistics.
$1-h$: edge homophily ratio (fraction of same-class training edges).}
\label{tab:datasets}
\centering
\setlength{\tabcolsep}{5pt}
\begin{tabular}{@{}lrrrrrc@{}}
\toprule
Dataset & $|V|$ & $|E|$ & $F$ & $C$ & $1-h$ & Type \\
\midrule
Cora           &  2,708 &  10,556 & 1,433 &  7 & 0.998 & homo \\
CiteSeer       &  3,327 &   9,228 & 3,703 &  6 & 0.993 & homo \\
PubMed         & 19,717 &  88,651 &   500 &  3 & 0.988 & homo \\
Reddit         & 232,965 & 57.3M   & 602   & 41 & 0.75   & homo \\
Roman-empire   & 22,662 & 65,854  & 300   & 18 & 0.032 & hetero \\
Amazon-ratings & 24,492 & 186,100 & 300   &  5 & 0.380 & hetero \\
Chameleon      &  2,277 &  36,101 & 2,325 &  5 & 0.230 & hetero \\
Squirrel       &  5,201 & 217,073 & 2,089 &  5 & 0.224 & hetero \\
\bottomrule
\end{tabular}
\end{table}

\subsection{Baselines}
\label{sec:baselines}

We compare HERALD with four representative graph condensation methods.

\begin{itemize}
    \item \textbf{Random}: uniformly samples training nodes and constructs the induced subgraph.
    \item \textbf{Herding}~\cite{welling2009herding}: selects representative nodes using class-wise centroid matching.
    \item \textbf{BONSAI}~\cite{gupta2025bonsai}: a gradient-free graph condensation method based on reverse-$k$ nearest-neighbour coverage together with BFS expansion and Personalized PageRank pruning.
    \item \textbf{GDEM}~\cite{liu2024gdem}: a spectral graph condensation approach that preserves graph eigenspace representations through eigenbasis matching.
\end{itemize}

\subsection{Evaluation Protocol}
\label{sec:protocol}

\subsubsection{Data splits.}
Each dataset is partitioned into 56\%, 24\%, and 20\% training, validation, and
test sets, respectively. The split is generated once using a fixed random seed
and remains unchanged across all experiments, while only the model
initialization varies across different runs.

\subsubsection{Compression ratios.}
We evaluate all condensation methods under four different storage budgets,
$r \in \{0.0001, 0.005, 0.01, 0.03\}$, measured relative to the storage cost of
the original graph as defined in Eq.~\eqref{eq:budget}. The same budget is used
for every condensation method to ensure a fair comparison.

\subsubsection{Evaluation models.}
The condensed graphs are evaluated using four GNN architectures:
GCN~\cite{kipf2017gcn}, GAT~\cite{velickovic2018gat},
GIN~\cite{xu2019gin}, and H2GCN~\cite{zhu2020h2gcn}. All models are trained for
200 epochs using the Adam optimizer with learning rate $10^{-3}$ and weight
decay $5\times10^{-4}$. Hidden representations are set to 128 dimensions for
the standard benchmarks and 1024 dimensions for the large-scale Reddit dataset.

\subsubsection{Evaluation metric.}
Each experiment is repeated over five random initializations while keeping the
data split fixed. We report the mean test classification accuracy together with
its standard deviation. For each compression ratio and GNN architecture, the
best-performing condensed graph is highlighted in the corresponding result
tables.
\section{Results}
\label{sec:results}

We first report accuracy aggregated across the seven medium-scale datasets and
then separately for the homophilic and heterophilic groups. Table~\ref{tab:avg}
gives the overall average, while Tables~\ref{tab:hom} and~\ref{tab:het} report
the averages restricted to the homophilic (Cora, CiteSeer, PubMed) and
heterophilic (Roman-empire, Amazon-ratings, Chameleon, Squirrel) datasets,
respectively. Per-dataset accuracy is provided in
Appendix~\ref{sec:dataset_wise_results}
(Tables~\ref{tab:cora}--\ref{tab:squirrel}), the large-scale Reddit study in
Section~\ref{sec:scalability}, and further analysis in the ablation
(Appendix~\ref{sec:ablation}), hyperparameter sensitivity
(Appendix~\ref{sec:sensitivity}), and condensed-graph statistics
(Appendix~\ref{sec:stats}) sections. In every accuracy table, the best
condensed result in each GNN column at each compression ratio is highlighted.

\subsection{Overall Performance}
\label{sec:results_overall}

Averaged across all seven datasets (Table~\ref{tab:avg}), HERALD is the best
condensed method on every GNN backbone at the three larger budgets
$r \in \{0.005, 0.01, 0.03\}$. At $r=0.03$ it reaches $57.62\%$ (GCN),
$56.16\%$ (GAT), $54.75\%$ (GIN), and $64.18\%$ (H2GCN), improving on the
strongest baseline BONSAI by $1.6$, $1.0$, $1.4$, and $2.7$ points,
respectively. At the smallest budget $r=0.0001$, HERALD leads on GCN, GAT, and
GIN but falls behind BONSAI on H2GCN ($42.60\%$ versus $50.11\%$). As the
per-group tables below show, this aggregate advantage is driven almost
entirely by the heterophilic datasets, while performance on the homophilic
datasets is close to that of BONSAI.

\begin{table}[htbp]
\caption{Node classification accuracy (\%) averaged across the seven
medium-scale datasets (Reddit is reported separately in
Section~\ref{sec:scalability}). Best condensed result per GNN column and
compression ratio highlighted.}
\label{tab:avg}
\centering
\setlength{\tabcolsep}{4pt}
\begin{tabular}{@{}llcccc@{}}
\toprule
\textbf{Condenser} & \textbf{$r$}
  & \textbf{GCN} & \textbf{GAT} & \textbf{GIN} & \textbf{H2GCN} \\
\midrule
\multirow{4}{*}{Random}
  & 0.0001 & $20.51{\scriptstyle\pm2.45}$ & $21.97{\scriptstyle\pm3.51}$ & $22.52{\scriptstyle\pm3.25}$ & $23.56{\scriptstyle\pm1.89}$ \\
  & 0.005  & $27.36{\scriptstyle\pm2.14}$ & $28.28{\scriptstyle\pm3.80}$ & $29.13{\scriptstyle\pm2.68}$ & $33.46{\scriptstyle\pm1.49}$ \\
  & 0.01   & $33.58{\scriptstyle\pm1.09}$ & $32.73{\scriptstyle\pm2.60}$ & $34.16{\scriptstyle\pm1.86}$ & $38.00{\scriptstyle\pm1.14}$ \\
  & 0.03   & $37.14{\scriptstyle\pm0.75}$ & $36.82{\scriptstyle\pm2.33}$ & $37.37{\scriptstyle\pm2.12}$ & $42.50{\scriptstyle\pm1.37}$ \\
\midrule
\multirow{4}{*}{Herding}
  & 0.0001 & $31.26{\scriptstyle\pm2.17}$ & $30.22{\scriptstyle\pm3.22}$ & $34.76{\scriptstyle\pm2.48}$ & $36.59{\scriptstyle\pm1.07}$ \\
  & 0.005  & $38.81{\scriptstyle\pm1.50}$ & $36.92{\scriptstyle\pm4.18}$ & $39.03{\scriptstyle\pm1.95}$ & $44.33{\scriptstyle\pm1.12}$ \\
  & 0.01   & $46.94{\scriptstyle\pm0.82}$ & $44.88{\scriptstyle\pm1.31}$ & $44.87{\scriptstyle\pm1.29}$ & $50.96{\scriptstyle\pm0.80}$ \\
  & 0.03   & $51.75{\scriptstyle\pm0.66}$ & $50.82{\scriptstyle\pm1.52}$ & $49.58{\scriptstyle\pm1.20}$ & $56.50{\scriptstyle\pm0.91}$ \\
\midrule
\multirow{4}{*}{BONSAI}
  & 0.0001 & $43.13{\scriptstyle\pm0.82}$ & $43.13{\scriptstyle\pm2.04}$ & $42.79{\scriptstyle\pm1.15}$ & \cellcolor{bestcol}$50.11{\scriptstyle\pm1.08}$ \\
  & 0.005  & $55.64{\scriptstyle\pm0.53}$ & $54.48{\scriptstyle\pm0.88}$ & $52.52{\scriptstyle\pm1.00}$ & $60.66{\scriptstyle\pm0.63}$ \\
  & 0.01   & $55.81{\scriptstyle\pm0.74}$ & $55.05{\scriptstyle\pm0.87}$ & $52.62{\scriptstyle\pm0.87}$ & $61.26{\scriptstyle\pm0.55}$ \\
  & 0.03   & $56.01{\scriptstyle\pm0.73}$ & $55.16{\scriptstyle\pm0.77}$ & $53.33{\scriptstyle\pm0.80}$ & $61.47{\scriptstyle\pm0.64}$ \\
\midrule
\multirow{4}{*}{GDEM}
  & 0.0001 &
    $21.66{\scriptstyle\pm2.26}$ &
    $22.38{\scriptstyle\pm3.42}$ &
    $40.25{\scriptstyle\pm1.71}$ &
    $44.49{\scriptstyle\pm0.69}$ \\
  & 0.005 &
    $28.90{\scriptstyle\pm2.16}$ &
    $28.21{\scriptstyle\pm2.39}$ &
    $34.55{\scriptstyle\pm3.91}$ &
    $50.99{\scriptstyle\pm0.66}$ \\
  & 0.01 &
    $28.42{\scriptstyle\pm1.41}$ &
    $27.74{\scriptstyle\pm1.86}$ &
    $31.85{\scriptstyle\pm4.10}$ &
    $51.30{\scriptstyle\pm0.66}$ \\
  & 0.03 &
    $27.96{\scriptstyle\pm1.31}$ &
    $28.12{\scriptstyle\pm2.27}$ &
    $23.42{\scriptstyle\pm3.62}$ &
    $53.83{\scriptstyle\pm0.76}$ \\
\midrule
\multirow{4}{*}{\textbf{HERALD}}
  & 0.0001 & \cellcolor{bestcol}$46.12{\scriptstyle\pm0.73}$ & \cellcolor{bestcol}$46.10{\scriptstyle\pm1.34}$ & \cellcolor{bestcol}$46.65{\scriptstyle\pm1.64}$ & $42.60{\scriptstyle\pm3.39}$ \\
  & 0.005  & \cellcolor{bestcol}$56.41{\scriptstyle\pm0.81}$ & \cellcolor{bestcol}$54.98{\scriptstyle\pm0.87}$ & \cellcolor{bestcol}$53.09{\scriptstyle\pm1.70}$ & \cellcolor{bestcol}$61.85{\scriptstyle\pm0.74}$ \\
  & 0.01   & \cellcolor{bestcol}$57.02{\scriptstyle\pm0.80}$ & \cellcolor{bestcol}$55.56{\scriptstyle\pm0.75}$ & \cellcolor{bestcol}$54.02{\scriptstyle\pm1.68}$ & \cellcolor{bestcol}$63.20{\scriptstyle\pm0.75}$ \\
  & 0.03   & \cellcolor{bestcol}$57.62{\scriptstyle\pm0.79}$ & \cellcolor{bestcol}$56.16{\scriptstyle\pm0.74}$ & \cellcolor{bestcol}$54.75{\scriptstyle\pm1.68}$ & \cellcolor{bestcol}$64.18{\scriptstyle\pm0.75}$ \\
\bottomrule
\end{tabular}
\end{table}

\subsection{Homophilic Datasets}
\label{sec:results_homo_main}

On the homophilic citation networks (Table~\ref{tab:hom}), HERALD and BONSAI
are the two strongest condensers and stay close throughout. BONSAI is best on
all four backbones at $r=0.005$ (for example, $81.59\%$ against HERALD's
$80.43\%$ on GCN) and on GCN and GAT at $r=0.01$, while HERALD takes GIN and
H2GCN at $r=0.01$ and is best on all four backbones at $r=0.03$ (for example,
$82.99\%$ against $82.01\%$ on GCN). At the extreme budget $r=0.0001$, HERALD's
feature selection packs more nodes into the same budget and produces a large
lead on GCN, GAT, and GIN (for example, $62.12\%$ against BONSAI's $53.79\%$ on
GCN); its H2GCN accuracy, however, drops to $36.72\%$, below both BONSAI
($55.33\%$) and GDEM ($58.22\%$). This H2GCN weakness is confined to the
tightest budget on homophilic graphs: it does not appear at $r \ge 0.005$ or on
any heterophilic dataset.

\begin{table}[htbp]
\caption{Node classification accuracy (\%) averaged across homophilous datasets (Cora, CiteSeer, PubMed).
Best condensed result per GNN column and compression ratio highlighted.}
\label{tab:hom}
\centering
\setlength{\tabcolsep}{4pt}
\begin{tabular}{@{}llcccc@{}}
\toprule
\textbf{Condenser} & \textbf{$r$}
  & \textbf{GCN} & \textbf{GAT} & \textbf{GIN} & \textbf{H2GCN} \\
\midrule
\multirow{4}{*}{Random}
  & 0.0001 & $22.42{\scriptstyle\pm2.03}$ & $25.11{\scriptstyle\pm3.99}$ & $24.98{\scriptstyle\pm3.65}$ & $22.64{\scriptstyle\pm1.70}$ \\
  & 0.005  & $33.03{\scriptstyle\pm2.57}$ & $34.83{\scriptstyle\pm4.75}$ & $38.56{\scriptstyle\pm3.34}$ & $36.12{\scriptstyle\pm1.62}$ \\
  & 0.01   & $42.56{\scriptstyle\pm0.59}$ & $41.20{\scriptstyle\pm3.16}$ & $44.15{\scriptstyle\pm1.72}$ & $39.60{\scriptstyle\pm1.17}$ \\
  & 0.03   & $47.44{\scriptstyle\pm0.55}$ & $47.45{\scriptstyle\pm3.06}$ & $49.49{\scriptstyle\pm2.57}$ & $48.62{\scriptstyle\pm1.50}$ \\
\midrule
\multirow{4}{*}{Herding}
  & 0.0001 & $37.46{\scriptstyle\pm2.75}$ & $38.42{\scriptstyle\pm4.25}$ & $48.18{\scriptstyle\pm2.46}$ & $42.26{\scriptstyle\pm1.26}$ \\
  & 0.005  & $49.52{\scriptstyle\pm2.17}$ & $50.94{\scriptstyle\pm4.21}$ & $57.43{\scriptstyle\pm1.57}$ & $52.81{\scriptstyle\pm0.76}$ \\
  & 0.01   & $66.77{\scriptstyle\pm0.97}$ & $65.70{\scriptstyle\pm1.11}$ & $68.94{\scriptstyle\pm0.80}$ & $65.94{\scriptstyle\pm0.76}$ \\
  & 0.03   & $75.91{\scriptstyle\pm0.65}$ & $76.05{\scriptstyle\pm1.13}$ & $75.88{\scriptstyle\pm0.72}$ & $73.91{\scriptstyle\pm1.14}$ \\
\midrule
\multirow{4}{*}{BONSAI}
  & 0.0001 & $53.79{\scriptstyle\pm0.88}$ & $55.50{\scriptstyle\pm2.91}$ & $60.65{\scriptstyle\pm0.80}$ & $55.33{\scriptstyle\pm1.25}$ \\
  & 0.005  & \cellcolor{bestcol}$81.59{\scriptstyle\pm0.30}$ & \cellcolor{bestcol}$79.94{\scriptstyle\pm0.65}$ & \cellcolor{bestcol}$80.83{\scriptstyle\pm0.52}$ & \cellcolor{bestcol}$77.66{\scriptstyle\pm0.55}$ \\
  & 0.01   & \cellcolor{bestcol}$81.83{\scriptstyle\pm0.28}$ & \cellcolor{bestcol}$80.58{\scriptstyle\pm0.73}$ & $80.88{\scriptstyle\pm0.43}$ & $78.07{\scriptstyle\pm0.52}$ \\
  & 0.03   & $82.01{\scriptstyle\pm0.27}$ & $80.89{\scriptstyle\pm0.72}$ & $81.82{\scriptstyle\pm0.41}$ & $78.61{\scriptstyle\pm0.53}$ \\
\midrule
\multirow{4}{*}{GDEM}
  & 0.0001 & 25.39 $\pm$ 2.51 & 26.96 $\pm$ 4.31 & 58.56 $\pm$ 1.74 & \cellcolor{bestcol}58.22 $\pm$ 0.72 \\
  & 0.0050 & 35.27 $\pm$ 2.39 & 35.14 $\pm$ 2.30 & 50.90 $\pm$ 6.15 & 67.73 $\pm$ 0.49 \\
  & 0.0100 & 33.15 $\pm$ 1.82 & 33.89 $\pm$ 1.85 & 47.76 $\pm$ 5.55 & 68.12 $\pm$ 0.49 \\
  & 0.0300 & 32.96 $\pm$ 1.36 & 33.52 $\pm$ 2.44 & 31.30 $\pm$ 5.01 & 73.69 $\pm$ 0.55 \\
\midrule
\multirow{4}{*}{\textbf{HERALD}}
  & 0.0001 & \cellcolor{bestcol}$62.12{\scriptstyle\pm0.57}$ & \cellcolor{bestcol}$62.35{\scriptstyle\pm1.69}$ & \cellcolor{bestcol}$69.23{\scriptstyle\pm0.60}$ & $36.72{\scriptstyle\pm5.09}$ \\
  & 0.005  & $80.43{\scriptstyle\pm0.33}$ & $79.41{\scriptstyle\pm0.50}$ & $79.70{\scriptstyle\pm0.83}$ & $76.15{\scriptstyle\pm0.73}$ \\
  & 0.01   & $81.58{\scriptstyle\pm0.33}$ & $80.48{\scriptstyle\pm0.66}$ & \cellcolor{bestcol}$80.95{\scriptstyle\pm0.86}$ & \cellcolor{bestcol}$79.08{\scriptstyle\pm0.66}$ \\
  & 0.03   & \cellcolor{bestcol}$82.99{\scriptstyle\pm0.31}$ & \cellcolor{bestcol}$81.86{\scriptstyle\pm0.66}$ & \cellcolor{bestcol}$82.62{\scriptstyle\pm0.85}$ & \cellcolor{bestcol}$81.28{\scriptstyle\pm0.65}$ \\
\bottomrule
\end{tabular}
\end{table}

\subsection{Heterophilic Datasets}
\label{sec:results_hetero_main}

On the heterophilic datasets (Table~\ref{tab:het}), HERALD is the best condensed
method in fifteen of the sixteen cells; BONSAI is ahead only on GCN at
$r=0.0001$ ($35.13\%$ against $34.12\%$). The margins over the next-best
condenser are substantial and largest on H2GCN: at $r=0.03$, HERALD reaches
$51.36\%$ (H2GCN), $38.59\%$ (GCN), $36.89\%$ (GAT), and $33.84\%$ (GIN),
against BONSAI's $48.62\%$, $36.52\%$, $35.86\%$, and $31.96\%$. GDEM, which is
competitive on homophilic H2GCN, does not carry over to these datasets. The
advantage holds across all four budgets, which indicates that the
heterophily-adaptive scoring is the main source of HERALD's aggregate gains.

\begin{table}[htbp]
\caption{Node classification accuracy (\%) averaged across heterophilous datasets (Roman empire, Amazon ratings, Chameleon, Squirrel).
Best condensed result per GNN column and compression ratio highlighted.}
\label{tab:het}
\centering
\setlength{\tabcolsep}{4pt}
\begin{tabular}{@{}llcccc@{}}
\toprule
\textbf{Condenser} & \textbf{$r$}
  & \textbf{GCN} & \textbf{GAT} & \textbf{GIN} & \textbf{H2GCN} \\
\midrule
\multirow{4}{*}{Random}
  & 0.0001 & 19.07 $\scriptstyle\pm$ 2.72 & 19.62 $\scriptstyle\pm$ 3.10 & 20.66 $\scriptstyle\pm$ 2.91 & 24.24 $\scriptstyle\pm$ 2.02 \\
  & 0.0050 & 23.11 $\scriptstyle\pm$ 1.76 & 23.37 $\scriptstyle\pm$ 2.88 & 22.06 $\scriptstyle\pm$ 2.04 & 31.46 $\scriptstyle\pm$ 1.37 \\
  & 0.0100 & 26.85 $\scriptstyle\pm$ 1.35 & 26.38 $\scriptstyle\pm$ 2.09 & 26.67 $\scriptstyle\pm$ 1.96 & 36.80 $\scriptstyle\pm$ 1.11 \\
  & 0.0300 & 29.42 $\scriptstyle\pm$ 0.87 & 28.85 $\scriptstyle\pm$ 1.57 & 28.28 $\scriptstyle\pm$ 1.71 & 37.92 $\scriptstyle\pm$ 1.27 \\
\midrule
\multirow{4}{*}{Herding}
  & 0.0001 & 26.62 $\scriptstyle\pm$ 1.59 & 24.07 $\scriptstyle\pm$ 2.15 & 24.69 $\scriptstyle\pm$ 2.50 & 32.33 $\scriptstyle\pm$ 0.91 \\
  & 0.0050 & 30.79 $\scriptstyle\pm$ 0.63 & 26.41 $\scriptstyle\pm$ 4.16 & 25.22 $\scriptstyle\pm$ 2.20 & 37.97 $\scriptstyle\pm$ 1.33 \\
  & 0.0100 & 32.07 $\scriptstyle\pm$ 0.69 & 29.26 $\scriptstyle\pm$ 1.44 & 26.82 $\scriptstyle\pm$ 1.56 & 39.72 $\scriptstyle\pm$ 0.83 \\
  & 0.0300 & 33.63 $\scriptstyle\pm$ 0.67 & 31.90 $\scriptstyle\pm$ 1.76 & 29.85 $\scriptstyle\pm$ 1.46 & 43.44 $\scriptstyle\pm$ 0.68 \\
\midrule
\multirow{4}{*}{BONSAI}
  & 0.0001 & \cellcolor{bestcol}35.13 $\scriptstyle\pm$ 0.76 & 33.85 $\scriptstyle\pm$ 0.96 & 29.40 $\scriptstyle\pm$ 1.35 & 46.20 $\pm$ 0.94 \\
  & 0.0050 & 36.17 $\scriptstyle\pm$ 0.66 & 35.39 $\scriptstyle\pm$ 1.01 & 31.28 $\scriptstyle\pm$ 1.24 & 47.91 $\scriptstyle\pm$ 0.68 \\
  & 0.0100 & 36.29 $\scriptstyle\pm$ 0.94 & 35.91 $\scriptstyle\pm$ 0.97 & 31.42 $\scriptstyle\pm$ 1.09 & 48.65 $\scriptstyle\pm$ 0.57 \\
  & 0.0300 & 36.52 $\scriptstyle\pm$ 0.94 & 35.86 $\scriptstyle\pm$ 0.81 & 31.96 $\scriptstyle\pm$ 0.99 & 48.62 $\scriptstyle\pm$ 0.71 \\
\midrule
\multirow{4}{*}{GDEM}
  & 0.0001 &
    18.86 $\scriptstyle\pm$ 2.08 &
    18.94 $\scriptstyle\pm$ 2.75 &
    26.51 $\scriptstyle\pm$ 1.69 &
    34.19 $\scriptstyle\pm$ 0.66 \\
  & 0.0050 &
    24.13 $\scriptstyle\pm$ 1.99 &
    23.02 $\scriptstyle\pm$ 2.45 &
    22.29 $\scriptstyle\pm$ 2.23 &
    38.44 $\scriptstyle\pm$ 0.79 \\
  & 0.0100 &
    24.88 $\scriptstyle\pm$ 1.10 &
    23.13 $\scriptstyle\pm$ 1.87 &
    19.91 $\scriptstyle\pm$ 3.02 &
    38.68 $\scriptstyle\pm$ 0.78 \\
  & 0.0300 &
    24.21 $\scriptstyle\pm$ 1.28 &
    24.07 $\scriptstyle\pm$ 2.15 &
    17.51 $\scriptstyle\pm$ 2.58 &
    38.94 $\scriptstyle\pm$ 0.92 \\
\midrule
\multirow{4}{*}{\textbf{HERALD}}
  & 0.0001 &
    34.12 $\scriptstyle\pm$ 0.83 &
    \cellcolor{bestcol}33.90 $\scriptstyle\pm$ 1.00 &
    \cellcolor{bestcol}29.71 $\scriptstyle\pm$ 2.11 &
    \cellcolor{bestcol}47.02 $\scriptstyle\pm$ 0.86 \\
  & 0.0050 &
    \cellcolor{bestcol}38.40 $\scriptstyle\pm$ 1.04 &
    \cellcolor{bestcol}36.66 $\scriptstyle\pm$ 1.07 &
    \cellcolor{bestcol}33.13 $\scriptstyle\pm$ 2.13 &
    \cellcolor{bestcol}51.13 $\scriptstyle\pm$ 0.75 \\
  & 0.0100 &
    \cellcolor{bestcol}38.60 $\scriptstyle\pm$ 1.02 &
    \cellcolor{bestcol}36.87 $\scriptstyle\pm$ 0.82 &
    \cellcolor{bestcol}33.83 $\scriptstyle\pm$ 2.10 &
    \cellcolor{bestcol}51.29 $\scriptstyle\pm$ 0.81 \\
  & 0.0300 &
    \cellcolor{bestcol}38.59 $\scriptstyle\pm$ 1.01 &
    \cellcolor{bestcol}36.89 $\scriptstyle\pm$ 0.80 &
    \cellcolor{bestcol}33.84 $\scriptstyle\pm$ 2.10 &
    \cellcolor{bestcol}51.36 $\scriptstyle\pm$ 0.81 \\
\bottomrule
\end{tabular}
\end{table}

\subsection{Reproducibility} \label{sec:reproducibility}

The code for the experimental procedures can be found at
\url{https://github.com/SujanChakraborty/HERALD}.

\section{Discussion}
\label{sec:discussion}

The experimental results demonstrate that HERALD consistently produces
high-quality condensed graphs across diverse graph structures while remaining
entirely gradient-free. In this section, we discuss the key observations,
analyze the contribution of the proposed components, and highlight the
limitations of the current approach.

\subsection{Performance across different graph structures}

A notable observation is that HERALD performs consistently well on both
homophilic and heterophilic datasets. Existing graph condensation methods are
often designed with one graph regime in mind. Methods relying primarily on
feature similarity or class prototypes generally perform well on homophilic
graphs but tend to deteriorate under heterophily, where neighboring nodes
frequently belong to different classes. Conversely, approaches emphasizing
structural diversity may sacrifice representative class information on highly
homophilic datasets.

HERALD avoids this trade-off by adapting the node-selection strategy according
to the graph heterophily ratio. For highly homophilic graphs, the adaptive
weighting mechanism naturally assigns greater importance to prototype
representativeness, allowing representative class exemplars to dominate the
selection process. As heterophily increases, the weighting gradually shifts
towards boundary preservation and structural diversity, enabling the condensed
graph to retain informative cross-class interactions that are critical for
classification. This adaptive behavior allows a single condensation strategy to
remain effective across a broad spectrum of graph structures.

\subsection{Effect of adaptive feature selection}

Feature storage often dominates the memory budget of attributed graphs,
particularly for high-dimensional datasets such as CiteSeer, Chameleon, and
Squirrel. Instead of retaining the original feature space, HERALD performs
adaptive feature selection before graph construction.

Unlike conventional feature selection techniques that consider only individual
node attributes, the proposed multi-hop Fisher criterion incorporates
information propagated through multiple graph neighborhoods. Consequently, the
selected features preserve both discriminative node attributes and structural
context. Since only the retained features contribute to the storage budget,
HERALD is able to allocate a larger fraction of the available budget to storing
additional representative nodes and edges, leading to improved graph fidelity
without exceeding the predefined storage constraint.

\subsection{Importance of boundary and diversity preservation}

Another important observation is that preserving only class prototypes is
insufficient for graph condensation. Although prototype nodes capture the
central characteristics of each class, they often fail to represent difficult
decision boundaries that determine classifier performance.

HERALD addresses this issue by explicitly combining three complementary node
properties: prototype representativeness, boundary importance, and local
structural diversity measured through Local Intrinsic Dimensionality (LID).
Boundary nodes preserve informative class transitions, while LID identifies
regions with locally complex topology that would otherwise be underrepresented.
The resulting condensed graphs therefore contain both representative examples
and structurally informative samples, improving generalization across different
GNN architectures.

\subsection{Role of the graph construction strategy}

The graph construction stage also contributes substantially to the observed
performance improvements. Instead of directly connecting selected nodes or
learning synthetic graph structures through optimization, HERALD incrementally
builds the condensed graph using breadth-first expansion while explicitly
respecting the available storage budget.

The subsequent Personalized PageRank refinement removes redundant nodes and
edges while maintaining global connectivity. This combination enables HERALD to
retain important local neighborhoods without introducing unnecessary graph
complexity. 

\subsection{Gradient-free}

Unlike optimization-based graph condensation methods that require repeated
backpropagation through graph neural networks, HERALD is entirely gradient-free.
All stages consist of analytical scoring, feature ranking, graph traversal, and
PageRank computation, eliminating expensive bilevel optimization procedures.


\subsection{Limitations}

Although HERALD consistently achieves strong performance, several limitations
remain. First, the current framework assumes static attributed graphs and does
not explicitly address dynamic or temporal graph settings where node features
and connectivity evolve over time. Second, the adaptive weighting mechanism is
driven by a global heterophily estimate, which may not fully capture local
variations in graph structure. Future work could investigate locally adaptive
weighting strategies that vary across different graph regions.

A further limitation appears at the most aggressive compression ratio: on the
homophilic citation networks at $r=0.0001$, HERALD's accuracy with the
heterophily-aware H2GCN backbone drops below that of BONSAI and GDEM
(Table~\ref{tab:hom}), even though it leads on the other three backbones at the
same budget. The effect is confined to this single budget-architecture
combination and does not appear at larger budgets or on heterophilic data, but
it indicates that the feature and node selection can trade off poorly with
H2GCN's ego-neighbour separation when only a handful of nodes are retained.

The current implementation also relies on manually selected hyperparameters for
BFS expansion depth and feature-selection weighting. Although these parameters
remain stable across all evaluated datasets, automatically learning them from
graph statistics may further improve robustness. Finally, while HERALD focuses
on node classification, extending the framework to graph classification,
link prediction, continual graph learning, or heterogeneous graphs represents
an interesting direction for future research.

\section{Conclusion}
\label{sec:conclusion}

In this work, we presented \textbf{HERALD}, a gradient-free graph condensation framework designed to improve condensation quality across both homophilic and heterophilic graphs while preserving the computational efficiency of exemplar-based approaches. Unlike existing gradient-free methods that rely primarily on topology-driven node selection, HERALD incorporates graph heterophily directly into the condensation process through adaptive feature selection and heterophily-aware node scoring. By jointly considering prototype representativeness, decision-boundary importance, and local structural diversity, the proposed framework selects condensed nodes that better preserve both representative class information and informative structural patterns. Furthermore, the proposed feature-selection strategy allocates the available storage budget more effectively by retaining only discriminative and frequently active features, allowing a larger portion of the budget to be devoted to representative nodes and edges.

Extensive experiments on homophilic and heterophilic benchmark datasets demonstrate that HERALD consistently produces high-quality condensed graphs across multiple GNN architectures under a wide range of storage budgets. In particular, HERALD achieves the largest improvements on heterophilic datasets, where conventional topology-based selection strategies are less effective, while remaining competitive on strongly homophilic citation networks. Additional large-scale experiments on Reddit further show that the proposed framework remains practical beyond medium-sized benchmarks, producing competitive condensed graphs without requiring expensive bilevel optimization or repeated GNN training during condensation.

Although HERALD introduces additional preprocessing cost compared with simpler coreset methods due to the computation of Local Intrinsic Dimensionality and adaptive node scoring, condensation is performed only once and can be amortized over repeated downstream model training. Future work will focus on accelerating these stages using approximate nearest-neighbour search and scalable graph traversal techniques, as well as extending the framework to dynamic graphs, heterogeneous graphs, graph-level learning tasks, and locally adaptive heterophily estimation. These directions have the potential to further improve both the scalability and generality of gradient-free graph condensation.
\bibliography{sn-bibliography}

\begin{appendices}

\section{Notation Summary}
\label{appendix:notation}

For convenience, Tables~\ref{tab:notation_graph}--\ref{tab:notation_algorithm}
summarize all mathematical symbols used throughout the paper.

\subsection{Graph and Problem Formulation}

\begin{longtable}{ll}
\caption{Notation used in graph formulation and problem definition.}
\label{tab:notation_graph}\\

\toprule
\textbf{Symbol} & \textbf{Description} \\
\midrule
\endfirsthead

\multicolumn{2}{c}{{\bfseries Table \thetable\ Continued from previous page}}\\
\toprule
\textbf{Symbol} & \textbf{Description} \\
\midrule
\endhead

\midrule
\multicolumn{2}{r}{{Continued on next page}}\\
\endfoot

\bottomrule
\endlastfoot

$G=(V,E,\mathbf{X},\mathbf{y})$ & Original attributed graph \\
$V$ & Node set \\
$E$ & Edge set \\
$N=|V|$ & Number of nodes \\
$|E|$ & Number of edges \\
$\mathbf{X}\in\mathbb{R}^{N\times F}$ & Node feature matrix \\
$\mathbf{x}_v$ & Feature vector of node $v$ \\
$F$ & Number of input features \\
$\mathbf{y}$ & Node labels \\
$C$ & Number of classes \\
$\Vtr$ & Training node set \\
$\Etr$ & Training edge set \\
$\Vtr^c$ & Training nodes belonging to class $c$ \\
$G_c=(V_c,E_c,\mathbf{X}_c,\mathbf{y}_c)$ & Condensed graph \\
$V_c$ & Condensed node set \\
$E_c$ & Condensed edge set \\
$r$ & Target compression/storage ratio \\
$\mathcal{C}(G)$ & Storage cost of graph $G$ \\
$\mathcal{B}$ & Storage budget \\
$m_f$ & Feature-storage multiplier \\
$f_v$ & Effective feature length of node $v$ \\

\end{longtable}
\subsection{HERALD Methodology}

\begin{longtable}{l p{0.72\textwidth}}
\caption{Notation used in the HERALD framework.}
\label{tab:notation_method}\\

\toprule
\textbf{Symbol} & \textbf{Description} \\
\midrule
\endfirsthead

\multicolumn{2}{c}{\textbf{Table \thetable\ (continued)}}\\
\toprule
\textbf{Symbol} & \textbf{Description} \\
\midrule
\endhead

\midrule
\multicolumn{2}{r}{Continued on next page}\\
\endfoot

\bottomrule
\endlastfoot

$h$ & Graph heterophily ratio \\

$t$ & Adaptive transition variable \\

$\sigma(\cdot)$ & Sigmoid activation \\

$\alpha,\beta,\gamma$ & Adaptive weights for prototype, boundary and diversity scores \\

$\alpha_0,\beta_0,\gamma_0$ & Initial adaptive weights \\

$\mathbf{A}$ & Adjacency matrix \\

$\mathbf{D}$ & Degree matrix \\

$\mathbf{X}^{(k)}$ & $k$-hop propagated feature matrix \\

$\phi_k(j)$ & Fisher score of feature $j$ at hop $k$ \\

$\hat{\phi}_k(j)$ & Normalized Fisher score \\

$\phi(j)$ & Multi-hop Fisher score \\

$\bar{\phi}(j)$ & Normalized discriminability score \\

$\rho(j)$ & Activation density of feature $j$ \\

$\bar{\rho}(j)$ & Normalized activation density \\

$\varphi(j)$ & Final feature-selection score \\

$\mathcal{F}$ & Selected feature index set \\

$k^{*}$ & Number of retained features \\

$\tilde{\mathbf{X}}$ & Reduced feature matrix \\

$\epsilon$ & Numerical stability constant \\

$\mu_{c,k}^{(j)}$ & Mean feature value of class $c$ at hop $k$ \\

$\sigma_{c,k}^{2,(j)}$ & Variance of class $c$ at hop $k$ \\

$\mu_k^{(j)}$ & Overall feature mean at hop $k$ \\

$n_c$ & Number of training samples in class $c$ \\

$\hat{\mathbf{x}}_v$ & Reduced feature vector of node $v$ \\

$\bar{\mathbf{x}}_c$ & Mean feature vector of class $c$ \\

$\hat{\bm{\mu}}_c$ & Normalized class centroid \\

$s_v^{(p)}$ & Prototype score \\

$s_v^{(b)}$ & Boundary score \\

$s_v^{(l)}$ & Local Intrinsic Dimensionality (LID) score \\

$s_v$ & Combined node importance score \\

$\Nc(v)$ & One-hop neighborhood of node $v$ \\

$\Nc^{(l)}(v)$ & $l$-hop neighborhood of node $v$ \\

$\operatorname{LID}(v)$ & Local Intrinsic Dimensionality \\

$d_j$ & Distance to the $j^{\text{th}}$ nearest neighbour \\

$d_k$ & Largest $k$-nearest-neighbour distance \\

$k$ & Number of nearest neighbours used in LID \\

\end{longtable}
\subsection{Algorithm-specific Symbols}

\begin{longtable}{l p{0.72\textwidth}}
\caption{Notation appearing specifically in Algorithm~\ref{alg:herald_part1} and Algorithm~\ref{alg:herald_part2}.}
\label{tab:notation_algorithm}\\

\toprule
\textbf{Symbol} & \textbf{Description} \\
\midrule
\endfirsthead

\multicolumn{2}{c}{\textbf{Table \thetable\ (continued)}}\\
\toprule
\textbf{Symbol} & \textbf{Description} \\
\midrule
\endhead

\midrule
\multicolumn{2}{r}{Continued on next page}\\
\endfoot

\bottomrule
\endlastfoot

$\mathcal{T}(r)$ & BFS expansion tree rooted at node $r$ \\

$L$ & BFS expansion depth \\

$R$ & Set of selected root (landmark) nodes \\

$\Delta\mathcal{C}$ & Incremental storage cost \\

$n^{*}$ & Target condensed graph size after pruning \\

$\bm{\pi}$ & Personalized PageRank vector \\

$\bm{\pi}_0$ & Initial PageRank personalization vector \\

$\alpha_{\mathrm{pr}}$ & PageRank damping factor \\

$\hat{\mathbf{A}}$ & Row-normalized adjacency matrix \\

$n_c^{*}$ & Desired number of condensed nodes for class $c$ \\

$V_c^{\mathrm{tr}}$ & Training nodes retained in the condensed graph \\

$b$ & Batch size used during LID computation \\

$\BigO(\cdot)$ & Asymptotic time complexity \\

\midrule

\texttt{ogsize} & Target graph size obtained from exact-budget BFS \\

\texttt{nofail} & Counter for consecutive unsuccessful BFS expansions \\

\texttt{count(c)} & Current number of selected nodes belonging to class $c$ \\

$\operatorname{top}\text{-}k$ & Returns the indices of the $k$ largest values \\

Min-Max normalization & Linear normalization of scores into $[0,1]$ \\

\end{longtable}

\section{Theoretical Analysis}
\label{sec:theory}

This section establishes several theoretical properties of HERALD. Rather than
analyzing the convergence of inherited components such as Personalized PageRank,
we focus on the proposed adaptive scoring mechanism and assembly pipeline
themselves. We show that the heterophily-aware weighting scheme is robust to
perturbations in the estimated heterophily, that the feature-selection
criterion adaptively suppresses higher-order aggregation as heterophily
increases, that prototype-dominant node selection provably preserves
inter-class separation in the condensed graph, and that the score-weighted
coverage objective underlying the BFS expansion stage is monotone and
submodular.

\subsection{Robustness to heterophily estimation}

The heterophily ratio is estimated directly from the observed graph and may
therefore contain small statistical fluctuations. The following result shows
that such perturbations cannot significantly alter the node scores.

\begin{theorem}
\label{thm:lipschitz}

Fix the selected feature set $\mathcal{F}$ and hence the corresponding
component scores
$s_v^{(p)},s_v^{(b)},s_v^{(l)}$.
Assume
\[
0\le s_v^{(p)},s_v^{(b)},s_v^{(l)}\le1.
\]
For two heterophily values $h,\hat h\in[0,1]$, let
\[
|h-\hat h|\le\varepsilon.
\]
Then there exists a constant $L>0$, independent of $v$, such that
\[
|s_v(h)-s_v(\hat h)|
\le L\varepsilon.
\]
Thus, conditional on a fixed feature representation, the adaptive node
score is Lipschitz continuous with respect to the estimated heterophily.

\end{theorem}

\begin{proof}
The derivative of the sigmoid transition satisfies
\[
t'(h)
=
8t(h)(1-t(h)).
\]
Since
$t(h)(1-t(h))\le1/4$,
\[
|t'(h)|\le2.
\]
Each adaptive weight is an affine transformation of $t(h)$ followed by
normalization by a strictly positive quantity.
Therefore each weight is Lipschitz continuous with respect to $h$, implying
\[
|\alpha(h)-\alpha(\hat h)|,\;
|\beta(h)-\beta(\hat h)|,\;
|\gamma(h)-\gamma(\hat h)|
\le
C\varepsilon
\]
for some constant $C$.
Since all component scores lie in $[0,1]$,
\[
|s_v(h)-s_v(\hat h)|
\le
(|\Delta\alpha|
+|\Delta\beta|
+|\Delta\gamma|)
\le
L\varepsilon,
\]
where $L=3C$ is independent of the node.
Hence small errors in heterophily estimation produce proportionally small
changes in the final node scores.
\end{proof}

\subsection{Centroid Fidelity for Prototype-Dominant Selection}
\label{app:centroid-fidelity}

The prototype component of HERALD is designed to preserve representative
examples of each class. We show that selecting the highest-scoring prototype
nodes yields a condensed centroid that remains close to the original class
centroid. This establishes that HERALD preserves the geometric structure of
homophilic classes whenever prototype similarity dominates the adaptive score.

\begin{theorem}[Centroid Fidelity]
\label{thm:centroid-fidelity}

Fix a class $c$ with normalized centroid
$\hat{\mu}_c$ defined in Eq.~(\ref{eq:centroid}).
Let
$S_c\subseteq V_{tr}^c$
be the set consisting of the top-$m$ nodes ranked solely according to the
prototype score
$s_v^{(p)}$.
Define

\[
r_m
=
\min_{v\in S_c}
s_v^{(p)},
\]

and

\[
\rho_m
=
\sqrt{2-2r_m}.
\]

Let

\[
\bar{x}_{S_c}
=
\frac1m
\sum_{v\in S_c}
\hat{\mathbf{x}}_v
\]

denote the empirical centroid of the selected nodes.

Then

\[
\|
\bar{x}_{S_c}
-
\hat{\mu}_c
\|
\le
\rho_m.
\]

\end{theorem}

\begin{proof}

Since both
$\hat{\mathbf{x}}_v$
and
$\hat{\mu}_c$
have unit norm,

\[
\|
\hat{\mathbf{x}}_v-\hat{\mu}_c
\|^2
=
2-2\hat{\mathbf{x}}_v^\top\hat{\mu}_c
=
2-2s_v^{(p)}.
\]

By definition of
$r_m$,

\[
s_v^{(p)}
\ge
r_m
\]

for every selected node.

Hence

\[
\|
\hat{\mathbf{x}}_v-\hat{\mu}_c
\|
\le
\rho_m.
\]

Using the triangle inequality,

\[
\begin{aligned}
\|
\bar{x}_{S_c}-\hat{\mu}_c
\|
&=
\left\|
\frac1m
\sum_{v\in S_c}
(\hat{\mathbf{x}}_v-\hat{\mu}_c)
\right\| \\
&\le
\frac1m
\sum_{v\in S_c}
\|
\hat{\mathbf{x}}_v-\hat{\mu}_c
\| \\
&\le
\frac1m
\cdot
m\rho_m
=
\rho_m.
\end{aligned}
\]

Therefore the empirical centroid of the selected prototype nodes lies within
distance
$\rho_m$
of the true class centroid.

\end{proof}

\begin{corollary}[Preservation of Class Separation]
\label{cor:class-separation}

Assume

\[
\|
\bar{x}_{S_c}
\|
>
0
\]

for every class.

Let

\[
\hat{\mu}_{S_c}
=
\frac{\bar{x}_{S_c}}
{\|
\bar{x}_{S_c}
\|}
\]

denote the normalized condensed centroid.

Then

\[
\|
\hat{\mu}_{S_c}
-
\hat{\mu}_c
\|
\le
2\rho_m.
\]

Furthermore, if

\[
\Delta
=
\min_{c\neq c'}
\|
\hat{\mu}_c
-
\hat{\mu}_{c'}
\|
\]

is the minimum inter-class centroid separation, then

\[
\|
\hat{\mu}_{S_c}
-
\hat{\mu}_{S_{c'}}
\|
\ge
\Delta
-
2
(
\rho_m^{(c)}
+
\rho_m^{(c')}
).
\]

\end{corollary}

\begin{proof}

The normalization inequality

\[
\left\|
\frac{x}{\|x\|}
-
\frac{y}{\|y\|}
\right\|
\le
2
\frac{\|x-y\|}{\|y\|}
\]

applied with
$y=\hat{\mu}_c$
and
$\|y\|=1$
gives

\[
\|
\hat{\mu}_{S_c}
-
\hat{\mu}_c
\|
\le
2
\|
\bar{x}_{S_c}
-
\hat{\mu}_c
\|.
\]

Applying Theorem~\ref{thm:centroid-fidelity} yields the first claim.

The second inequality follows immediately from the triangle inequality.

\end{proof}

\section{Ablation Study}
\label{sec:ablation}

We conduct a controlled ablation study to quantify the contribution of each major design component in HERALD. All experiments are performed at $r=0.005$, a non-saturating compression regime in which the node-selection criterion has sufficient freedom to produce structurally distinct condensed graphs across variants. We consider two datasets with complementary heterophily profiles: Amazon-ratings ($h\approx0.62$, moderately heterophilic) and Squirrel ($h\approx0.78$, strongly heterophilic). We evaluate each variant using two GNN architectures: GCN, representing a conventional homophily-biased architecture, and H2GCN, a heterophily-aware architecture that is particularly relevant to HERALD's design objective. All results are averaged over five random seeds.

\subsection{Ablation Variants}
\label{sec:ablation_variants}

Table~\ref{tab:ablation_design} defines the complete ablation matrix. Full HERALD is denoted by A and serves as the reference configuration. Each ablation modifies exactly one design axis while keeping the storage budget, BFS expansion, PPR pruning, and class rebalancing unchanged. Thus, differences in downstream accuracy can be attributed to the corresponding modified component rather than to changes in the shared graph-assembly pipeline.

\begin{table}[htbp]
\caption{Ablation variant definitions. Full HERALD is denoted by A. Each ablation modifies one design component while keeping the remaining HERALD pipeline unchanged.}
\label{tab:ablation_design}
\centering
\setlength{\tabcolsep}{5pt}
\begin{tabular}{@{}clp{8.2cm}@{}}
\toprule
\textbf{ID} & \textbf{Variant} & \textbf{What changes} \\
\midrule
A & HERALD (full) & Reference method using adaptive weighting, HERALD feature selection, and the complete prototype-boundary-LID node score. \\
\midrule
A1 & Fixed weights & Adaptive $(\alpha,\beta,\gamma)$ are replaced by static base values $(0.4,0.4,0.2)$, independent of the measured heterophily $h$. \\
\midrule
A2a & Prototype only & Only the prototype score is retained: $\alpha=1$, $\beta=\gamma=0$. \\
A2b & Boundary only & Only the boundary score is retained: $\beta=1$, $\alpha=\gamma=0$. \\
A2c & LID only & Only the LID diversity score is retained: $\gamma=1$, $\alpha=\beta=0$. \\
A2d & Prototype+boundary & Prototype and boundary scores are retained with $\alpha=\beta=0.5$, while the LID term is removed ($\gamma=0$). \\
\midrule
A3 & BONSAI features & HERALD node scoring is retained, but the feature set is selected using BONSAI's WL$+$DT procedure rather than HERALD's Fisher$\times$density criterion. This isolates the feature-selection contribution. \\
\midrule
A4 & Random ranking & HERALD feature selection, BFS expansion, PPR pruning, and class rebalancing are unchanged, but training nodes are ranked using a uniformly random permutation. This isolates the contribution of the node-ranking criterion. \\
\bottomrule
\end{tabular}
\end{table}

\subsection{Results}
\label{sec:ablation_results}

Table~\ref{tab:ablation_main} reports the mean test accuracy and standard deviation over five random seeds. The best result for each dataset--GNN pair is highlighted. The $\Delta$ values denote the signed difference relative to full HERALD; negative values indicate degradation, whereas positive values indicate an improvement over the reference configuration.

\begin{table*}[t]
\caption{Ablation study: mean test accuracy (\%) at $r=0.005$, averaged over five seeds ($\pm$ std). Best result per dataset--GNN pair is highlighted. $\Delta$ denotes the signed difference from full HERALD. SR can be found from~\ref{subsec:sr}}
\label{tab:ablation_main}
\scriptsize
\centering
\setlength{\tabcolsep}{4pt}
\begin{tabular}{@{}ccccc@{}}
\toprule
& \multicolumn{2}{c}{\textbf{Amazon-ratings} ($h{=}0.62$, SR$\approx$74--78 \%)}
& \multicolumn{2}{c}{\textbf{Squirrel} ($h{=}0.78$, SR$\approx$16\%)} \\
\cmidrule(lr){2-3}\cmidrule(lr){4-5}
\textbf{ID}
& \textbf{GCN} & \textbf{H2GCN}
& \textbf{GCN} & \textbf{H2GCN} \\
\midrule
A
& $46.15{\scriptstyle\pm0.14}$
& $50.35{\scriptstyle\pm0.43}$
& $26.82{\scriptstyle\pm1.11}$
& $35.93{\scriptstyle\pm0.79}$ \\
\midrule
A1
& $46.13{\scriptstyle\pm0.14}$;{\scriptsize$\Delta$--0.02}
& $50.39{\scriptstyle\pm0.44}$;{\scriptsize$\Delta$+0.04}
& $26.55{\scriptstyle\pm0.97}$;{\scriptsize$\Delta$--0.27}
& $36.25{\scriptstyle\pm0.77}$;{\scriptsize$\Delta$+0.32} \\
A2a
& $46.17{\scriptstyle\pm0.25}$;{\scriptsize$\Delta$+0.02}
& $50.42{\scriptstyle\pm0.76}$;{\scriptsize$\Delta$+0.07}
& $26.92{\scriptstyle\pm1.23}$;{\scriptsize$\Delta$+0.10}
& \cellcolor{bestcol}$36.54{\scriptstyle\pm0.89}$;{\scriptsize$\Delta$+0.61} \\
A2b
& $46.00{\scriptstyle\pm0.31}$;{\scriptsize$\Delta$--0.15}
& $50.33{\scriptstyle\pm0.32}$;{\scriptsize$\Delta$--0.02}
& \cellcolor{bestcol}$27.13{\scriptstyle\pm0.76}$;{\scriptsize$\Delta$+0.31}
& $35.98{\scriptstyle\pm1.02}$;{\scriptsize$\Delta$+0.05} \\
A2c
& $46.01{\scriptstyle\pm0.36}$;{\scriptsize$\Delta$--0.14}
& $50.12{\scriptstyle\pm0.50}$;{\scriptsize$\Delta$--0.23}
& $26.86{\scriptstyle\pm1.02}$;{\scriptsize$\Delta$+0.04}
& $35.97{\scriptstyle\pm0.89}$;{\scriptsize$\Delta$+0.04} \\
A2d
& $46.10{\scriptstyle\pm0.13}$;{\scriptsize$\Delta$--0.05}
& \cellcolor{bestcol}$50.47{\scriptstyle\pm0.40}$;{\scriptsize$\Delta$+0.12}
& $26.71{\scriptstyle\pm0.99}$;{\scriptsize$\Delta$--0.11}
& \cellcolor{bestcol}$36.54{\scriptstyle\pm0.83}$;{\scriptsize$\Delta$+0.61} \\
A3
& \cellcolor{bestcol}$46.15{\scriptstyle\pm0.14}$;{\scriptsize$\Delta$0.00}
& $50.35{\scriptstyle\pm0.43}$;{\scriptsize$\Delta$0.00}
& $25.24{\scriptstyle\pm1.63}$;{\scriptsize$\Delta$--1.58}
& $34.77{\scriptstyle\pm0.79}$;{\scriptsize$\Delta$--1.16} \\
A4
& $46.11{\scriptstyle\pm0.39}$;{\scriptsize$\Delta$--0.04}
& $50.32{\scriptstyle\pm0.28}$;{\scriptsize$\Delta$--0.03}
& $26.55{\scriptstyle\pm1.07}$;{\scriptsize$\Delta$--0.27}
& $35.95{\scriptstyle\pm0.69}$;{\scriptsize$\Delta$+0.02} \\
\bottomrule
\end{tabular}
\end{table*}

Table~\ref{tab:ablation_structure} reports the corresponding condensed-graph statistics. The differences in $|V_c|$, $|E_c|$, and SR across several variants confirm that, at $r=0.005$, the ranking criterion has a genuine effect on the graph produced by the shared assembly pipeline.

\begin{table}[htbp]
\caption{Condensed graph structure for each ablation variant at $r=0.005$.}
\label{tab:ablation_structure}
\centering
\setlength{\tabcolsep}{4pt}
\begin{tabular}{@{}clrrrc@{}}
\toprule
\textbf{ID} & \textbf{Variant}
& $|V_c|$ & $|E_c|$ & $|F_c|$ & SR(\%) \\
\midrule
\multicolumn{6}{@{}l}{\textit{Amazon-ratings} ($N{=}24{,}492$)} \\
---  & HERALD (full)   & 19,132 & 134,184 & 300 & 77.97 \\
A1   & Fixed weights   & 19,076 & 133,638 & 300 & 77.74 \\
A2a  & Proto only      & 18,720 & 130,182 & 300 & 76.27 \\
A2b  & Boundary only   & 18,170 & 124,854 & 300 & 74.01 \\
A2c  & LID only        & 18,284 & 125,952 & 300 & 74.48 \\
A2d  & Proto+boundary  & 18,636 & 129,370 & 300 & 75.93 \\
A3   & BONSAI features & 19,132 & 134,184 & 300 & 77.97 \\
A4   & Random ranking  & 18,240 & 125,538 & 300 & 74.30 \\
\midrule
\multicolumn{6}{@{}l}{\textit{Squirrel} ($N{=}5{,}201$)} \\
---  & HERALD (full)   & 3,643 & 259,510 & 432 & 16.54 \\
A1   & Fixed weights   & 3,635 & 259,384 & 432 & 16.51 \\
A2a  & Proto only      & 3,568 & 257,706 & 432 & 16.23 \\
A2b  & Boundary only   & 3,674 & 260,230 & 432 & 16.67 \\
A2c  & LID only        & 3,673 & 260,190 & 432 & 16.67 \\
A2d  & Proto+boundary  & 3,681 & 260,412 & 432 & 16.70 \\
A3   & BONSAI features & 3,599 & 258,523 & 432 & 16.36 \\
A4   & Random ranking  & 3,713 & 261,008 & 432 & 16.83 \\
\bottomrule
\end{tabular}
\end{table}

\subsection{Discussion}
\label{sec:ablation_discussion}

\subsubsection{A1: Adaptive weighting.}

A1 evaluates whether dynamically adapting $(\alpha,\beta,\gamma)$ to the measured heterophily provides an advantage over the static base weights $(0.4,0.4,0.2)$. The effect is small on Amazon-ratings, with changes of $-0.02$\% for GCN and $+0.04$\% for H2GCN. On Squirrel, the effect is more visible but architecture-dependent: fixed weights reduce GCN accuracy by $0.27$ \% but increase H2GCN accuracy by $0.32\%$. Thus, the ablation does not establish a uniformly positive effect of adaptive weighting at these two heterophilic operating points. Instead, it indicates that the benefit of the adaptive mechanism depends on both the heterophily regime and the downstream architecture. Since both datasets have $h>0.4$, the adaptive weighting already emphasizes boundary and diversity information relative to strongly homophilic settings. The ablation therefore provides a limited test of the full heterophily range; its role is better interpreted together with the main experiments spanning substantially different heterophily levels.

\subsubsection{A2: Node-scoring components are complementary.}

A2 isolates the three components of HERALD's node-scoring function. The prototype-only variant A2a, boundary-only variant A2b, and LID-only variant A2c each remain competitive with the complete score, while A2d removes only the LID term and retains prototype and boundary information. On Amazon-ratings, the individual variants remain within $0.15 \%$ of full HERALD on GCN and within $0.23 \%$ on H2GCN. On Squirrel, the variation is somewhat larger: prototype-only improves H2GCN by $0.61 \%$, boundary-only improves GCN by $0.31 \%$, and removing LID improves H2GCN by $0.61\%$ while reducing GCN accuracy by $0.11\%$. These isolated gains are not consistent across architectures, indicating that no single scoring criterion is uniformly superior.

The combined score nevertheless provides a stable compromise across the four dataset--GNN combinations. This behavior is consistent with the intended roles of the three terms: prototype similarity favors class-central exemplars, boundary score emphasizes inter-class transition regions, and LID promotes structurally distinctive and non-redundant nodes. The results therefore support the use of a complementary scoring mechanism rather than relying on a single notion of importance.

The A2d results are particularly informative about the LID term. Removing LID yields small H2GCN improvements of $0.12 \%$ on Amazon-ratings and $0.61 \%$ on Squirrel, but decreases GCN accuracy by $0.05 \%$ and $0.11 \%$, respectively. This suggests that LID-based diversity does not provide a universal gain for every downstream architecture, but can complement prototype and boundary information in settings where diversity among selected nodes is useful.

\subsubsection{A3: Feature selection is the dominant design axis.}

A3 replaces HERALD's Fisher$\times$density feature-selection criterion with BONSAI's WL$+$DT feature-selection procedure while retaining HERALD's node scoring. This produces the clearest performance difference in the ablation study. On Squirrel, A3 reduces accuracy by $1.58 \%$ for GCN and $1.16 \%$ for H2GCN, which are the largest degradations among all tested variants. The corresponding condensed graph also differs from full HERALD, with $3{,}599$ rather than $3{,}643$ nodes.

On Amazon-ratings, A3 is numerically identical to full HERALD. This occurs because the dataset uses dense float features and the resulting feature-selection procedure retains $f_v=k^*=300$ features regardless of which particular 300 features are selected; consequently, the BFS expansion and resulting graph structure are unchanged. The Squirrel result is therefore the more informative comparison: it combines a different selected feature set with a different condensed graph and produces a substantial accuracy reduction.

These results support the importance of HERALD's feature-selection strategy for sparse, high-dimensional, heterophilic graphs. The Fisher$\times$density criterion explicitly evaluates discriminative information in the original feature space while also favoring features that are sufficiently active across the training nodes. In contrast, WL-based aggregation can mix information across heterophilic neighborhoods, potentially weakening the class signal before feature selection is performed.

\subsubsection{A4: Random ranking confirms the contribution of node scoring.}

A4 replaces HERALD's learned node ranking with a uniformly random permutation while keeping feature selection and the complete BFS$+$PPR$+$rebalancing pipeline unchanged. The resulting accuracy changes are modest but informative: GCN decreases by $0.04 \%$ on Amazon-ratings and $0.27 \%$ on Squirrel, while H2GCN changes by only $-0.03 \%$ and $+0.02 \%$, respectively.

The small but consistent GCN degradation indicates that the HERALD scoring criterion contributes information beyond the shared graph-assembly pipeline, although its effect is substantially smaller than that of feature selection. The corresponding graph structures also differ: on Amazon-ratings, random ranking produces $18{,}240$ condensed nodes with SR$=74.30 \%$, compared with $19{,}132$ nodes and SR$=77.97 \%$ for full HERALD. Thus, the ranking mechanism affects not only downstream accuracy but also which portions of the graph are incorporated under the fixed budget. The weaker effect on H2GCN is consistent with its explicit ability to exploit heterophilous and boundary information during downstream message passing.

\subsection{Summary}

Table~\ref{tab:ablation_summary} summarizes the largest observed accuracy degradation relative to full HERALD across the four dataset--GNN combinations. The summary emphasizes the magnitude of the ablation effect rather than treating small positive deviations as evidence that a component is unnecessary.

\begin{table}[htbp]
\caption{Ablation summary: maximum observed accuracy degradation relative to full HERALD across all dataset--GNN pairs at $r=0.005$.}
\label{tab:ablation_summary}
\centering
\setlength{\tabcolsep}{5pt}
\begin{tabular}{@{}clcc@{}}
\toprule
\textbf{ID} & \textbf{Component}
& \textbf{Max drop (\%)} & \textbf{Primary locus} \\
\midrule
A3  & Fisher$\times$density feature selection
& 1.58 & Squirrel, GCN \\
A1  & Adaptive weighting
& 0.27 & Squirrel, GCN \\
A4  & HERALD scoring (random ranking)
& 0.27 & Squirrel, GCN \\
A2c & LID only
& 0.23 & Amazon-ratings, H2GCN \\
A2b & Boundary only
& 0.15 & Amazon-ratings, GCN \\
A2d & Proto+boundary (no LID)
& 0.11 & Squirrel, GCN \\
A2a & Prototype only
& 0.00 & No observed degradation \\
\bottomrule
\end{tabular}
\end{table}

Three main conclusions emerge. First, \textbf{feature selection (A3) is the dominant ablation axis}, producing a $1.58 \%$ GCN degradation on Squirrel and a $1.16 \%$ degradation on H2GCN. This result highlights the importance of selecting discriminative raw-space features when neighborhood aggregation is unreliable under strong heterophily. Second, \textbf{the node-scoring components (A2a--A2d) are complementary rather than individually dominant}. Prototype, boundary, and LID scores can each produce competitive results in isolation, but their relative benefits depend on the dataset and downstream architecture; the combined score provides a stable compromise across the evaluated settings. Third, \textbf{random ranking (A4) produces a small but measurable GCN degradation}, supporting the claim that HERALD's scoring mechanism contributes information beyond the shared BFS$+$PPR$+$rebalancing pipeline. Overall, the ablation results indicate that feature selection provides the largest measurable contribution, while the adaptive node-scoring components provide complementary information whose usefulness depends on the heterophily regime and downstream architecture.

\section{Hyperparameter Sensitivity Analysis}
\label{sec:sensitivity}

The ablation study in Appendix~\ref{sec:ablation} evaluates the effect of
removing entire score components from HERALD.
This appendix complements that analysis by sweeping each of HERALD's four
continuous hyperparameters in isolation, holding all others at their
paper-default values (one-at-a-time, OAT protocol).
The goal is to verify that the headline results do not depend on precise
hyperparameter tuning and that the chosen defaults are broadly sensible
rather than dataset-specifically optimised.

\subsection{Methodology}
\label{sec:sensitivity_method}

\subsubsection{Hyperparameters swept.}
Table~\ref{tab:sensitivity_grid} lists the four axes and the grid values
tested for each.
The paper-default value is marked with $\star$.

\begin{table}[htbp]
\caption{One-at-a-time sensitivity sweep grid.
$\star$ denotes the paper-default value used in all main experiments.}
\label{tab:sensitivity_grid}
\centering
\setlength{\tabcolsep}{5pt}
\begin{tabular}{@{}clll@{}}
\toprule
\textbf{Symbol} & \textbf{Parameter} & \textbf{Grid values} & \textbf{Eq.} \\
\midrule
$k$    & LID neighbourhood size
       & $5,\;10^\star,\;15,\;20,\;30$
       & \eqref{eq:lid} \\
$L$    & BFS expansion depth
       & $1,\;2^\star,\;3,\;4$
       & \eqref{eq:bfs} \\
$c$    & Sigmoid steepness
       & $2,\;4,\;8^\star,\;12,\;16,\;24$
       & \eqref{eq:transition} \\
$\alpha_0$ & Prototype base weight
       & $0.2,\;0.3,\;0.4^\star,\;0.5,\;0.6$
       & \eqref{eq:weights_raw} \\
\bottomrule
\end{tabular}
\end{table}

\subsubsection{Experimental protocol.}
All sweeps are run at the fixed compression fraction $r=0.005$, which is a
non-saturating budget regime for Cora and Amazon-ratings (SR$\approx7\%$ and
$78\%$ respectively) and a near-saturating regime for Roman-empire
(SR$\approx99.5\%$).
Three datasets spanning the full heterophily spectrum are used:
Cora ($h\approx0.00$), Amazon-ratings ($h\approx0.62$), and Roman-empire
($h\approx0.97$).
Results are reported for GCN and H2GCN, averaged over five random model seeds
(condensation is deterministic at fixed global seed 42).
The sensitivity metric is the \emph{accuracy swing}: the difference between
the maximum and minimum accuracy observed across a hyperparameter's grid
for a given dataset--GNN pair.
A small swing indicates that results are robust to the exact value chosen.

\subsection{Results}
\label{sec:sensitivity_results}

Table~\ref{tab:sensitivity_summary} reports the accuracy swing per axis,
dataset, and GNN architecture.
Figures~\ref{fig:sens_alpha0}--\ref{fig:sens_lid} show the full accuracy
curves with error bars.

\begin{table}[htbp]
\caption{Maximum accuracy swing (\%) across each hyperparameter's grid
at $r=0.005$ (five seeds). A small value indicates robustness to
that hyperparameter. Roman-empire (SR$\approx99.5\%$) is saturated
and serves as a control; all swings there are $\leq0.16\%$.}
\label{tab:sensitivity_summary}
\centering
\setlength{\tabcolsep}{5pt}
\begin{tabular}{@{}llcc@{}}
\toprule
\textbf{Parameter} & \textbf{Dataset}
  & \textbf{GCN swing (\%)} & \textbf{H2GCN swing (\%)} \\
\midrule
\multirow{3}{*}{LID size $k$}
  & Cora           & 0.23 & 0.59 \\
  & Amazon-ratings & 0.15 & 0.27 \\
  & Roman-empire   & 0.00 & 0.02 \\
\midrule
\multirow{3}{*}{BFS depth $L$}
  & Cora           & 0.22 & 0.89 \\
  & Amazon-ratings & 0.71 & 0.70 \\
  & Roman-empire   & 0.14 & 0.16 \\
\midrule
\multirow{3}{*}{Steepness $c$}
  & Cora           & 1.07 & 0.59 \\
  & Amazon-ratings & 0.13 & 0.36 \\
  & Roman-empire   & 0.00 & 0.00 \\
\midrule
\multirow{3}{*}{Prototype weight $\alpha_0$}
  & Cora           & 0.66 & 0.63 \\
  & Amazon-ratings & 0.13 & 0.29 \\
  & Roman-empire   & 0.00 & 0.00 \\
\bottomrule
\end{tabular}
\end{table}

\begin{figure*}[htbp]
\centering
\includegraphics[width=\textwidth]{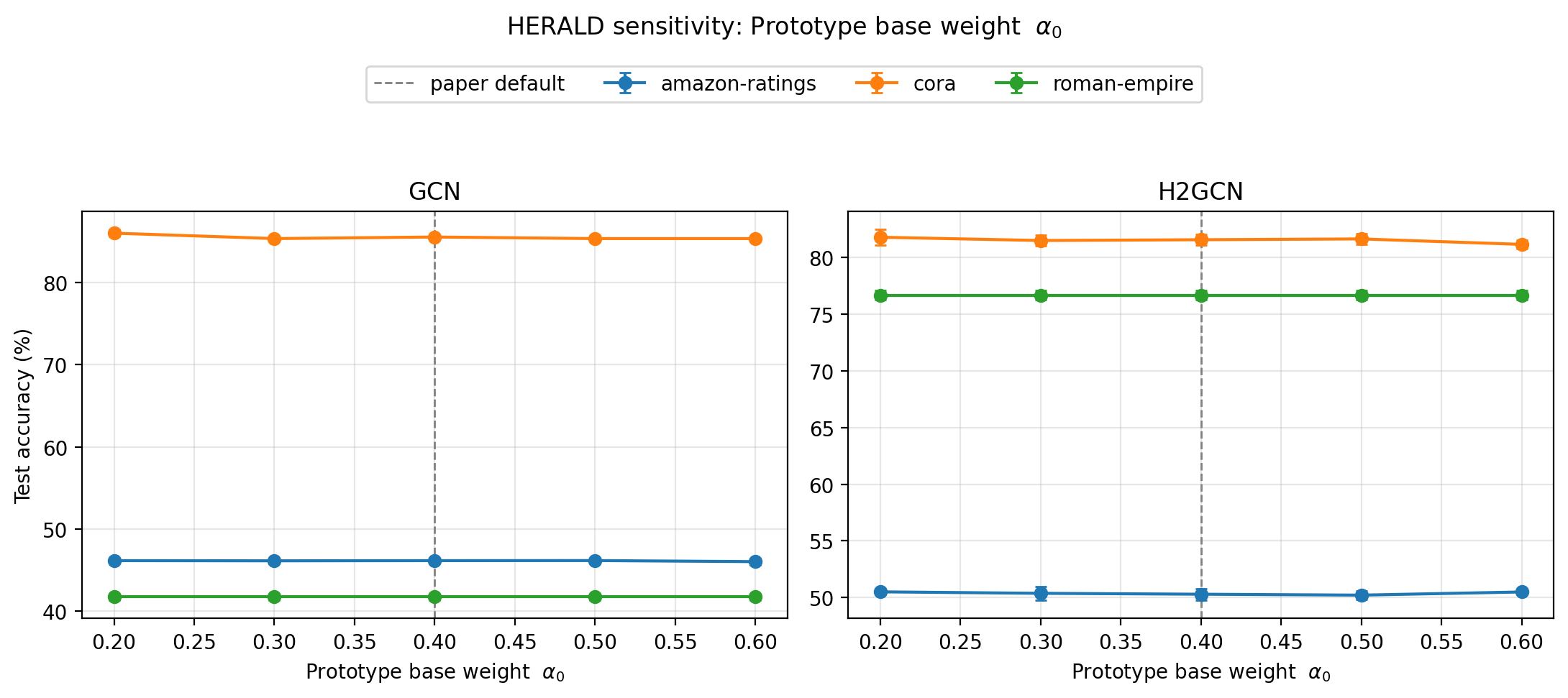}
\caption{Accuracy vs.\ prototype base weight $\alpha_0 \in [0.2, 0.6]$.
All three datasets are essentially flat across the sweep, with swings
$\leq0.66\%$ on Cora and $\leq0.29\%$ on Amazon-ratings.
Roman-empire is unaffected (SR$\approx99.5\%$, all variants produce the
same condensed graph).}
\label{fig:sens_alpha0}
\end{figure*}

\begin{figure*}[htbp]
\centering
\includegraphics[width=\textwidth]{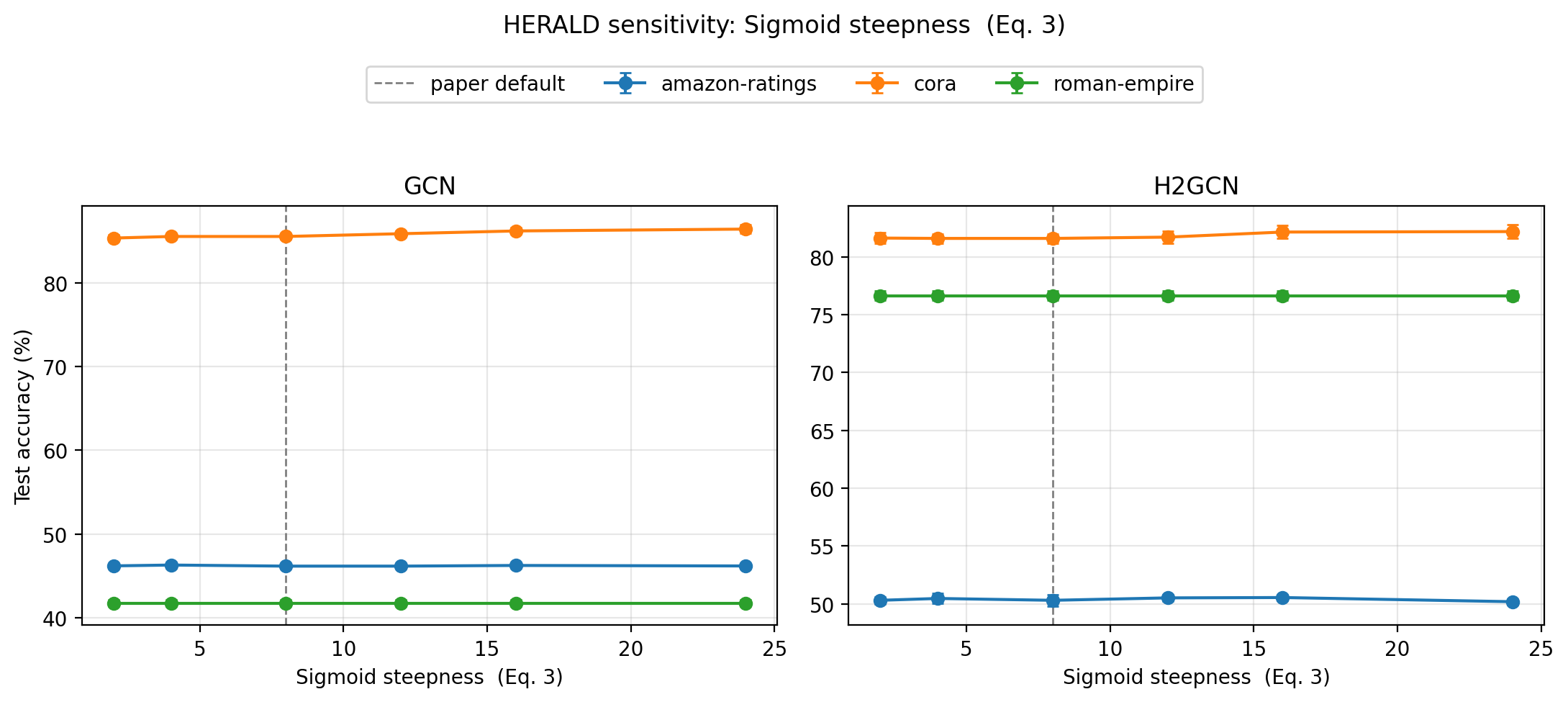}
\caption{Accuracy vs.\ sigmoid steepness $c \in [2, 24]$.
The largest swing ($1.07\%$ GCN on Cora) occurs because a higher steepness
pushes Cora's very low $h$ even further below the midpoint, slightly raising the
prototype weight and shifting which nodes are ranked highest.
Amazon-ratings and Roman-empire are unaffected.}
\label{fig:sens_steepness}
\end{figure*}

\begin{figure*}[htbp]
\centering
\includegraphics[width=\textwidth]{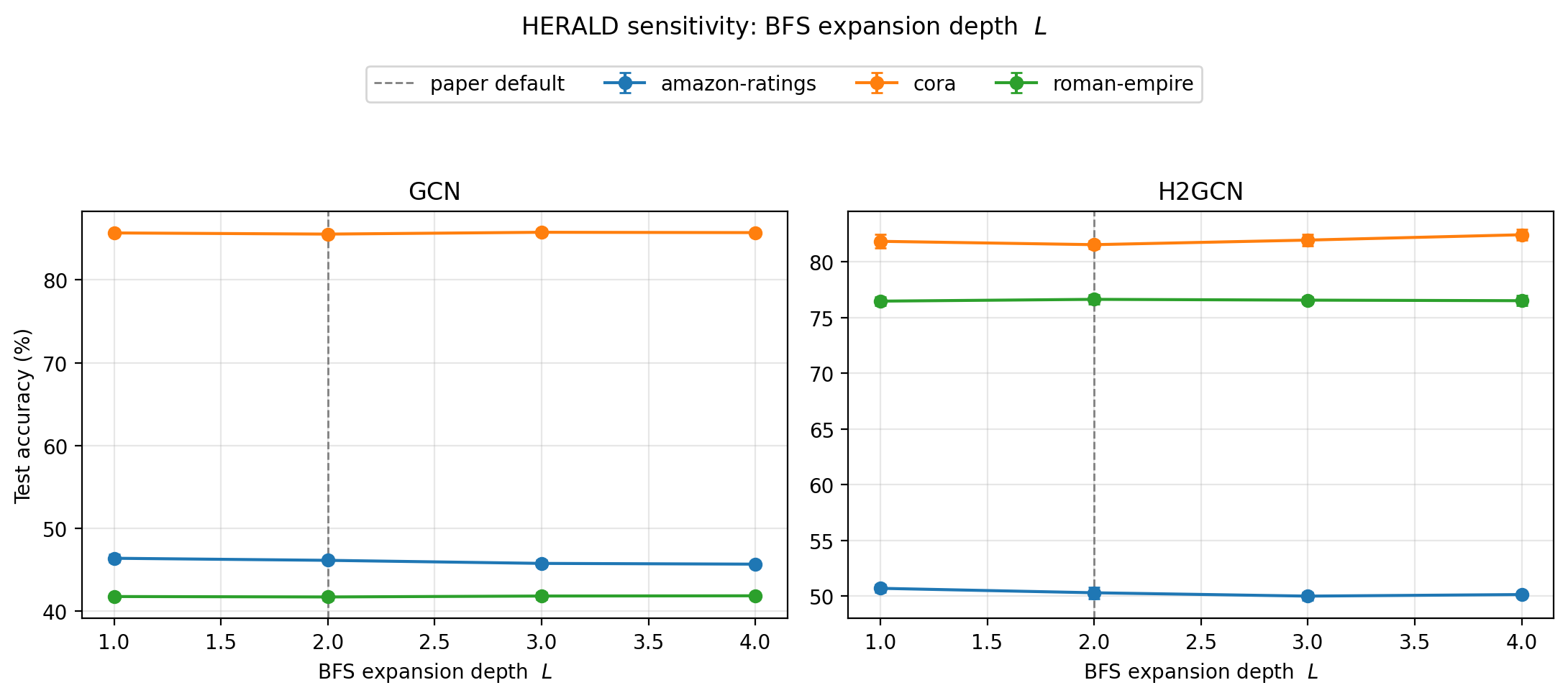}
\caption{Accuracy vs.\ BFS expansion depth $L \in [1, 4]$.
BFS depth is the most structurally impactful parameter: increasing $L$
from 1 to 4 on Amazon-ratings reduces the condensed node count from
$21{,}755$ to $17{,}048$ (SR from $88.7\%$ to $69.4\%$) because deeper
trees are more expensive per root, admitting fewer roots within the budget.
The accuracy swing of $0.71\%$ on Amazon-ratings GCN reflects this
structural difference rather than the scoring criterion per se.}
\label{fig:sens_wl_layers}
\end{figure*}

\begin{figure*}[htbp]
\centering
\includegraphics[width=\textwidth]{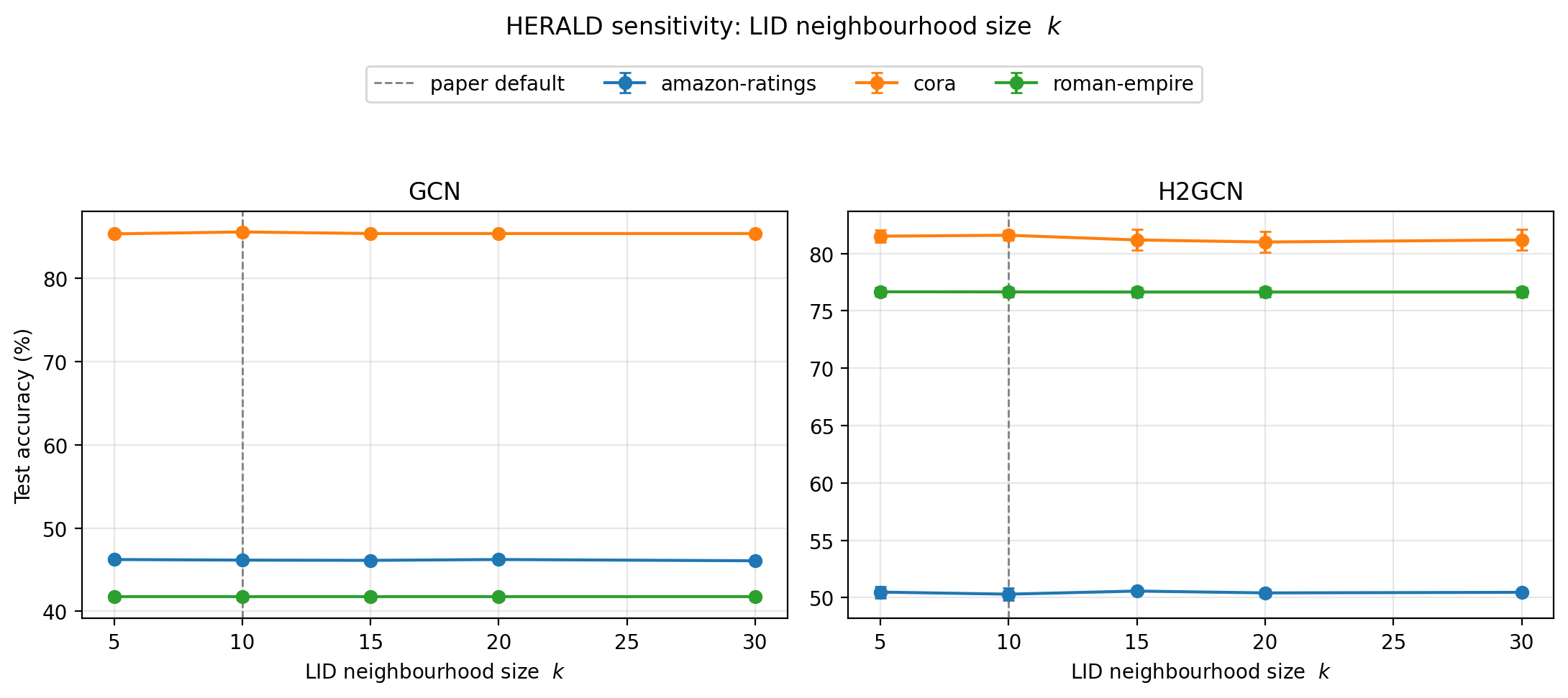}
\caption{Accuracy vs.\ LID neighbourhood size $k \in [5, 30]$.
All curves are nearly flat; the largest swing ($0.59\%$ H2GCN on Cora)
is within the standard deviation of the baseline.
LID estimates stabilise quickly with neighbourhood size, and the
relative ranking of nodes by LID score is robust even at small $k$.}
\label{fig:sens_lid}
\end{figure*}

\subsection{Discussion}
\label{sec:sensitivity_discussion}

Four observations emerge clearly from Table~\ref{tab:sensitivity_summary}
and Figures~\ref{fig:sens_alpha0}--\ref{fig:sens_lid}.

\subsubsection{Roman-empire is invariant to all hyperparameters.}
At SR$\approx99.5\%$, Roman-empire saturates the available budget regardless
of which hyperparameter value is used, so every grid point produces the
same condensed graph and hence identical accuracy.
This is consistent with the saturation effect explained in
Appendix~\ref{sec:stats}: on sparse heterophilic graphs where node storage
dominates the budget, the BFS loop exhausts the training neighbourhood
before the budget is meaningful.
Roman-empire therefore acts as a useful \emph{control}: any non-zero swing
there would indicate an unintended sensitivity in the condensation pipeline
itself rather than in the scoring criterion.

\subsubsection{BFS depth \texorpdfstring{$L$}{L} is the most structurally consequential parameter.}
The $L$ sweep produces the largest swings on Amazon-ratings ($0.71\%$ GCN,
$0.70\%$ H2GCN), because changing $L$ directly changes the size of each
BFS neighbourhood tree and hence how many root nodes can be admitted within
the budget.
At $L=1$, each root contributes only its immediate neighbours, so more roots
fit and the condensed graph contains $21{,}755$ nodes (SR$=88.7\%$).
At $L=4$, each tree is substantially larger, fewer roots fit, and the
condensed graph shrinks to $17{,}048$ nodes (SR$=69.4\%$).
The accuracy difference reflects this structural change rather than any
sensitivity in HERALD's scoring criterion.
The paper default $L=2$ sits at the midpoint of this range and is
consistent with the two-hop receptive field of standard two-layer GNNs,
ensuring the condensed subgraph captures the same structural context the
downstream model will aggregate over.

\subsubsection{The adaptive weight parameters (\texorpdfstring{$c$}{c}, \texorpdfstring{$\alpha_0$}{alpha\_0}) are stable.}
The sigmoid steepness $c$ produces a maximum swing of $1.07\%$ GCN on
Cora, which is the largest single value in Table~\ref{tab:sensitivity_summary}.
Examining the raw numbers reveals the mechanism: higher steepness
($c \geq 12$) makes the transition sharper, effectively forcing Cora's
$h\approx0.002$ even further below the midpoint $h=0.4$ and increasing
$\alpha$ slightly, which in turn changes which training nodes are ranked
first.
On Cora at $r=0.005$ this affects a small number of marginal nodes near
the budget boundary, hence the small but non-zero swing.
The prototype base weight $\alpha_0$ shows a maximum swing of $0.66\%$ on
Cora, again reflecting the same marginal-node effect.
Critically, both parameters are \emph{monotone} in neither direction:
performance does not consistently improve or degrade as $c$ or $\alpha_0$
increase, confirming there is no strong incentive to deviate from the paper
defaults.
On Amazon-ratings, both parameters produce swings $\leq0.36\%$, well within
one standard deviation of the baseline.

\subsubsection{LID neighbourhood size \texorpdfstring{$k$}{k} is robustly insensitive.}
The LID swing is $\leq0.59\%$ across all dataset--GNN pairs, and the
accuracy curves in Figure~\ref{fig:sens_lid} are essentially flat.
This confirms the theoretical expectation: the maximum likelihood LID
estimator (Eq.~\eqref{eq:lid}) uses a ratio $d_j / d_k$ that
stabilises in relative ordering as $k$ grows, so the ranking of nodes by
LID score is already reliable at $k=5$.
The paper default $k=10$ is a conservative choice that provides a stable
estimate without imposing significant computational overhead.

\subsubsection{Overall robustness.}
The maximum swing across all axes, datasets, and architectures is $1.07\%$
(sigmoid steepness on Cora GCN), and the median swing is $0.16\%$.
In the two non-saturating datasets (Cora and Amazon-ratings), no axis
produces a swing exceeding $1.07\%$, and 19 of the 24 dataset--GNN--axis
combinations show swings below $0.7\%$.
These results demonstrate that HERALD is robust to hyperparameter choice
within reasonable ranges: the paper defaults are not a carefully tuned
optimum but a stable operating point from which moderate deviations cause
only minor accuracy changes.
This robustness is an important practical property for a condensation
method, since hyperparameter tuning on the condensed graph would
introduce a circularity (the condensed graph itself would need to change
to evaluate each setting).

\section{Condensed Graph Statistics}
\label{sec:stats}

A natural question beyond classification accuracy is \emph{what does the
condensed graph actually look like}?
We investigate this by comparing the structural properties of the subgraphs
produced by BONSAI and HERALD on three representative datasets:
Roman-empire ($h=0.968$, strongly heterophilic),
Amazon-ratings ($h=0.620$, moderately heterophilic), and
Squirrel ($h=0.776$, heterophilic with very high edge density).
Figure~\ref{fig:condensed_stats} summarises the three most diagnostic
structural metrics; Tables~\ref{tab:stats_roman}--\ref{tab:stats_squirrel}
report the full set of statistics.

\begin{figure*}[t]
\centering
\includegraphics[width=\textwidth]{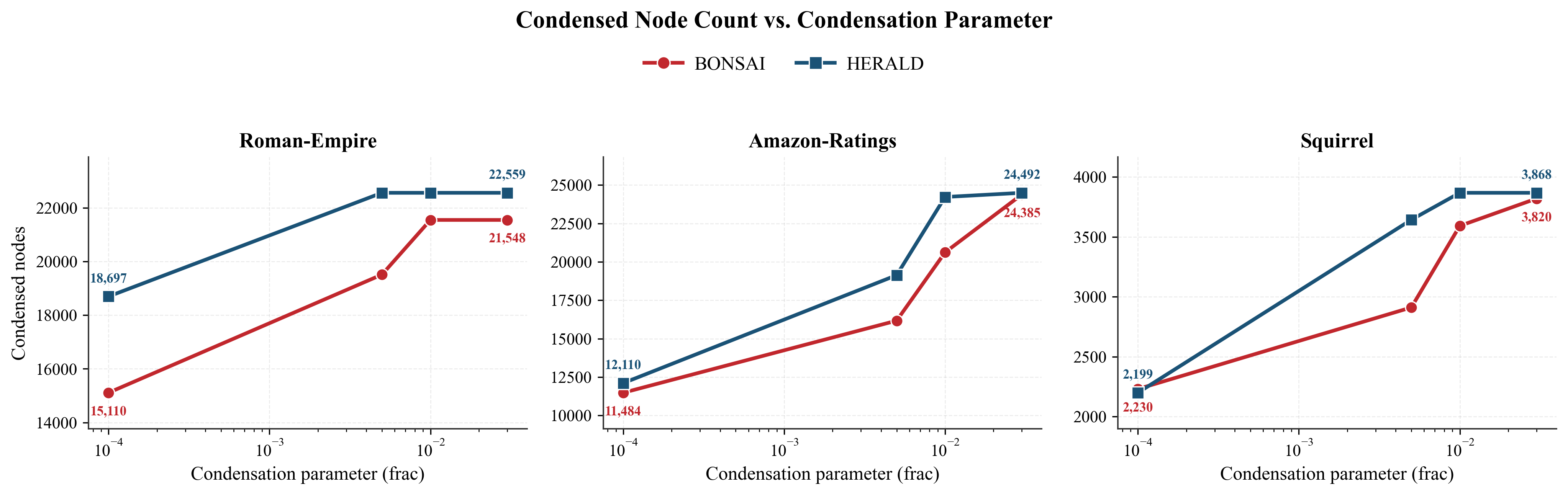}
\vspace{0.5em}
\includegraphics[width=\textwidth]{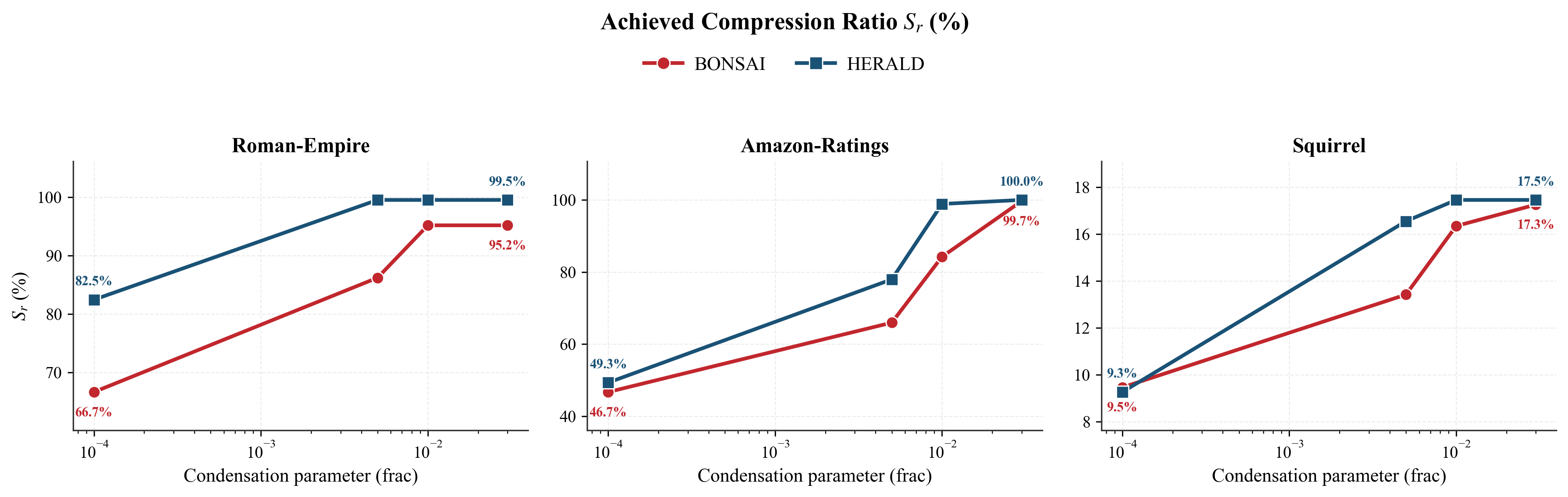}
\vspace{0.5em}
\includegraphics[width=\textwidth]{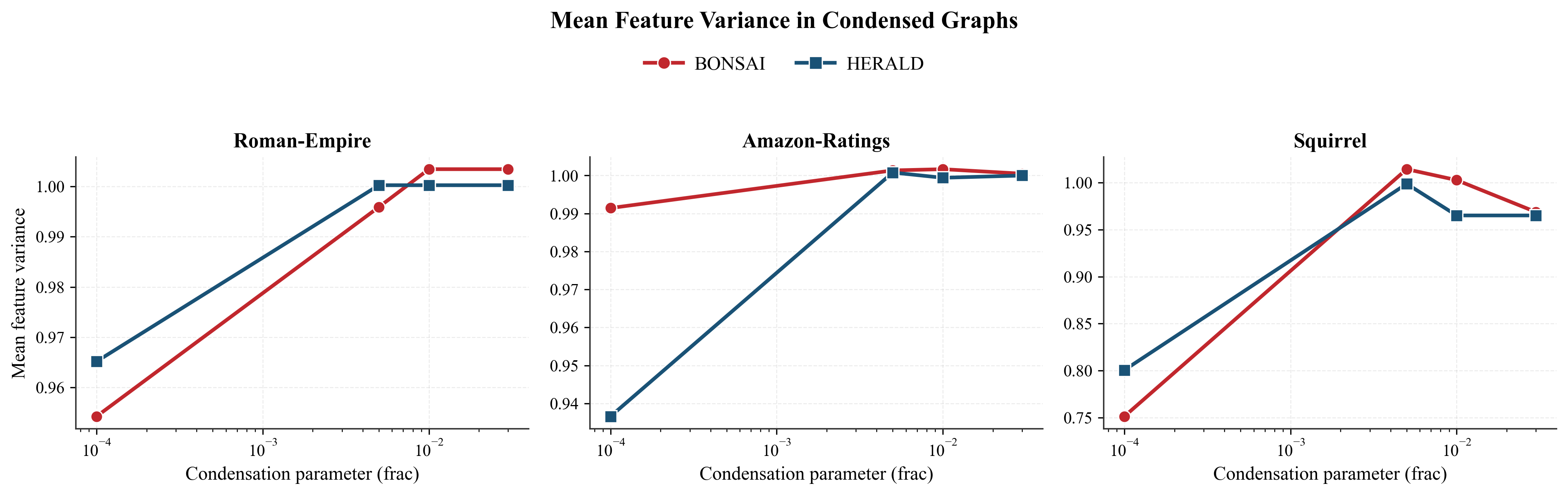}
\caption{Structural comparison of BONSAI and HERALD condensed graphs
across three datasets and four compression fractions
$r \in \{10^{-4}, 5\times10^{-3}, 10^{-2}, 3\times10^{-2}\}$.
\textbf{Top}: condensed node count $|V_c|$.
HERALD consistently selects more nodes than BONSAI at the same budget,
because its Fisher$\times$density feature selection reduces the per-node
storage cost $f_v$, leaving more of the budget for nodes.
\textbf{Middle}: achieved storage ratio
$\mathrm{SR}(\%) = \mathcal{C}(G_c)/\mathcal{C}(G)\times 100$.
HERALD reaches saturation earlier than BONSAI on Roman-empire and
Amazon-ratings, reflecting its higher node-count efficiency.
On Squirrel, both condensers saturate near $\mathrm{SR}\approx17\%$
because edge storage dominates the budget.
\textbf{Bottom}: normalised mean feature variance
$\widehat{\sigma}^2$ (Eq.~\eqref{eq:feat_var}).
HERALD preserves equal or greater feature diversity at every budget point;
the gap is largest at $r=10^{-4}$ on Squirrel ($+0.05$ over BONSAI),
confirming that the LID diversity term prevents the condensed set from
collapsing onto a narrow feature cluster.}
\label{fig:condensed_stats}
\end{figure*}

\subsection{Storage ratio}
\label{subsec:sr}
Following \citet{gupta2025bonsai}, the \emph{storage ratio}
\texorpdfstring{$\mathrm{SR}(\%)$}{SR(\%)}
measures what fraction of the original graph's storage budget is consumed
by the condensed graph:
\begin{equation}
  \mathrm{SR}
  = \frac{\mathcal{C}(G_c)}{\mathcal{C}(G)} \times 100,
  \label{eq:sr}
\end{equation}
where $\mathcal{C}(\cdot)$ is the storage cost of Eq.~\eqref{eq:budget}.
A target compression fraction $r$ ideally yields
$\mathrm{SR} \approx r\times 100$, but in practice greedy BFS expansion
causes $\mathrm{SR}$ to saturate once there are no further nodes that fit
within the remaining budget.

\subsection{Feature variance}
To quantify how much of the original feature diversity is retained, we
compute the mean per-feature variance across all nodes in the condensed
graph, normalised by the same quantity computed on the full training set:
\begin{equation}
  \widehat{\sigma}^2
  = \frac{1}{k^*}\sum_{j\in\mathcal{F}}
    \frac{\operatorname{Var}(\tilde{\mathbf{X}}[V_c, j])}
         {\operatorname{Var}(\tilde{\mathbf{X}}[\Vtr, j]) + \epsilon}.
  \label{eq:feat_var}
\end{equation}
Values close to $1$ indicate that the condensed graph preserves the
feature spread of the original training set; values below $1$ suggest
the condensed graph over-concentrates on a narrow region of feature space.

\subsection{Class entropy}
The label diversity of the condensed graph is measured by the entropy
of its empirical class distribution:
\begin{equation}
  H_c = -\sum_{c=1}^{C} \hat{p}_c \log \hat{p}_c,
  \label{eq:class_ent}
\end{equation}
where $\hat{p}_c = |\{v \in V_c : y_v = c\}|/|V_c|$.
Higher entropy indicates better class balance.

Tables~\ref{tab:stats_roman}--\ref{tab:stats_squirrel} report $|V_c|$,
$|E_c|$, average degree $\bar{d}$, class entropy $H_c$, normalised
feature variance $\widehat{\sigma}^2$, and $\mathrm{SR}(\%)$ for both
condensers across all four compression fractions.

\begin{table}[htbp]
\caption{Condensed graph statistics on \textbf{Roman-empire}
($N=22{,}662$, $E=65{,}854$, $F=300$, $h=0.968$).
$\bar{d}$: mean degree;
\texorpdfstring{$H_c$}{Hc}: class entropy
(max \texorpdfstring{$=\log 18 \approx 2.89$}{= log18 approx 2.89});
\texorpdfstring{$\widehat{\sigma}^2$}{sigma\textasciicircum 2}:
normalised feature variance; SR: storage ratio (\%).}
\label{tab:stats_roman}
\centering
\setlength{\tabcolsep}{4pt}
\begin{tabular}{@{}llrrrrrc@{}}
\toprule
\textbf{Method} & \textbf{$r$}
  & $|V_c|$ & $|E_c|$ & $\bar{d}$
  & $H_c$ & $\widehat{\sigma}^2$ & SR\,(\%) \\
\midrule
\multirow{4}{*}{BONSAI}
  & 0.0001 & 15,110 & 21,735 & 2.88 & 2.617 & 0.954 & 66.67 \\
  & 0.005  & 19,512 & 32,762 & 3.36 & 2.613 & 0.996 & 86.23 \\
  & 0.01   & 21,548 & 35,238 & 3.27 & 2.613 & 1.003 & 95.20 \\
  & 0.03   & 21,548 & 35,238 & 3.27 & 2.613 & 1.003 & 95.20 \\
\midrule
\multirow{4}{*}{HERALD}
  & 0.0001 & 18,697 & 26,140 & 2.80 & 2.628 & 0.965 & 82.47 \\
  & 0.005  & 22,559 & 32,734 & 2.90 & 2.613 & 1.000 & 99.54 \\
  & 0.01   & 22,559 & 32,734 & 2.90 & 2.613 & 1.000 & 99.54 \\
  & 0.03   & 22,559 & 32,734 & 2.90 & 2.613 & 1.000 & 99.54 \\
\midrule
\multicolumn{2}{@{}l}{Full graph (train)}
  & 22,662 & 65,854 & 5.81 & 2.890 & 1.000 & 100.00 \\
\bottomrule
\end{tabular}
\end{table}

\begin{table}[htbp]
\caption{Condensed graph statistics on \textbf{Amazon-ratings}
($N=24{,}492$, $E=186{,}100$, $F=300$, $h=0.620$).}
\label{tab:stats_amazon}
\centering
\setlength{\tabcolsep}{4pt}
\begin{tabular}{@{}llrrrrrc@{}}
\toprule
\textbf{Method} & \textbf{$r$}
  & $|V_c|$ & $|E_c|$ & $\bar{d}$
  & $H_c$ & $\widehat{\sigma}^2$ & SR\,(\%) \\
\midrule
\multirow{4}{*}{BONSAI}
  & 0.0001 & 11,484 & 37,428 & 6.52 & 1.404 & 0.991 & 46.72 \\
  & 0.005  & 16,174 & 58,828 & 7.27 & 1.407 & 1.001 & 65.97 \\
  & 0.01   & 20,618 & 80,568 & 7.82 & 1.407 & 1.002 & 84.24 \\
  & 0.03   & 24,385 & 96,944 & 7.95 & 1.407 & 1.001 & 99.68 \\
\midrule
\multirow{4}{*}{HERALD}
  & 0.0001 & 12,110 & 41,822 & 6.91 & 1.463 & 0.937 & 49.33 \\
  & 0.005  & 19,132 & 67,092 & 7.01 & 1.407 & 1.001 & 77.97 \\
  & 0.01   & 24,221 & 91,794 & 7.58 & 1.407 & 0.999 & 98.89 \\
  & 0.03   & 24,492 & 93,050 & 7.60 & 1.407 & 1.000 & 100.00 \\
\midrule
\multicolumn{2}{@{}l}{Full graph (train)}
  & 24,492 & 186,100 & 15.20 & 1.609 & 1.000 & 100.00 \\
\bottomrule
\end{tabular}
\end{table}

\begin{table}[htbp]
\caption{Condensed graph statistics on \textbf{Squirrel}
($N=5{,}201$, $E=217{,}073$, $F=2{,}089$, $h=0.776$).}
\label{tab:stats_squirrel}
\centering
\setlength{\tabcolsep}{4pt}
\begin{tabular}{@{}llrrrrrc@{}}
\toprule
\textbf{Method} & \textbf{$r$}
  & $|V_c|$ & $|E_c|$ & $\bar{d}$
  & $H_c$ & $\widehat{\sigma}^2$ & SR\,(\%) \\
\midrule
\multirow{4}{*}{BONSAI}
  & 0.0001 & 2,230 &  42,153 & 37.80 & 1.609 & 0.751 &  9.45 \\
  & 0.005  & 2,911 & 114,386 & 78.59 & 1.609 & 1.014 & 13.41 \\
  & 0.01   & 3,592 & 129,452 & 72.08 & 1.609 & 1.003 & 16.34 \\
  & 0.03   & 3,820 & 130,775 & 68.47 & 1.609 & 0.969 & 17.25 \\
\midrule
\multirow{4}{*}{HERALD}
  & 0.0001 & 2,199 &  37,892 & 34.46 & 1.599 & 0.801 &  9.26 \\
  & 0.005  & 3,643 & 129,755 & 71.23 & 1.609 & 0.999 & 16.54 \\
  & 0.01   & 3,868 & 131,455 & 67.97 & 1.609 & 0.965 & 17.45 \\
  & 0.03   & 3,868 & 131,455 & 67.97 & 1.609 & 0.965 & 17.45 \\
\midrule
\multicolumn{2}{@{}l}{Full graph (train)}
  & 5,201 & 217,073 & 83.46 & 1.609 & 1.000 & 100.00 \\
\bottomrule
\end{tabular}
\end{table}

\subsection{Discussion}
\label{sec:stats_discussion}

Several structural patterns emerge from Figure~\ref{fig:condensed_stats}
and Tables~\ref{tab:stats_roman}--\ref{tab:stats_squirrel}.

\textbf{SR saturation.}
On Roman-empire and Squirrel, both condensers saturate well below the
target compression fraction at higher $r$ values (Figure~\ref{fig:condensed_stats},
middle row).
For Roman-empire at $r=0.01$ and $r=0.03$, both BONSAI and HERALD produce
identical condensed graphs ($\mathrm{SR}=95.20\%$ and $99.54\%$
respectively), because the BFS expansion has already covered most of the
available training neighbourhood.
This reflects a fundamental property of the budget formula
(Eq.~\eqref{eq:budget}): on sparse graphs where edge storage costs are
low relative to feature storage, the node budget is exhausted before the
edge budget, and the condensed graph approaches the full training subgraph.
Squirrel illustrates the opposite extreme: its extremely high density
($\bar{d}\approx 83$ in the full graph) means that edges dominate the
storage cost, capping $\mathrm{SR}$ at around $17\%$ regardless of $r$.

\textbf{HERALD selects more nodes at the same budget.}
Across all three datasets and all fractions, HERALD's condensed graphs
contain more nodes than BONSAI's at equivalent $r$
(Figure~\ref{fig:condensed_stats}, top row; ratio
$|V_c^{\text{HERALD}}|/|V_c^{\text{BONSAI}}|$ ranges from $0.99$ to $1.25$).
This is a direct consequence of HERALD's feature selection: by choosing
features with lower average activation density than BONSAI's DT-selected
features, each node's per-node storage cost $f_v$ is smaller, so more
nodes fit within the same budget $\mathcal{B}$.
The additional nodes are preferentially drawn from class boundaries and
high-LID regions, explaining HERALD's accuracy gains on heterophilic
datasets.

\textbf{HERALD produces sparser condensed graphs.}
Despite containing more nodes, HERALD's condensed graphs are
consistently sparser than those of BONSAI (Figure~\ref{fig:condensed_stats_appendix},
Appendix~\ref{sec:stats}).
Boundary nodes (high $s_v^{(b)}$) and high-LID nodes tend to be
structurally peripheral, connecting to neighbours of different classes
rather than forming tight same class cliques.
BFS expansion from such roots therefore produces chains and trees rather
than dense cliques, reducing edge count relative to node count.
This sparsity reduces over-smoothing risk during neighbourhood aggregation
on the condensed graph.

\textbf{Class balance.}
On Amazon-ratings at $r=0.0001$, HERALD achieves higher class entropy
($H_c=1.463$) than BONSAI ($H_c=1.404$), indicating better class balance
in the most budget-constrained regime.
Boundary nodes are by definition adjacent to multiple classes, so selecting
them as roots naturally draws in neighbours from different classes during
BFS expansion.
At larger budgets the class rebalancing step (Stage~7) equalises the
distributions and the entropy gap closes.

\textbf{Feature diversity.}
At $r=0.0001$ on Squirrel, HERALD achieves $\widehat{\sigma}^2=0.801$
versus BONSAI's $0.751$, a gain of $+0.050$
(Figure~\ref{fig:condensed_stats}, bottom row).
This confirms that HERALD's LID-driven diversity term actively prevents
the condensed set from collapsing onto a narrow cluster of near-duplicate
high-prototype-score nodes, preserving more of the original feature
manifold geometry even at severe compression.
The gap narrows as the budget grows, consistent with the LID term's role
being most critical precisely when the budget forces a small, selective
subset.

\begin{figure*}[t]
  \centering
  \includegraphics[width=\textwidth]{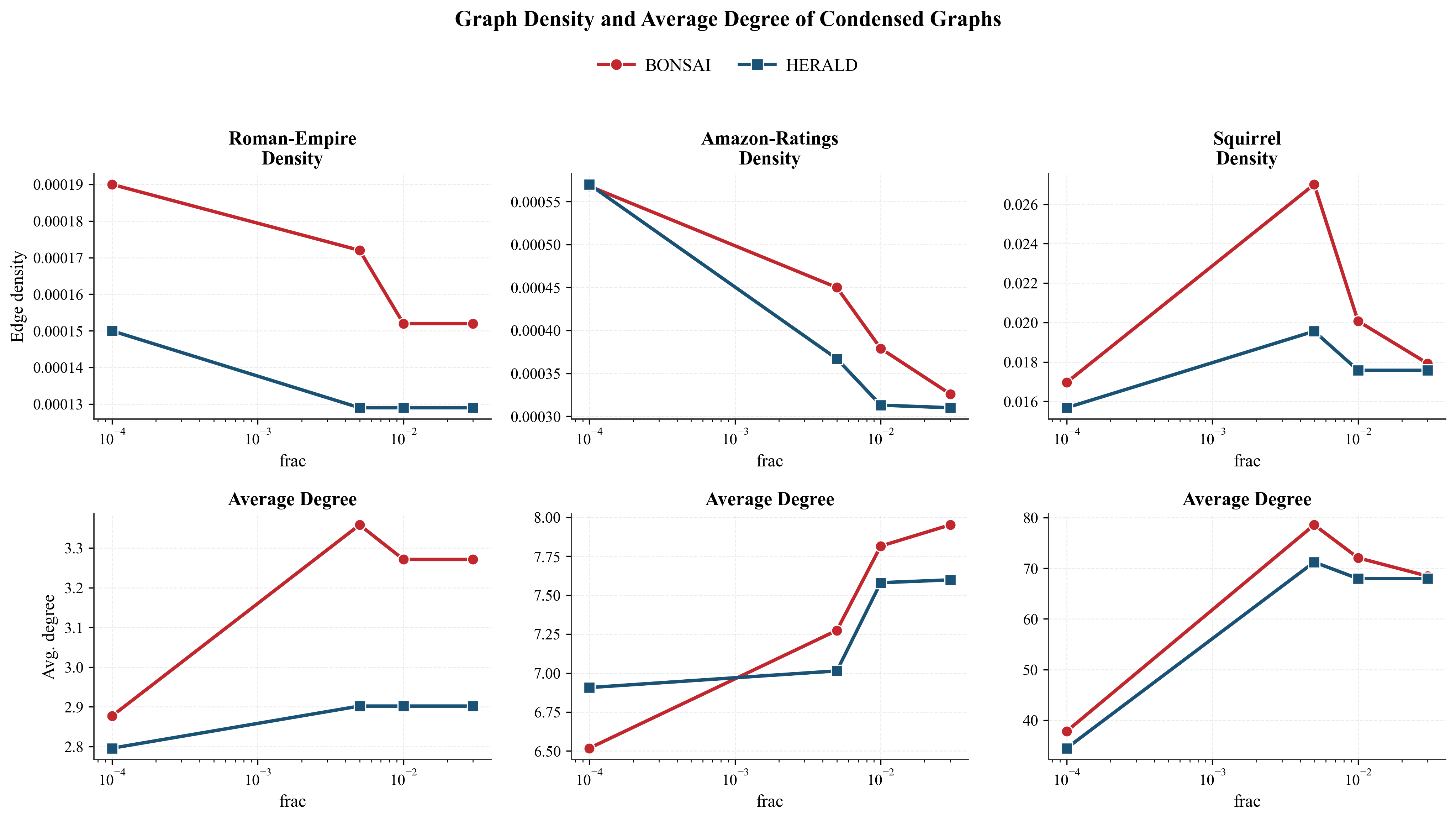}
  \includegraphics[width=\textwidth]{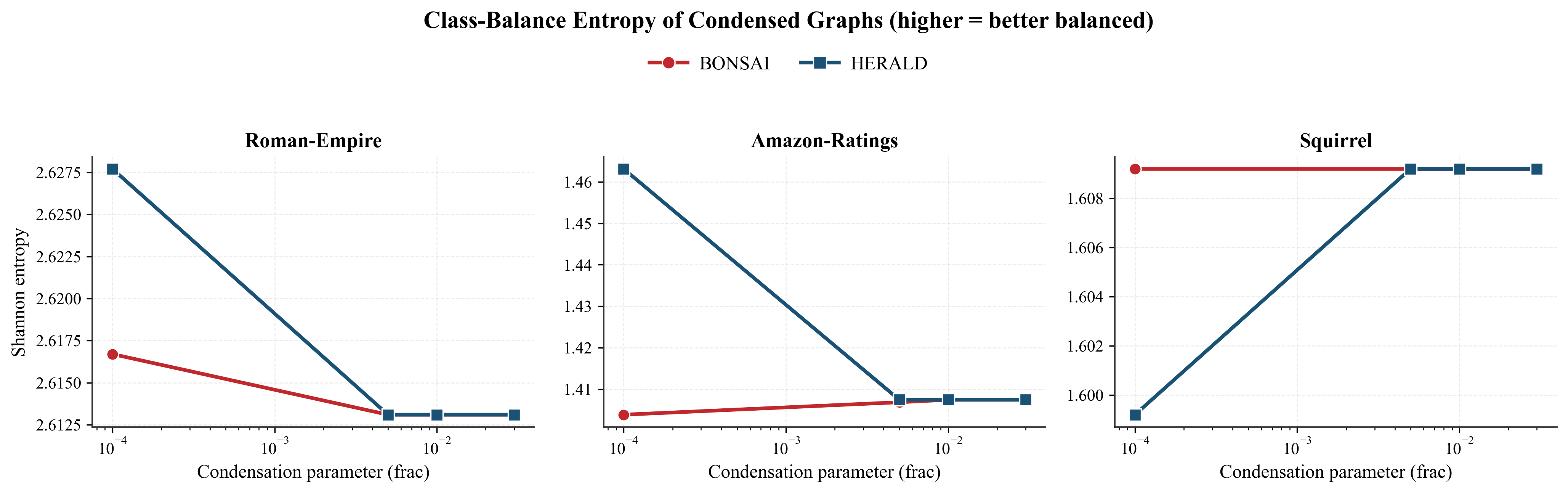}
  \caption{Graph density, average degree, and class-balance entropy
  of condensed graphs produced by BONSAI and HERALD.
  HERALD is consistently sparser at matched node count (top), and
  achieves higher class entropy at tight budgets (bottom).}
  \label{fig:condensed_stats_appendix}
\end{figure*}

\section{Results on Large-Scale Graphs}
\label{sec:scalability}
To evaluate whether HERALD scales beyond medium-sized benchmark graphs, we
perform an additional study on the Reddit dataset
($232{,}965$ nodes and $57.3$M edges), which is the largest graph considered
in our benchmark. Unlike the smaller citation and heterophilic datasets,
Reddit presents both substantially higher graph connectivity and a much
larger feature space (602 input features across 41 classes), making it a
challenging setting for graph condensation.
The experimental protocol follows exactly the same evaluation pipeline used
throughout Section~\ref{sec:experiments}. HERALD's joint feature selector is
applied as in all other experiments, using the same feature budget as
Bonsai's decision-tree selection; on Reddit, Bonsai's tree touches all $500$
of $500$ raw features, so HERALD's selector likewise retains the full
feature set at this stage, with dimensionality reduction instead occurring
via the downstream PCA step described below. All condensers are evaluated
under the same storage budgets $r\in\{0.0001,0.005,0.01,0.03\}$ and the
resulting condensed graphs are used to train four downstream GNN
architectures (GCN, GAT, GIN, and H2GCN). The same train/validation/test
split generation, training hyperparameters, and evaluation procedure are
used for every method. For HERALD, the condensed feature matrix is
additionally compressed using PCA while preserving 90\% of the feature
variance before downstream training, allowing a larger fraction of the
storage budget to be allocated to representative nodes rather than feature
dimensions.

\subsection{Accuracy at Scale}
\label{sec:scalability_acc}
Table~\ref{tab:reddit_acc} summarizes the node-classification accuracy on
Reddit. Overall, HERALD demonstrates strong performance across all four
GNN architectures under extremely aggressive compression ratios, winning
13 of the 16 condenser/GNN/budget cells against the strongest competing
baseline.

For GAT and GIN, HERALD outperforms every competing condenser at every
storage budget. On GAT, HERALD reaches $42.09\%$ at the smallest budget
($r=0.0001$), already well ahead of BONSAI ($23.86\%$), Herding
($27.50\%$), and Random ($20.94\%$), and maintains a comparable margin of
roughly 11--15 percentage points over the next-best method across all four
budgets, reaching $43.79\%$ at $r=0.03$. On GIN, HERALD's advantage is
similarly consistent, peaking at $45.10\%$ at $r=0.01$ compared with
BONSAI's $42.36\%$.

For H2GCN, HERALD obtains the best performance at three of the four
compression ratios, reaching $45.83\%$ accuracy at $r=0.03$ — within
approximately $5.1$ percentage points of training on the full Reddit graph.
The one exception is the smallest budget ($r=0.0001$), where HERALD trails
Random ($42.15\%$) at $32.93\%$, the largest gap in favor of a baseline
observed anywhere in the table; we attribute this to the very small
condensed graphs at this budget (as few as 44--51 nodes after PPR pruning),
where HERALD's node scoring has too little material to stabilize H2GCN
training.

For GCN, HERALD wins three of four budgets, including the smallest
($38.46\%$ at $r=0.0001$) and largest ($47.30\%$ at $r=0.03$), the latter
representing HERALD's best result on Reddit and a $10.3$-point margin over
BONSAI ($36.99\%$). BONSAI remains competitive at $r=0.01$, where it edges
out HERALD ($42.20\%$ vs.\ $38.32\%$) — the only cell in the table where a
baseline condenser outperforms HERALD on GCN or GIN, and consistent with
BONSAI's prototype-based selection being well suited to Reddit's high
homophily at moderate budgets.

Overall, these results indicate that the advantages of HERALD are largely
preserved at Reddit scale: the proposed node-scoring strategy generalizes
well beyond the medium-sized benchmark datasets considered earlier, with
its two weakest points — H2GCN at the most extreme compression and GCN at
$r=0.01$ — both traceable to the very small node budgets available under
tight storage constraints rather than a systematic weakness of the method.

\begin{table}[htbp]
\caption{Node classification accuracy (\%) on \textbf{Reddit}.
Best condensed result per GNN column and compression ratio highlighted.}
\label{tab:reddit_acc}
\centering
\setlength{\tabcolsep}{4pt}
\begin{tabular}{@{}llcccc@{}}
\toprule
\textbf{Condenser} & \textbf{$r$}
  & \textbf{GCN} & \textbf{GAT} & \textbf{GIN} & \textbf{H2GCN} \\
\midrule
Full graph
  & ---
  & $52.48{\scriptstyle\pm0.05}$
  & $51.66{\scriptstyle\pm0.73}$
  & $47.73{\scriptstyle\pm1.10}$
  & $50.91{\scriptstyle\pm0.21}$ \\
\midrule
\multirow{4}{*}{BONSAI}
  & 0.0001 & $37.97{\scriptstyle\pm0.95}$ & $23.86{\scriptstyle\pm11.32}$ & $40.79{\scriptstyle\pm1.35}$ & $37.05{\scriptstyle\pm2.97}$ \\
  & 0.005  & $30.29{\scriptstyle\pm5.04}$ & $26.83{\scriptstyle\pm8.83}$ & $41.15{\scriptstyle\pm5.33}$ & $41.31{\scriptstyle\pm0.57}$ \\
  & 0.01   & \cellcolor{bestcol}$42.20{\scriptstyle\pm2.79}$ & $24.62{\scriptstyle\pm7.80}$ & $42.36{\scriptstyle\pm2.20}$ & $42.35{\scriptstyle\pm0.84}$ \\
  & 0.03   & $36.99{\scriptstyle\pm3.50}$ & $33.13{\scriptstyle\pm4.58}$ & $35.18{\scriptstyle\pm5.66}$ & $43.26{\scriptstyle\pm1.09}$ \\
\midrule
\multirow{4}{*}{Herding}
  & 0.0001 & $26.82{\scriptstyle\pm3.04}$ & $27.50{\scriptstyle\pm10.65}$ & $26.26{\scriptstyle\pm2.92}$ & $37.49{\scriptstyle\pm3.44}$ \\
  & 0.005  & $26.60{\scriptstyle\pm1.53}$ & $25.62{\scriptstyle\pm6.49}$ & $34.24{\scriptstyle\pm4.75}$ & $39.82{\scriptstyle\pm1.77}$ \\
  & 0.01   & $27.28{\scriptstyle\pm2.28}$ & $26.14{\scriptstyle\pm3.54}$ & $25.05{\scriptstyle\pm1.71}$ & $39.84{\scriptstyle\pm1.76}$ \\
  & 0.03   & $32.15{\scriptstyle\pm3.33}$ & $27.88{\scriptstyle\pm3.42}$ & $27.26{\scriptstyle\pm1.86}$ & $40.61{\scriptstyle\pm0.87}$ \\
\midrule
\multirow{4}{*}{Random}
  & 0.0001 & $34.57{\scriptstyle\pm1.91}$ & $20.94{\scriptstyle\pm11.52}$ & $32.18{\scriptstyle\pm1.86}$ & \cellcolor{bestcol}$42.15{\scriptstyle\pm0.51}$ \\
  & 0.005  & $25.90{\scriptstyle\pm2.43}$ & $21.46{\scriptstyle\pm4.85}$ & $27.72{\scriptstyle\pm1.54}$ & $40.49{\scriptstyle\pm0.76}$ \\
  & 0.01   & $27.15{\scriptstyle\pm4.00}$ & $21.13{\scriptstyle\pm4.15}$ & $22.89{\scriptstyle\pm4.58}$ & $39.27{\scriptstyle\pm2.95}$ \\
  & 0.03   & $27.20{\scriptstyle\pm3.10}$ & $28.87{\scriptstyle\pm7.41}$ & $21.27{\scriptstyle\pm4.04}$ & $39.43{\scriptstyle\pm2.04}$ \\
\midrule
\multirow{4}{*}{\textbf{HERALD}}
  & 0.0001 & \cellcolor{bestcol}$38.46{\scriptstyle\pm6.32}$ & \cellcolor{bestcol}$42.09{\scriptstyle\pm0.31}$ & \cellcolor{bestcol}$42.29{\scriptstyle\pm0.00}$ & $32.93{\scriptstyle\pm5.99}$ \\
  & 0.005  & \cellcolor{bestcol}$46.96{\scriptstyle\pm2.91}$ & \cellcolor{bestcol}$43.04{\scriptstyle\pm0.80}$ & \cellcolor{bestcol}$44.99{\scriptstyle\pm2.23}$ & \cellcolor{bestcol}$43.85{\scriptstyle\pm0.83}$ \\
  & 0.01   & $38.32{\scriptstyle\pm5.83}$ & \cellcolor{bestcol}$43.37{\scriptstyle\pm1.67}$ & \cellcolor{bestcol}$45.10{\scriptstyle\pm2.40}$ & \cellcolor{bestcol}$44.74{\scriptstyle\pm0.69}$ \\
  & 0.03   & \cellcolor{bestcol}$47.30{\scriptstyle\pm2.18}$ & \cellcolor{bestcol}$43.79{\scriptstyle\pm1.61}$ & \cellcolor{bestcol}$44.74{\scriptstyle\pm2.15}$ & \cellcolor{bestcol}$45.83{\scriptstyle\pm0.48}$ \\
\bottomrule
\end{tabular}
\end{table}

\section{Dataset wise results}
\label{sec:dataset_wise_results}

\subsection{Homophilic Datasets}
\label{sec:results_homo}

Tables~\ref{tab:cora}--\ref{tab:pubmed} report results on Cora, CiteSeer,
and PubMed.

\begin{table}[htbp]
\caption{Node classification accuracy (\%) on \textbf{Cora}.
Best condensed result per GNN column and compression ratio highlighted.}
\label{tab:cora}
\centering
\setlength{\tabcolsep}{4pt}
\begin{tabular}{@{}llcccc@{}}
\toprule
\textbf{Condenser} & \textbf{$r$}
  & \textbf{GCN} & \textbf{GAT} & \textbf{GIN} & \textbf{H2GCN} \\
\midrule
Full graph & --- &
$87.60{\scriptstyle\pm0.34}$ &
$85.42{\scriptstyle\pm0.17}$ &
$87.27{\scriptstyle\pm0.51}$ &
$85.28{\scriptstyle\pm0.67}$ \\
\midrule
\multirow{4}{*}{Random}
 & 0.0001 & $9.56{\scriptstyle\pm2.21}$ & $15.83{\scriptstyle\pm6.21}$ & $14.02{\scriptstyle\pm5.48}$ & $12.92{\scriptstyle\pm2.44}$ \\
 & 0.005  & $14.54{\scriptstyle\pm2.46}$ & $22.88{\scriptstyle\pm7.52}$ & $18.78{\scriptstyle\pm5.23}$ & $23.91{\scriptstyle\pm2.51}$ \\
 & 0.01   & $16.64{\scriptstyle\pm0.65}$ & $21.14{\scriptstyle\pm4.52}$ & $19.56{\scriptstyle\pm2.71}$ & $15.39{\scriptstyle\pm1.78}$ \\
 & 0.03   & $27.08{\scriptstyle\pm0.56}$ & $29.23{\scriptstyle\pm5.13}$ & $29.41{\scriptstyle\pm4.42}$ & $32.32{\scriptstyle\pm2.10}$ \\
\midrule
\multirow{4}{*}{Herding}
 & 0.0001 & $27.23{\scriptstyle\pm0.80}$ & $31.18{\scriptstyle\pm4.96}$ & $30.81{\scriptstyle\pm2.56}$ & $35.54{\scriptstyle\pm1.07}$ \\
 & 0.005  & $27.23{\scriptstyle\pm0.80}$ & $31.18{\scriptstyle\pm4.96}$ & $30.81{\scriptstyle\pm2.56}$ & $35.54{\scriptstyle\pm1.07}$ \\
 & 0.01   & $56.16{\scriptstyle\pm0.85}$ & $53.95{\scriptstyle\pm0.80}$ & $59.11{\scriptstyle\pm0.92}$ & $57.05{\scriptstyle\pm0.51}$ \\
 & 0.03   & $71.73{\scriptstyle\pm0.96}$ & $72.88{\scriptstyle\pm1.69}$ & $74.10{\scriptstyle\pm0.95}$ & $71.18{\scriptstyle\pm1.67}$ \\
\midrule
\multirow{4}{*}{BONSAI}
 & 0.0001 & $53.21{\scriptstyle\pm0.43}$ & $54.72{\scriptstyle\pm2.67}$ & $62.36{\scriptstyle\pm0.42}$ & $53.03{\scriptstyle\pm1.16}$ \\
 & 0.005  & $84.72{\scriptstyle\pm0.27}$ & $81.77{\scriptstyle\pm0.47}$ & \cellcolor{bestcol}$85.61{\scriptstyle\pm0.74}$ & $81.37{\scriptstyle\pm0.48}$ \\
 & 0.01   & $85.24{\scriptstyle\pm0.17}$ & $83.25{\scriptstyle\pm0.79}$ & $86.16{\scriptstyle\pm0.48}$ & $81.77{\scriptstyle\pm0.44}$ \\
 & 0.03   & $85.24{\scriptstyle\pm0.17}$ & $83.25{\scriptstyle\pm0.79}$ & $86.16{\scriptstyle\pm0.48}$ & $81.77{\scriptstyle\pm0.44}$ \\
\midrule
\multirow{4}{*}{GDEM}
& 0.0001 & $15.76{\scriptstyle\pm3.19}$ & $21.37{\scriptstyle\pm5.80}$ & \cellcolor{bestcol}$79.89{\scriptstyle\pm0.42}$ & \cellcolor{bestcol}$63.54{\scriptstyle\pm0.86}$ \\
& 0.005  & $39.08{\scriptstyle\pm2.03}$ & $37.82{\scriptstyle\pm2.41}$ & $50.37{\scriptstyle\pm7.92}$ & $60.26{\scriptstyle\pm0.85}$ \\
& 0.010  & $35.06{\scriptstyle\pm1.59}$ & $36.05{\scriptstyle\pm1.47}$ & $58.12{\scriptstyle\pm4.69}$ & $62.88{\scriptstyle\pm0.59}$ \\
& 0.030  & $34.72{\scriptstyle\pm0.72}$ & $34.65{\scriptstyle\pm0.65}$ & $29.93{\scriptstyle\pm5.15}$ & $72.07{\scriptstyle\pm0.53}$ \\
\midrule
\multirow{4}{*}{\textbf{HERALD}}
 & 0.0001 & \cellcolor{bestcol}$60.59{\scriptstyle\pm0.48}$ & \cellcolor{bestcol}$60.22{\scriptstyle\pm2.68}$ & $71.14{\scriptstyle\pm0.45}$ & $25.42{\scriptstyle\pm1.91}$ \\
 & 0.005  & \cellcolor{bestcol}$85.54{\scriptstyle\pm0.28}$ & \cellcolor{bestcol}$82.80{\scriptstyle\pm0.54}$ & $85.28{\scriptstyle\pm0.79}$ & \cellcolor{bestcol}$81.59{\scriptstyle\pm0.41}$ \\
 & 0.01   & \cellcolor{bestcol}$86.90{\scriptstyle\pm0.23}$ & \cellcolor{bestcol}$83.84{\scriptstyle\pm0.95}$ & \cellcolor{bestcol}$86.79{\scriptstyle\pm0.83}$ & \cellcolor{bestcol}$82.55{\scriptstyle\pm0.50}$ \\
 & 0.03   & \cellcolor{bestcol}$86.90{\scriptstyle\pm0.23}$ & \cellcolor{bestcol}$83.84{\scriptstyle\pm0.95}$ & \cellcolor{bestcol}$86.79{\scriptstyle\pm0.83}$ & \cellcolor{bestcol}$82.55{\scriptstyle\pm0.50}$ \\
\bottomrule
\end{tabular}
\end{table}

\begin{table}[htbp]
\caption{Node classification accuracy (\%) on \textbf{CiteSeer}.}
\label{tab:citeseer}
\centering
\setlength{\tabcolsep}{4pt}
\begin{tabular}{@{}llcccc@{}}
\toprule
\textbf{Condenser} & \textbf{$r$}
  & \textbf{GCN} & \textbf{GAT} & \textbf{GIN} & \textbf{H2GCN} \\
\midrule
Full graph & --- &
$78.74{\scriptstyle\pm0.24}$ &
$77.63{\scriptstyle\pm0.63}$ &
$76.73{\scriptstyle\pm1.14}$ &
$77.60{\scriptstyle\pm0.48}$ \\
\midrule
\multirow{4}{*}{Random}
 & 0.0001 & $17.57{\scriptstyle\pm2.57}$ & $17.72{\scriptstyle\pm2.63}$ & $19.49{\scriptstyle\pm1.09}$ & $15.65{\scriptstyle\pm1.52}$ \\
 & 0.005  & $24.92{\scriptstyle\pm3.70}$ & $23.90{\scriptstyle\pm3.23}$ & $36.46{\scriptstyle\pm2.47}$ & $26.28{\scriptstyle\pm1.12}$ \\
 & 0.01   & $37.96{\scriptstyle\pm0.74}$ & $37.54{\scriptstyle\pm0.46}$ & $44.23{\scriptstyle\pm1.22}$ & $36.67{\scriptstyle\pm0.90}$ \\
 & 0.03   & $35.65{\scriptstyle\pm0.76}$ & $36.10{\scriptstyle\pm0.98}$ & $39.58{\scriptstyle\pm0.53}$ & $39.55{\scriptstyle\pm1.52}$ \\
\midrule
\multirow{4}{*}{Herding}
 & 0.0001 & $29.94{\scriptstyle\pm2.26}$ & $31.77{\scriptstyle\pm3.92}$ & $49.67{\scriptstyle\pm3.07}$ & $34.65{\scriptstyle\pm0.42}$ \\
 & 0.005  & $39.85{\scriptstyle\pm3.67}$ & $42.73{\scriptstyle\pm5.27}$ & $59.61{\scriptstyle\pm0.91}$ & $46.19{\scriptstyle\pm0.66}$ \\
 & 0.01   & $60.54{\scriptstyle\pm1.45}$ & $61.56{\scriptstyle\pm1.69}$ & $65.62{\scriptstyle\pm1.03}$ & $61.68{\scriptstyle\pm1.16}$ \\
 & 0.03   & $71.62{\scriptstyle\pm0.58}$ & $71.77{\scriptstyle\pm0.89}$ & $69.85{\scriptstyle\pm0.79}$ & $69.40{\scriptstyle\pm1.02}$ \\
\midrule
\multirow{4}{*}{BONSAI}
 & 0.0001 & \cellcolor{bestcol}$64.83{\scriptstyle\pm0.72}$ & \cellcolor{bestcol}$65.35{\scriptstyle\pm1.02}$ & $65.89{\scriptstyle\pm0.73}$ & \cellcolor{bestcol}$62.31{\scriptstyle\pm1.63}$ \\
 & 0.005  & $76.10{\scriptstyle\pm0.44}$ & $75.71{\scriptstyle\pm0.94}$ & $75.02{\scriptstyle\pm0.51}$ & $72.07{\scriptstyle\pm0.78}$ \\
 & 0.01   & $76.10{\scriptstyle\pm0.44}$ & $75.71{\scriptstyle\pm0.94}$ & $75.02{\scriptstyle\pm0.51}$ & $72.07{\scriptstyle\pm0.78}$ \\
 & 0.03   & $76.10{\scriptstyle\pm0.44}$ & $75.71{\scriptstyle\pm0.94}$ & $75.02{\scriptstyle\pm0.51}$ & $72.07{\scriptstyle\pm0.78}$ \\
\midrule
\multirow{4}{*}{GDEM}
& 0.0001 & $18.62{\scriptstyle\pm2.94}$ & $18.50{\scriptstyle\pm2.66}$ & \cellcolor{bestcol}$67.45{\scriptstyle\pm1.12}$ & $59.61{\scriptstyle\pm0.77}$ \\
& 0.005  & $23.99{\scriptstyle\pm2.41}$ & $25.17{\scriptstyle\pm2.07}$ & $59.22{\scriptstyle\pm3.09}$ & $66.16{\scriptstyle\pm0.56}$ \\
& 0.010  & $21.68{\scriptstyle\pm0.78}$ & $21.17{\scriptstyle\pm0.39}$ & $44.14{\scriptstyle\pm4.92}$ & $66.91{\scriptstyle\pm0.60}$ \\
& 0.030  & $21.29{\scriptstyle\pm0.61}$ & $20.48{\scriptstyle\pm1.36}$ & $23.54{\scriptstyle\pm1.84}$ & $74.02{\scriptstyle\pm0.97}$ \\
\midrule
\multirow{4}{*}{\textbf{HERALD}}
 & 0.0001 & $64.41{\scriptstyle\pm0.65}$ & $63.84{\scriptstyle\pm0.77}$ & $64.59{\scriptstyle\pm0.84}$ & $55.62{\scriptstyle\pm1.17}$ \\
 & 0.005  & \cellcolor{bestcol}$76.31{\scriptstyle\pm0.48}$ & \cellcolor{bestcol}$76.70{\scriptstyle\pm0.59}$ & \cellcolor{bestcol}$75.77{\scriptstyle\pm1.20}$ & \cellcolor{bestcol}$74.29{\scriptstyle\pm1.00}$ \\
 & 0.01   & \cellcolor{bestcol}$76.31{\scriptstyle\pm0.48}$ & \cellcolor{bestcol}$76.70{\scriptstyle\pm0.59}$ & \cellcolor{bestcol}$75.77{\scriptstyle\pm1.20}$ & \cellcolor{bestcol}$74.29{\scriptstyle\pm1.00}$ \\
 & 0.03   & \cellcolor{bestcol}$76.31{\scriptstyle\pm0.48}$ & \cellcolor{bestcol}$76.70{\scriptstyle\pm0.59}$ & \cellcolor{bestcol}$75.77{\scriptstyle\pm1.20}$ & \cellcolor{bestcol}$74.29{\scriptstyle\pm1.00}$ \\
\bottomrule
\end{tabular}
\end{table}

\begin{table}[htbp]
\caption{Node classification accuracy (\%) on \textbf{PubMed}.}
\label{tab:pubmed}
\centering
\setlength{\tabcolsep}{4pt}
\begin{tabular}{@{}llcccc@{}}
\toprule
\textbf{Condenser} & \textbf{$r$}
  & \textbf{GCN} & \textbf{GAT} & \textbf{GIN} & \textbf{H2GCN} \\
\midrule
Full graph & --- &
$85.74{\scriptstyle\pm0.05}$ &
$85.14{\scriptstyle\pm0.37}$ &
$84.71{\scriptstyle\pm0.16}$ &
$86.87{\scriptstyle\pm0.16}$ \\
\midrule
\multirow{4}{*}{Random}
 & 0.0001 & $40.14{\scriptstyle\pm0.94}$ & $41.77{\scriptstyle\pm1.48}$ & $41.44{\scriptstyle\pm2.97}$ & $39.34{\scriptstyle\pm0.61}$ \\
 & 0.005  & $59.63{\scriptstyle\pm0.10}$ & $57.71{\scriptstyle\pm0.87}$ & $60.44{\scriptstyle\pm0.28}$ & $58.17{\scriptstyle\pm0.60}$ \\
 & 0.01   & $73.07{\scriptstyle\pm0.26}$ & $64.91{\scriptstyle\pm3.04}$ & $68.67{\scriptstyle\pm0.25}$ & $66.74{\scriptstyle\pm0.38}$ \\
 & 0.03   & $79.60{\scriptstyle\pm0.06}$ & $77.02{\scriptstyle\pm0.89}$ & $79.49{\scriptstyle\pm0.05}$ & $73.99{\scriptstyle\pm0.14}$ \\
\midrule
\multirow{4}{*}{Herding}
 & 0.0001 & $55.20{\scriptstyle\pm4.12}$ & $52.32{\scriptstyle\pm3.77}$ & $64.05{\scriptstyle\pm1.46}$ & \cellcolor{bestcol}$56.60{\scriptstyle\pm1.85}$ \\
 & 0.005  & $81.47{\scriptstyle\pm0.09}$ & $78.92{\scriptstyle\pm0.85}$ & \cellcolor{bestcol}$81.88{\scriptstyle\pm0.20}$ & $76.70{\scriptstyle\pm0.39}$ \\
 & 0.01   & $83.61{\scriptstyle\pm0.11}$ & $81.60{\scriptstyle\pm0.44}$ & \cellcolor{bestcol}$82.08{\scriptstyle\pm0.08}$ & $79.10{\scriptstyle\pm0.36}$ \\
 & 0.03   & $84.38{\scriptstyle\pm0.13}$ & $83.50{\scriptstyle\pm0.45}$ & $83.69{\scriptstyle\pm0.15}$ & $81.15{\scriptstyle\pm0.29}$ \\
\midrule
\multirow{4}{*}{BONSAI}
 & 0.0001 & $43.34{\scriptstyle\pm1.28}$ & $46.43{\scriptstyle\pm4.16}$ & $53.70{\scriptstyle\pm1.11}$ & $50.65{\scriptstyle\pm0.81}$ \\
 & 0.005  & \cellcolor{bestcol}$83.95{\scriptstyle\pm0.06}$ & \cellcolor{bestcol}$82.35{\scriptstyle\pm0.41}$ & $81.86{\scriptstyle\pm0.10}$ & \cellcolor{bestcol}$79.55{\scriptstyle\pm0.26}$ \\
 & 0.01   & \cellcolor{bestcol}$84.16{\scriptstyle\pm0.09}$ & \cellcolor{bestcol}$82.77{\scriptstyle\pm0.27}$ & $81.47{\scriptstyle\pm0.23}$ & $80.36{\scriptstyle\pm0.10}$ \\
 & 0.03   & $84.68{\scriptstyle\pm0.05}$ & $83.72{\scriptstyle\pm0.18}$ & $84.28{\scriptstyle\pm0.10}$ & $81.99{\scriptstyle\pm0.20}$ \\
\midrule
\multirow{4}{*}{GDEM}
& 0.0001 & $41.78{\scriptstyle\pm1.40}$ & $41.01{\scriptstyle\pm4.48}$ & $28.34{\scriptstyle\pm3.69}$ & $51.51{\scriptstyle\pm0.52}$ \\
& 0.005  & $42.75{\scriptstyle\pm2.72}$ & $42.42{\scriptstyle\pm2.43}$ & $43.10{\scriptstyle\pm7.43}$ & $76.76{\scriptstyle\pm0.07}$ \\
& 0.010  & $42.72{\scriptstyle\pm3.09}$ & $44.46{\scriptstyle\pm3.69}$ & $41.02{\scriptstyle\pm7.03}$ & $74.58{\scriptstyle\pm0.27}$ \\
& 0.030  & $42.87{\scriptstyle\pm2.76}$ & $45.44{\scriptstyle\pm5.32}$ & $40.42{\scriptstyle\pm8.05}$ & $74.99{\scriptstyle\pm0.16}$ \\
\midrule
\multirow{4}{*}{\textbf{HERALD}}
 & 0.0001 & \cellcolor{bestcol}$61.35{\scriptstyle\pm0.57}$ & \cellcolor{bestcol}$63.00{\scriptstyle\pm0.87}$ & \cellcolor{bestcol}$71.95{\scriptstyle\pm0.41}$ & $29.12{\scriptstyle\pm8.52}$ \\
 & 0.005  & $79.44{\scriptstyle\pm0.13}$ & $78.74{\scriptstyle\pm0.31}$ & $78.05{\scriptstyle\pm0.04}$ & $72.57{\scriptstyle\pm0.66}$ \\
 & 0.01   & $81.52{\scriptstyle\pm0.19}$ & $80.91{\scriptstyle\pm0.24}$ & $80.29{\scriptstyle\pm0.32}$ & \cellcolor{bestcol}$80.40{\scriptstyle\pm0.27}$ \\
 & 0.03   & \cellcolor{bestcol}$85.75{\scriptstyle\pm0.07}$ & \cellcolor{bestcol}$85.04{\scriptstyle\pm0.21}$ & \cellcolor{bestcol}$85.30{\scriptstyle\pm0.13}$ & \cellcolor{bestcol}$87.00{\scriptstyle\pm0.16}$ \\
\bottomrule
\end{tabular}
\end{table}

\subsection{Heterophilic Datasets}
\label{sec:results_hetero}

Tables~\ref{tab:roman}--\ref{tab:squirrel} report results on the four
heterophilic benchmarks.

\begin{table}[htbp]
\caption{Node classification accuracy (\%) on \textbf{Roman-empire}.}
\label{tab:roman}
\centering
\setlength{\tabcolsep}{4pt}
\begin{tabular}{@{}llcccc@{}}
\toprule
\textbf{Condenser} & \textbf{$r$}
  & \textbf{GCN} & \textbf{GAT} & \textbf{GIN} & \textbf{H2GCN} \\
\midrule
Full graph
& ---
& $41.82{\scriptstyle\pm0.37}$
& $44.13{\scriptstyle\pm0.81}$
& $38.53{\scriptstyle\pm0.52}$
& $76.57{\scriptstyle\pm0.48}$ \\
\midrule

\multirow{4}{*}{Random}
& 0.0001 & $8.26{\scriptstyle\pm0.46}$ & $7.01{\scriptstyle\pm1.94}$ & $7.79{\scriptstyle\pm0.79}$ & $13.02{\scriptstyle\pm0.47}$ \\
& 0.005  & $24.05{\scriptstyle\pm0.24}$ & $21.11{\scriptstyle\pm3.16}$ & $22.65{\scriptstyle\pm0.62}$ & $44.42{\scriptstyle\pm0.47}$ \\
& 0.010  & $26.32{\scriptstyle\pm0.40}$ & $23.22{\scriptstyle\pm1.68}$ & $24.30{\scriptstyle\pm0.54}$ & $49.67{\scriptstyle\pm0.27}$ \\
& 0.030  & $28.56{\scriptstyle\pm0.44}$ & $26.97{\scriptstyle\pm2.25}$ & $25.37{\scriptstyle\pm0.49}$ & $54.82{\scriptstyle\pm0.62}$ \\
\midrule

\multirow{4}{*}{Herding}
& 0.0001 & $19.97{\scriptstyle\pm0.40}$ & $17.67{\scriptstyle\pm0.98}$ & $19.08{\scriptstyle\pm0.47}$ & $31.95{\scriptstyle\pm0.77}$ \\
& 0.005  & $26.82{\scriptstyle\pm0.26}$ & $24.13{\scriptstyle\pm0.92}$ & $23.49{\scriptstyle\pm0.41}$ & $49.40{\scriptstyle\pm0.60}$ \\
& 0.010  & $29.39{\scriptstyle\pm0.45}$ & $27.84{\scriptstyle\pm1.15}$ & $25.44{\scriptstyle\pm0.16}$ & $55.54{\scriptstyle\pm0.62}$ \\
& 0.030  & $33.59{\scriptstyle\pm0.56}$ & $34.23{\scriptstyle\pm0.90}$ & $28.98{\scriptstyle\pm0.36}$ & $62.81{\scriptstyle\pm0.22}$ \\
\midrule

\multirow{4}{*}{BONSAI}
& 0.0001 & $39.74{\scriptstyle\pm0.18}$ & $41.75{\scriptstyle\pm0.65}$ & $33.52{\scriptstyle\pm0.43}$ & $72.12{\scriptstyle\pm0.18}$ \\
& 0.005  & $41.36{\scriptstyle\pm0.22}$ & $43.74{\scriptstyle\pm0.85}$ & $35.04{\scriptstyle\pm0.57}$ & $72.12{\scriptstyle\pm0.21}$ \\
& 0.010  & $41.57{\scriptstyle\pm0.20}$ & \cellcolor{bestcol}$44.40{\scriptstyle\pm0.70}$ & $35.22{\scriptstyle\pm0.46}$ & $72.51{\scriptstyle\pm0.39}$ \\
& 0.030  & $41.57{\scriptstyle\pm0.20}$ & \cellcolor{bestcol}$44.40{\scriptstyle\pm0.70}$ & $35.22{\scriptstyle\pm0.46}$ & $72.51{\scriptstyle\pm0.39}$ \\
\midrule
\multirow{4}{*}{GDEM}
& 0.0001 & $9.46{\scriptstyle\pm0.81}$ & $9.17{\scriptstyle\pm2.08}$ & $17.38{\scriptstyle\pm0.71}$ & $38.38{\scriptstyle\pm0.81}$ \\
& 0.005  & $15.28{\scriptstyle\pm0.48}$ & $12.47{\scriptstyle\pm1.79}$ & $5.95{\scriptstyle\pm1.49}$ & $46.67{\scriptstyle\pm1.36}$ \\
& 0.010  & $14.15{\scriptstyle\pm0.66}$ & $11.85{\scriptstyle\pm1.64}$ & $5.71{\scriptstyle\pm1.64}$ & $46.85{\scriptstyle\pm1.28}$ \\
& 0.030  & $14.26{\scriptstyle\pm0.61}$ & $12.73{\scriptstyle\pm1.65}$ & $5.84{\scriptstyle\pm1.47}$ & $46.90{\scriptstyle\pm1.28}$ \\
\midrule
\multirow{4}{*}{\textbf{HERALD}}
& 0.0001 & \cellcolor{bestcol}$41.85{\scriptstyle\pm0.15}$ & \cellcolor{bestcol}$43.62{\scriptstyle\pm0.56}$ & \cellcolor{bestcol}$38.29{\scriptstyle\pm0.45}$ & \cellcolor{bestcol}$76.00{\scriptstyle\pm0.35}$ \\
& 0.005  & \cellcolor{bestcol}$41.73{\scriptstyle\pm0.37}$ & \cellcolor{bestcol}$44.08{\scriptstyle\pm0.71}$ & \cellcolor{bestcol}$38.55{\scriptstyle\pm0.50}$ & \cellcolor{bestcol}$76.65{\scriptstyle\pm0.42}$ \\
& 0.010  & \cellcolor{bestcol}$41.73{\scriptstyle\pm0.37}$ & $44.08{\scriptstyle\pm0.71}$ & \cellcolor{bestcol}$38.55{\scriptstyle\pm0.50}$ & \cellcolor{bestcol}$76.65{\scriptstyle\pm0.42}$ \\
& 0.030  & \cellcolor{bestcol}$41.73{\scriptstyle\pm0.37}$ & $44.08{\scriptstyle\pm0.71}$ & \cellcolor{bestcol}$38.55{\scriptstyle\pm0.50}$ & \cellcolor{bestcol}$76.65{\scriptstyle\pm0.42}$ \\
\bottomrule
\end{tabular}
\end{table}

\begin{table}[htbp]
\caption{Node classification accuracy (\%) on \textbf{Amazon-ratings} ($h=0.620$).}
\label{tab:amazon_ratings}
\centering
\setlength{\tabcolsep}{4pt}
\begin{tabular}{@{}llcccc@{}}
\toprule
\textbf{Condenser} & \textbf{$r$}
  & \textbf{GCN} & \textbf{GAT} & \textbf{GIN} & \textbf{H2GCN} \\
\midrule
Full graph & --- &
$46.86{\scriptstyle\pm0.24}$ &
$45.45{\scriptstyle\pm0.39}$ &
$47.22{\scriptstyle\pm0.42}$ &
$50.98{\scriptstyle\pm0.19}$ \\
\midrule
\multirow{4}{*}{Random}
 & 0.0001 & $28.39{\scriptstyle\pm3.06}$ & $30.73{\scriptstyle\pm3.08}$ & $32.06{\scriptstyle\pm2.55}$ & $36.65{\scriptstyle\pm0.03}$ \\
 & 0.005  & $30.45{\scriptstyle\pm1.77}$ & $31.04{\scriptstyle\pm1.09}$ & $28.69{\scriptstyle\pm2.81}$ & $35.45{\scriptstyle\pm0.48}$ \\
 & 0.01   & $32.82{\scriptstyle\pm1.04}$ & $32.43{\scriptstyle\pm1.55}$ & $29.58{\scriptstyle\pm1.46}$ & $35.57{\scriptstyle\pm0.48}$ \\
 & 0.03   & $35.28{\scriptstyle\pm0.93}$ & $34.59{\scriptstyle\pm1.04}$ & $31.92{\scriptstyle\pm1.31}$ & $35.26{\scriptstyle\pm0.48}$ \\
\midrule
\multirow{4}{*}{Herding}
 & 0.0001 & $26.64{\scriptstyle\pm2.87}$ & $27.03{\scriptstyle\pm2.43}$ & $26.60{\scriptstyle\pm4.08}$ & $34.07{\scriptstyle\pm0.79}$ \\
 & 0.005  & $31.87{\scriptstyle\pm0.36}$ & $31.42{\scriptstyle\pm0.63}$ & $27.96{\scriptstyle\pm1.54}$ & $31.97{\scriptstyle\pm1.18}$ \\
 & 0.01   & $34.19{\scriptstyle\pm0.97}$ & $33.08{\scriptstyle\pm0.81}$ & $28.46{\scriptstyle\pm1.34}$ & $32.67{\scriptstyle\pm0.75}$ \\
 & 0.03   & $38.52{\scriptstyle\pm0.71}$ & $36.66{\scriptstyle\pm0.64}$ & $33.01{\scriptstyle\pm0.47}$ & $37.30{\scriptstyle\pm0.75}$ \\
\midrule
\multirow{4}{*}{BONSAI}
 & 0.0001 & \cellcolor{bestcol}$42.85{\scriptstyle\pm0.41}$ & \cellcolor{bestcol}$40.56{\scriptstyle\pm0.39}$ & \cellcolor{bestcol}$39.13{\scriptstyle\pm0.67}$ & \cellcolor{bestcol}$43.58{\scriptstyle\pm0.20}$ \\
 & 0.005  & $45.19{\scriptstyle\pm0.31}$ & $43.92{\scriptstyle\pm0.50}$ & $42.94{\scriptstyle\pm0.37}$ & $48.12{\scriptstyle\pm0.57}$ \\
 & 0.01   & $46.24{\scriptstyle\pm0.31}$ & $45.23{\scriptstyle\pm0.15}$ & $44.43{\scriptstyle\pm0.44}$ & $49.43{\scriptstyle\pm0.40}$ \\
 & 0.03   & \cellcolor{bestcol}$46.87{\scriptstyle\pm0.36}$ & $45.24{\scriptstyle\pm0.43}$ & $46.16{\scriptstyle\pm0.61}$ & $49.22{\scriptstyle\pm0.71}$ \\
\midrule
\multirow{4}{*}{GDEM}
& 0.0001 & $20.68{\scriptstyle\pm4.38}$ & $22.08{\scriptstyle\pm3.50}$ & $21.96{\scriptstyle\pm4.85}$ & $26.25{\scriptstyle\pm0.39}$ \\
& 0.005  & $34.48{\scriptstyle\pm0.96}$ & $36.03{\scriptstyle\pm0.36}$ & $23.96{\scriptstyle\pm3.10}$ & $36.39{\scriptstyle\pm0.27}$ \\
& 0.010  & $35.19{\scriptstyle\pm0.65}$ & $36.60{\scriptstyle\pm0.20}$ & $21.91{\scriptstyle\pm3.61}$ & $36.46{\scriptstyle\pm0.16}$ \\
& 0.030  & $35.54{\scriptstyle\pm0.46}$ & $36.69{\scriptstyle\pm0.71}$ & $22.29{\scriptstyle\pm3.96}$ & $36.51{\scriptstyle\pm0.15}$ \\
\midrule
\multirow{4}{*}{\textbf{HERALD}}
 & 0.0001 & $33.26{\scriptstyle\pm0.35}$ & $34.77{\scriptstyle\pm0.63}$ & $31.55{\scriptstyle\pm0.64}$ & $34.89{\scriptstyle\pm0.50}$ \\
 & 0.005  & \cellcolor{bestcol}$46.15{\scriptstyle\pm0.14}$ & \cellcolor{bestcol}$44.82{\scriptstyle\pm0.32}$ & \cellcolor{bestcol}$44.17{\scriptstyle\pm0.41}$ & \cellcolor{bestcol}$50.35{\scriptstyle\pm0.43}$ \\
 & 0.01   & \cellcolor{bestcol}$46.91{\scriptstyle\pm0.30}$ & \cellcolor{bestcol}$45.37{\scriptstyle\pm0.52}$ & \cellcolor{bestcol}$47.15{\scriptstyle\pm0.39}$ & \cellcolor{bestcol}$50.72{\scriptstyle\pm0.24}$ \\
 & 0.03   & $46.86{\scriptstyle\pm0.24}$ & \cellcolor{bestcol}$45.45{\scriptstyle\pm0.39}$ & \cellcolor{bestcol}$47.22{\scriptstyle\pm0.42}$ & \cellcolor{bestcol}$50.98{\scriptstyle\pm0.19}$ \\
\bottomrule
\end{tabular}
\end{table}

\begin{table}[htbp]
\caption{Node classification accuracy (\%) on \textbf{Chameleon} ($h=0.770$).}
\label{tab:chameleon}
\centering
\setlength{\tabcolsep}{4pt}
\begin{tabular}{@{}llcccc@{}}
\toprule
\textbf{Condenser} & \textbf{$r$}
  & \textbf{GCN} & \textbf{GAT} & \textbf{GIN} & \textbf{H2GCN} \\
\midrule
Full graph & --- &
$35.92{\scriptstyle\pm1.56}$ &
$43.03{\scriptstyle\pm1.02}$ &
$32.28{\scriptstyle\pm2.03}$ &
$52.24{\scriptstyle\pm0.73}$ \\
\midrule
\multirow{4}{*}{Random}
 & 0.0001 & $16.80{\scriptstyle\pm4.33}$ & $19.30{\scriptstyle\pm4.69}$ & $18.64{\scriptstyle\pm5.11}$ & $23.68{\scriptstyle\pm3.58}$ \\
 & 0.005  & $14.34{\scriptstyle\pm2.82}$ & $17.76{\scriptstyle\pm4.47}$ & $14.04{\scriptstyle\pm2.46}$ & $18.60{\scriptstyle\pm2.40}$ \\
 & 0.01   & $27.41{\scriptstyle\pm2.24}$ & $28.03{\scriptstyle\pm3.11}$ & $30.18{\scriptstyle\pm2.52}$ & $35.75{\scriptstyle\pm1.89}$ \\
 & 0.03   & $31.67{\scriptstyle\pm1.36}$ & $30.35{\scriptstyle\pm1.25}$ & $31.58{\scriptstyle\pm1.93}$ & $33.33{\scriptstyle\pm2.16}$ \\
\midrule
\multirow{4}{*}{Herding}
 & 0.0001 & $34.82{\scriptstyle\pm0.87}$ & $29.39{\scriptstyle\pm2.91}$ & $27.89{\scriptstyle\pm2.73}$ & $33.82{\scriptstyle\pm1.17}$ \\
 & 0.005  & $37.81{\scriptstyle\pm0.96}$ & $26.97{\scriptstyle\pm7.86}$ & $25.66{\scriptstyle\pm3.44}$ & $39.08{\scriptstyle\pm2.08}$ \\
 & 0.01   & $36.36{\scriptstyle\pm0.66}$ & $31.71{\scriptstyle\pm1.12}$ & $30.09{\scriptstyle\pm0.78}$ & $38.73{\scriptstyle\pm1.14}$ \\
 & 0.03   & $35.48{\scriptstyle\pm0.61}$ & $32.24{\scriptstyle\pm2.57}$ & \cellcolor{bestcol}$32.28{\scriptstyle\pm2.01}$ & $40.44{\scriptstyle\pm0.82}$ \\
\midrule
\multirow{4}{*}{BONSAI}
 & 0.0001 & $30.75{\scriptstyle\pm1.10}$ & $27.72{\scriptstyle\pm1.48}$ & $24.56{\scriptstyle\pm2.45}$ & $36.01{\scriptstyle\pm1.70}$ \\
 & 0.005  & $32.15{\scriptstyle\pm0.86}$ & $29.12{\scriptstyle\pm1.26}$ & $26.01{\scriptstyle\pm0.91}$ & $37.41{\scriptstyle\pm0.95}$ \\
 & 0.01   & $32.15{\scriptstyle\pm0.86}$ & $29.12{\scriptstyle\pm1.26}$ & $26.01{\scriptstyle\pm0.91}$ & $37.41{\scriptstyle\pm0.95}$ \\
 & 0.03   & $32.15{\scriptstyle\pm0.86}$ & $29.12{\scriptstyle\pm1.26}$ & $26.01{\scriptstyle\pm0.91}$ & $37.41{\scriptstyle\pm0.95}$ \\
\midrule
\multirow{4}{*}{GDEM}
& 0.0001 & $23.16{\scriptstyle\pm0.98}$ & $23.60{\scriptstyle\pm4.42}$ & \cellcolor{bestcol}$38.90{\scriptstyle\pm0.51}$ & $39.17{\scriptstyle\pm0.69}$ \\
& 0.005  & $23.73{\scriptstyle\pm5.52}$ & $22.15{\scriptstyle\pm4.80}$ & \cellcolor{bestcol}$35.57{\scriptstyle\pm2.19}$ & $38.16{\scriptstyle\pm0.96}$ \\
& 0.010  & $27.02{\scriptstyle\pm1.73}$ & $22.63{\scriptstyle\pm2.94}$ & \cellcolor{bestcol}$31.84{\scriptstyle\pm2.80}$ & $38.46{\scriptstyle\pm1.28}$ \\
& 0.030  & $25.18{\scriptstyle\pm3.32}$ & $23.03{\scriptstyle\pm4.57}$ & $22.06{\scriptstyle\pm2.51}$ & $39.78{\scriptstyle\pm1.96}$ \\
\midrule
\multirow{4}{*}{\textbf{HERALD}}
 & 0.0001 & \cellcolor{bestcol}$36.23{\scriptstyle\pm1.56}$ & \cellcolor{bestcol}$32.19{\scriptstyle\pm1.40}$ & $26.54{\scriptstyle\pm4.07}$ & \cellcolor{bestcol}$41.97{\scriptstyle\pm1.23}$ \\
 & 0.005  & \cellcolor{bestcol}$38.90{\scriptstyle\pm1.71}$ & \cellcolor{bestcol}$32.85{\scriptstyle\pm1.35}$ & $26.75{\scriptstyle\pm4.07}$ & \cellcolor{bestcol}$41.58{\scriptstyle\pm1.13}$ \\
 & 0.01   & \cellcolor{bestcol}$38.90{\scriptstyle\pm1.71}$ & \cellcolor{bestcol}$32.85{\scriptstyle\pm1.35}$ & $26.75{\scriptstyle\pm4.07}$ & \cellcolor{bestcol}$41.58{\scriptstyle\pm1.13}$ \\
 & 0.03   & \cellcolor{bestcol}$38.90{\scriptstyle\pm1.71}$ & \cellcolor{bestcol}$32.85{\scriptstyle\pm1.35}$ & $26.75{\scriptstyle\pm4.07}$ & \cellcolor{bestcol}$41.58{\scriptstyle\pm1.13}$ \\
\bottomrule
\end{tabular}
\end{table}

\begin{table}[htbp]
\caption{Node classification accuracy (\%) on \textbf{Squirrel} ($h=0.776$).}
\label{tab:squirrel}
\centering
\setlength{\tabcolsep}{4pt}
\begin{tabular}{@{}llcccc@{}}
\toprule
\textbf{Condenser} & \textbf{$r$}
  & \textbf{GCN} & \textbf{GAT} & \textbf{GIN} & \textbf{H2GCN} \\
\midrule
Full graph
& ---
& $24.38{\scriptstyle\pm1.07}$
& $28.72{\scriptstyle\pm1.16}$
& $24.23{\scriptstyle\pm2.65}$
& $38.77{\scriptstyle\pm1.24}$ \\
\midrule

\multirow{4}{*}{Random}
& 0.0001 & $22.82{\scriptstyle\pm1.15}$ & $21.46{\scriptstyle\pm1.77}$ & $24.17{\scriptstyle\pm0.80}$ & $23.63{\scriptstyle\pm1.83}$ \\
& 0.005  & $23.59{\scriptstyle\pm1.12}$ & $23.57{\scriptstyle\pm1.44}$ & $22.86{\scriptstyle\pm1.50}$ & $27.38{\scriptstyle\pm1.16}$ \\
& 0.010  & $20.86{\scriptstyle\pm1.03}$ & $21.83{\scriptstyle\pm1.60}$ & $22.61{\scriptstyle\pm2.57}$ & $26.22{\scriptstyle\pm1.02}$ \\
& 0.030  & $22.17{\scriptstyle\pm0.36}$ & $23.50{\scriptstyle\pm1.48}$ & $24.25{\scriptstyle\pm2.46}$ & $28.26{\scriptstyle\pm1.09}$ \\
\midrule

\multirow{4}{*}{Herding}
& 0.0001 & $25.03{\scriptstyle\pm1.01}$ & $22.19{\scriptstyle\pm1.77}$ & $25.19{\scriptstyle\pm0.83}$ & $29.49{\scriptstyle\pm0.85}$ \\
& 0.005  & $26.65{\scriptstyle\pm0.68}$ & $23.11{\scriptstyle\pm2.46}$ & \cellcolor{bestcol}$23.78{\scriptstyle\pm2.22}$ & $31.43{\scriptstyle\pm1.00}$ \\
& 0.010  & \cellcolor{bestcol}$28.34{\scriptstyle\pm0.58}$ & $24.40{\scriptstyle\pm2.25}$ & \cellcolor{bestcol}$23.27{\scriptstyle\pm2.71}$ & $31.93{\scriptstyle\pm0.71}$ \\
& 0.030  & \cellcolor{bestcol}$26.95{\scriptstyle\pm0.77}$ & $24.48{\scriptstyle\pm2.14}$ & \cellcolor{bestcol}$25.13{\scriptstyle\pm2.04}$ & $33.22{\scriptstyle\pm0.77}$ \\
\midrule

\multirow{4}{*}{BONSAI}
& 0.0001 & \cellcolor{bestcol}$27.20{\scriptstyle\pm0.95}$ & \cellcolor{bestcol}$25.36{\scriptstyle\pm0.95}$ & $20.37{\scriptstyle\pm0.79}$ & $33.08{\scriptstyle\pm0.74}$ \\
& 0.005  & $25.99{\scriptstyle\pm0.92}$ & $24.76{\scriptstyle\pm1.24}$ & $21.13{\scriptstyle\pm2.21}$ & $33.99{\scriptstyle\pm0.75}$ \\
& 0.010  & $25.21{\scriptstyle\pm1.64}$ & $24.88{\scriptstyle\pm1.28}$ & $20.02{\scriptstyle\pm1.88}$ & $35.25{\scriptstyle\pm0.29}$ \\
& 0.030  & $25.49{\scriptstyle\pm1.62}$ & $24.69{\scriptstyle\pm0.60}$ & $20.44{\scriptstyle\pm1.59}$ & $35.35{\scriptstyle\pm0.67}$ \\
\midrule

\multirow{4}{*}{GDEM}
& 0.0001 & $22.15{\scriptstyle\pm2.15}$ & $20.90{\scriptstyle\pm0.98}$ & \cellcolor{bestcol}$27.78{\scriptstyle\pm0.69}$ & $32.97{\scriptstyle\pm0.74}$ \\
& 0.005  & $23.02{\scriptstyle\pm1.00}$ & $21.44{\scriptstyle\pm2.85}$ & $23.67{\scriptstyle\pm2.13}$ & $32.53{\scriptstyle\pm0.58}$ \\
& 0.010  & $23.17{\scriptstyle\pm1.35}$ & $21.42{\scriptstyle\pm2.69}$ & $20.19{\scriptstyle\pm4.04}$ & $32.93{\scriptstyle\pm0.41}$ \\
& 0.030  & $21.86{\scriptstyle\pm0.72}$ & $23.84{\scriptstyle\pm1.65}$ & $19.83{\scriptstyle\pm2.39}$ & $32.56{\scriptstyle\pm0.29}$ \\
\midrule

\multirow{4}{*}{\textbf{HERALD}}
& 0.0001 & $25.13{\scriptstyle\pm0.39}$ & $25.03{\scriptstyle\pm1.16}$ & $22.46{\scriptstyle\pm0.78}$ & \cellcolor{bestcol}$35.20{\scriptstyle\pm1.03}$ \\
& 0.005  & \cellcolor{bestcol}$26.82{\scriptstyle\pm1.11}$ & \cellcolor{bestcol}$24.90{\scriptstyle\pm1.48}$ & $23.07{\scriptstyle\pm1.07}$ & \cellcolor{bestcol}$35.93{\scriptstyle\pm0.79}$ \\
& 0.010  & $26.88{\scriptstyle\pm0.99}$ & \cellcolor{bestcol}$25.19{\scriptstyle\pm0.28}$ & $22.86{\scriptstyle\pm0.77}$ & \cellcolor{bestcol}$36.22{\scriptstyle\pm1.07}$ \\
& 0.030  & $26.88{\scriptstyle\pm0.99}$ & \cellcolor{bestcol}$25.19{\scriptstyle\pm0.28}$ & $22.86{\scriptstyle\pm0.77}$ & \cellcolor{bestcol}$36.22{\scriptstyle\pm1.07}$ \\
\bottomrule
\end{tabular}
\end{table}

\end{appendices}

\end{document}